\PassOptionsToPackage{table}{xcolor}
\documentclass{article}
\usepackage[accepted]{icml2026}

\usepackage{titletoc}
\usepackage{float}

\usepackage{microtype}
\usepackage{graphicx}
\usepackage{subfigure}
\usepackage{wrapfig}
\usepackage{adjustbox}

\usepackage{booktabs} % for professional tables
\usepackage{enumitem}
\usepackage{xcolor}

\colorlet{ours}{blue!10}

\usepackage{hyperref}

\usepackage{comment}
\usepackage{verbatim} % comment
\usepackage{amssymb,amsmath,amsthm}
\usepackage{mathrsfs}
\usepackage{multirow}
\usepackage{algorithm}
\usepackage{algorithmic}
\usepackage{bm}
\usepackage{pgfplots}

\usepackage{graphicx,booktabs}
\usepackage{subcaption}

\usepackage[capitalize,noabbrev]{cleveref}

\usepackage{tcolorbox}
\usepackage{tabularx}
\usepackage{pifont}

\theoremstyle{plain}
\newtheorem{theorem}{Theorem}[section]
\newtheorem{proposition}[theorem]{Proposition}
\newtheorem{lemma}[theorem]{Lemma}
\newtheorem{corollary}[theorem]{Corollary}
\theoremstyle{definition}
\newtheorem{definition}[theorem]{Definition}
\newtheorem{assumption}[theorem]{Assumption}
\newtheorem{remark}[theorem]{Remark}

\usepackage{mathtools}
\newcommand{\eqdef}{\coloneqq}

\usepackage{xurl}

\icmltitlerunning{Federated Sketching LoRA: A Flexible Framework for Heterogeneous Collaborative Fine-Tuning of LLMs}

\begin{document}

\twocolumn[\icmltitle{Federated Sketching LoRA: A Flexible Framework for Heterogeneous Collaborative Fine-Tuning of LLMs}




\begin{icmlauthorlist}
\icmlauthor{Wenzhi Fang}{purdue}
\icmlauthor{Dong-Jun Han}{yongsei}
\icmlauthor{Liangqi Yuan}{purdue}
\icmlauthor{Seyyedali Hosseinalipour}{buffalo}
\icmlauthor{Christopher G. Brinton}{purdue}
\end{icmlauthorlist}

% Department of Electrical and Computer Engineering, Purdue University
% Department of Computer Science and Engineering, Yonsei University
% Department of Electrical Engineering, University at Buffalo-SUNY
\icmlaffiliation{purdue}{Purdue University}
\icmlaffiliation{yongsei}{Yonsei University}
\icmlaffiliation{buffalo}{University at Buffalo-SUNY}

\icmlcorrespondingauthor{Wenzhi Fang}{fang375@purdue.edu}
%\icmlcorrespondingauthor{Firstname2 Lastname2}{first2.last2@www.uk}

\icmlkeywords{Machine Learning, ICML}

\vskip 0.3in
]



\printAffiliationsAndNotice{}  % leave blank if no need to mention equal contribution

\begin{abstract}
%There is growing interest in on-device fine-tuning of large language models (LLMs). 
Fine-tuning large language models (LLMs) on resource-constrained clients remains a challenging problem.
Recent works have fused low-rank adaptation (LoRA) techniques with federated fine-tuning to mitigate challenges associated with client model sizes and data scarcity. Still, the heterogeneity of resources remains a critical bottleneck: while higher-rank modules generally enhance performance, varying client capabilities constrain LoRA's feasible rank range.
Existing approaches attempting to resolve this issue either lack analytical justification or impose additional computational overhead, leaving a wide gap for efficient and theoretically-grounded solutions.
To address these challenges, we propose federated sketching LoRA (FSLoRA), which leverages a sketching mechanism to enable clients to selectively update submatrices of global LoRA modules maintained by the server. By adjusting the sketching ratios, which determine the ranks of the submatrices on the clients, FSLoRA flexibly adapts to client-specific communication and computational constraints.
We provide a rigorous convergence analysis of FSLoRA that characterizes how the sketching ratios affect the convergence rate. 
%Through comprehensive experiments on multiple datasets and LLM models, we demonstrate FSLoRA's performance improvements compared to various baselines. 
Through extensive experiments, we demonstrate that FSLoRA outperforms baselines and significantly improves training efficiency while preserving stable convergence.
\end{abstract}

\begin{comment}
\textbf{storyline}

 constrained resources and

Accompanied by the popularity of large language models,  the on-device large language model, collaborative on-device fine-tuning has attracted a growing body of attention recently.

Story flow of the intro

LLM success, from cloud-LLM to on-device LLM.

LLM is generally trained over a large datasets, fine-tuning is necessary to adapt to specific downstream tasks. fine-tuning schemes, ... . For on-device LLM, limited data is not enough for effective fine-tuning. Federated learning provides a promising solution. data privacy. 

Common LLM fine-tuning method, .... . However, all of these methods focus on centralized setup, which requires a central server to aggregate data from clients, which is thus not applicable for on-device LLM fine-tuning. 

Challenges in FedLLM fine-tuning. 

FedLLM fine-tuning method. and their problems. 

\end{comment}

\section{Introduction}

Lightweight client-side large language models (LLMs) have recently gained significant attention as a promising complement to cloud-based LLMs \citep{fan2024device}. They align with the typical paradigm of LLMs: starting from a base model pre-trained on large-scale datasets to learn general linguistic patterns, semantics, and context, and then undergoing fine-tuning on task-specific data to enhance performance on specialized or domain-specific applications. However, an LLM fine-tuned on a single client often achieves unsatisfactory performance due to the limited data. Federated learning (FL) \citep{mcmahan2017communication, chen2023federated} has been investigated as a potential solution here, enabling the model to be fine-tuned across a group of distributed clients within the same task domain, without any raw data sharing.

%However, federated learning imposes significant computational and memory costs, as each client must adapt the LLM using its local dataset and send updates to the server for model aggregation.
%However, federated LLM fine-tuning is costly in both computation and communication, since each client performs local adaptation and transmits full updates to the server.
However, federated LLM fine-tuning is costly in both computation and communication due to the massive parameter volume.
%Although fine-tuning requires much fewer computational resources than pre-training, the memory cost would still be a hard barrier for full model fine-tuning. 
%For example, fine-tuning a full LLaMA 3B model with a batch size of one demands at least $25$ GB memory, including $6$ GB for trainable parameters, $18$ GB for Adam optimizer states and weight gradients, and $1$ GB for activations. 
Importantly, many parameter-efficient fine-tuning methods have been proposed \citep{lester2021power, li2021prefix, hu2021lora} to reduce the model adaptation cost. 
Among them, low-rank adaptation (LoRA) \citep{hu2021lora} stands out as a particularly effective approach due to its flexibility. %Based on the observation that updates to the weights \( \Delta {\bf W} \) from the pre-trained model to specific tasks have a low intrinsic rank, LoRA reduces the volume of trainable parameters by approximating \( \Delta {\bf W} \) as the product of two low-rank matrices \( {\bf B} \) and \( {\bf A} \), i.e., \( \Delta {\bf W} = {\bf BA} \), enabling efficient fine-tuning.
In particular, LoRA enables efficient fine-tuning by approximating weight updates \( \Delta {\bf W} \) through a low-rank decomposition \( \Delta {\bf W} = {\bf BA} \), where matrices \( {\bf B} \) and \( {\bf A} \) contain significantly fewer trainable parameters than the original weight matrix. 
%To support efficient federated LLM fine-tuning, \citet{zhang2024towards,ye2024openfedllm} incorporated LoRA into conventional federated averaging (FedAvg) \citep{mcmahan2017communication}, significantly reducing the fine-tuning cost by cutting down the number of trainable parameters.
%Building on this foundation, recent works combined LoRA with federated averaging (FedAvg) \citep{zhang2024towards,ye2024openfedllm}, showing that federated LoRA significantly reduce the training overhead.
Building on this foundation, recent works incorporated LoRA with federated averaging (FedAvg) to support collaborative LLM fine-tuning \citep{zhang2024towards,ye2024openfedllm,zou2025joint}, significantly reducing the fine-tuning cost by cutting down the number of parameters that need to be synchronized across distributed clients.

\textbf{Challenges.} While incorporating LoRA into federated LLM fine-tuning reduces the number of trainable parameters, \textit{computation and communication costs are still forced to increase with the LoRA rank}.
This poses challenges when complex tasks demand higher-rank LoRA modules, particularly on resource-constrained clients. Furthermore, the \textit{heterogeneity in resource availability across distributed clients makes a uniform rank adopted in federated LoRA inefficient}: a fixed rank $r$ may be too large for some constrained clients, while being too small for more powerful ones, resulting in underutilized resources. Consequently, an approach that further reduces computation and communication overhead while adapting LoRA ranks to heterogeneous client capabilities is highly desirable. Although some existing approaches \citep{cho2024heterogeneouslorafederatedfinetuning,bai2024federatedfinetuninglargelanguage,wang2024flora} have attempted to provide a solution, they either lack theoretical justification or impose additional computational overhead, leaving a gap for an efficient and theoretically-grounded solution. A detailed discussion of these limitations is provided in Section~\ref{sec:limitations_of_existing}. 
Overall, a comprehensive approach that preserves the analytical and practical benefits of LoRA while allowing heterogeneous collaborative fine-tuning under tight client resource constraints remains elusive.

%\vspace{-1mm}
\subsection{Contributions}
%\vspace{-1mm}
%Motivated by these limitations, the objective of this work is to develop a methodology for \textit{collaborative on-device LLM fine-tuning with convergence guarantees that maintains the flexibility of LoRA while addressing to the challenges arising from the system heterogeneity and resource limitations of distributed devices.} 

% \begin{wrapfigure}{r}{0.575\textwidth}
%     \vspace{-15mm}
%     \begin{center}
% \includegraphics[width=.575\textwidth]{Figure/Overview.png}
%     \end{center}
%     \vspace{-3mm}
%     \caption{\small An illustration of our proposed methodology where the server maintains a pair of global LoRA modules while the clients adaptively update submatrices of the global LoRA modules through sketching during each round.}
% \label{fig:Motivation}
%     \vspace{-2.6mm}
% \end{wrapfigure}

Motivated by these limitations, this work develops a methodology for efficient federated LLM fine-tuning that (i) retains the flexibility of LoRA, (ii) provides theoretical convergence guarantees, and (iii) addresses the challenges posed by heterogeneous and constrained resources across distributed clients. As depicted in Figure \ref{fig:Motivation}, our key idea is to introduce a sketching-based LoRA update to the local fine-tuning, which allows clients to selectively update a subset of columns and rows of the LoRA modules during each round, reducing the computation and communication consumption. Additionally, our method customizes the fine-tuning process by adjusting the sparsity level of the sketching matrix, i.e., the size of the updated submatrices for each client in each iteration. As we will see, the impact of the introduced sketching mechanism on the overall optimization landscape requires careful modeling consideration, posing additional challenges for the theoretical analysis that we address in this work.

In summary, we make the following contributions:
\begin{itemize}[leftmargin=*]
    \item %We introduce FSLoRA, a resource-adaptive algorithm for collaborative on-device LLM fine-tuning. By leveraging configurable sketching matrices, FSLoRA provides flexible control over local computation and communication overhead, allowing distributed devices to optimize their update matrices based on resource constraints. 
    We propose federated sketching LoRA (FSLoRA), which leverages a sketching mechanism to enable clients to selectively update submatrices of global LoRA modules maintained by the server. By adjusting the sketching ratios, which determine the ranks of the submatrices on clients, FSLoRA effectively adapts to client-specific communication and computational constraints. 
    \item We present a rigorous convergence analysis of FSLoRA under non-uniform submatrix updates across clients (i.e., heterogeneous LoRA configurations), revealing how the sketching ratios (i.e., resource heterogeneity) affect the convergence rate through scaled smoothness constants. 
    %\textcolor{magenta}{I think we can provide some interesting or non-trivial insights in 1-2 sentences.} 
    %Further, our results show that while increasing the sketching ratios improves convergence theoretically, it also raises communication and computation costs, suggesting a potential trade-off in selecting the sketching ratios.
    Our analysis establishes convergence under arbitrary sketching ratios and recovers the convergence rate of FedAvg in the limiting case.
    \item %We evaluate FSLoRA on a wide range of datasets and models with diverse parameter settings, consistently showing that FSLoRA outperforms existing solutions for collaborative on-device LLM fine-tuning in the federated system with heterogeneous resource distribution.
    We conduct extensive experiments across multiple datasets and LLM models with diverse parameter settings, demonstrating FSLoRA's superior performance compared to various baselines in accuracy, training efficiency, and resource utilization. Our ablation studies further validate the effectiveness of the sketching mechanism and the ability of clients to exploit larger global ranks under FSLoRA.
\end{itemize}

\begin{figure}[t]
    \centering
    \includegraphics[width=\linewidth]{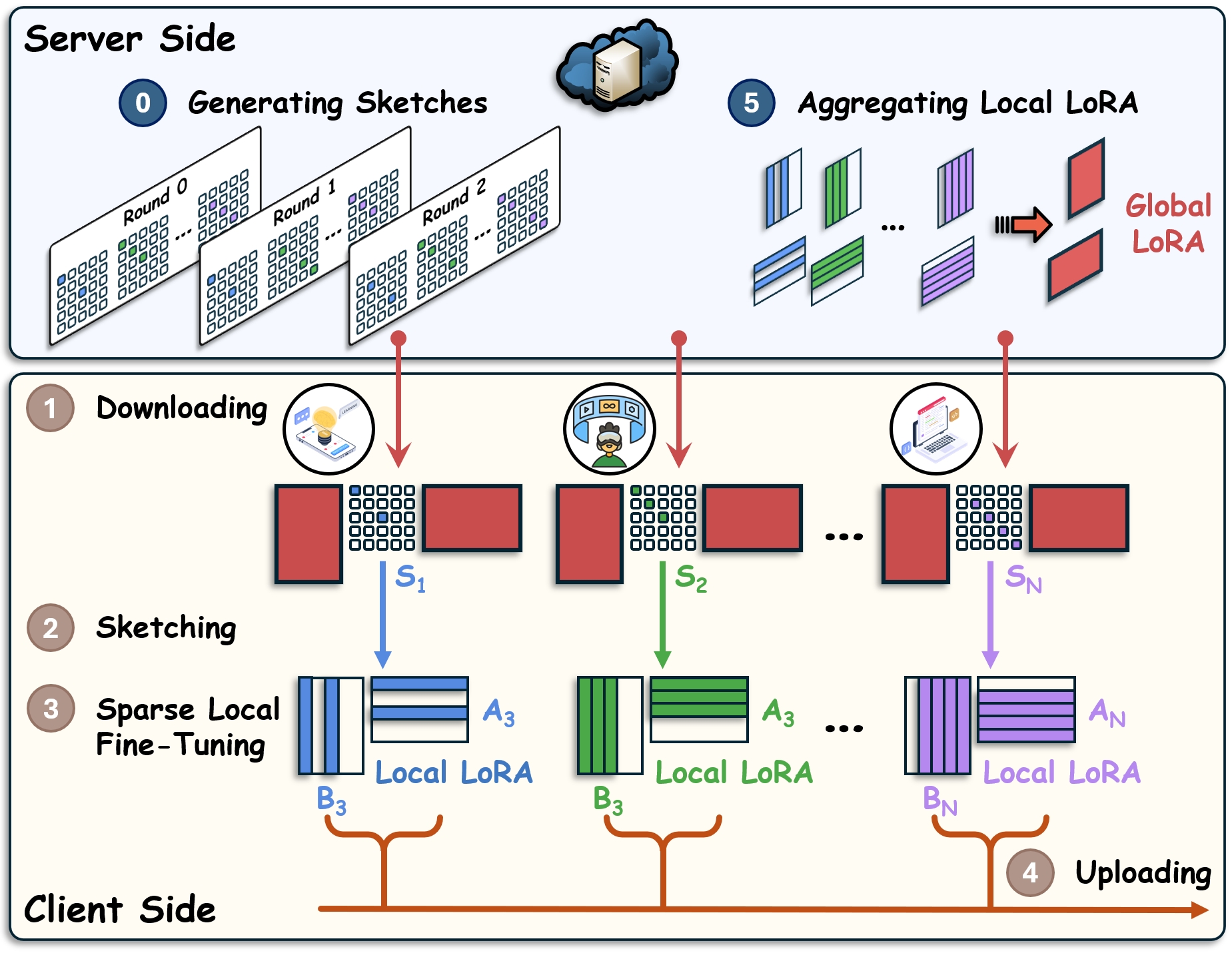}
    \caption{ An illustration of our proposed methodology where the server maintains a pair of global LoRA modules while the clients adaptively update submatrices of the global LoRA modules through sketching in each round.}
    \label{fig:Motivation}
    \vspace{-2mm}
\end{figure}
%Concretely, to reduce communication and computation overhead, our approach involves updating only smaller subtensors of the LoRA modules during each iteration. By using a sketching matrix (see Figure \ref{fig:Motivation}), which acts as a filter, we decrease the volume of information transferred between edge devices and the server, making the implementation of distributed LLM systems across large amount of edge devices more practical. Additionally, recognizing the diverse computational capabilities across edge devices, we introduce a resource-adaptive strategy for fine-tuning. Instead of uniformly applying the same updates across all devices, our method customizes the fine-tuning process by adjusting the sparsity level of the sketching matrix, i.e., the size of the updated subtensor for each device in each iteration, as demonstrated in Figure \ref{fig:Motivation}. On the other hand, the use of non-uniform LoRA modules introduces extra challenge for the theoretical analysis. We thoroughly investigate this problem and make the following contributions:
% \begin{figure}[t]
%     \centering
%     \includegraphics[width=\linewidth]{Figure/Overview}
%     \caption{An illustration of our proposed methodology where the server maintains a pair of global LoRA modules while the devices adaptively update submatrices of the global LoRA modules through sketching during each round.}
%     \label{fig:Motivation}
%     \vspace{-2mm}
% \end{figure}
\subsection{Related Works}
\textbf{Collaborative fine-tuning via federated LoRA:} Federated LoRA is an efficient approach for collaborative LLM fine-tuning among distributed clients \citep{chen2023federated,sun2024improving,guo2025selectiveaggregationlowrankadaptation,chen2025robust,guo2024selective}. Building on this foundation, \citet{kuo2024federatedlorasparsecommunication} proposed integrating communication compression with federated LoRA to further reduce communication overhead. Meanwhile, \citet{bai2024federatedfinetuninglargelanguage,cho2024heterogeneouslorafederatedfinetuning,byun2024towards,wang2024flora,koo2024robustefficientfederatedlowrank} explored the challenges of resource heterogeneity across distributed clients and introduced heterogeneous LoRA as a solution.
However, the approaches proposed in \citep{cho2024heterogeneouslorafederatedfinetuning,koo2024robustefficientfederatedlowrank,byun2024towards,bai2024federatedfinetuninglargelanguage} lack a theoretical foundation. Moreover, the FlexLoRA method introduced in \citep{bai2024federatedfinetuninglargelanguage} incurs additional computational overhead due to its reliance on singular value decomposition (SVD). Furthermore, the FLoRA algorithm proposed in \citep{wang2024flora} requires the clients to merge the LoRA modules into the base model, thereby compromising the inherent flexibility of LoRA. Overall, there is still a lack of a systematic and theoretically grounded solution that can effectively tackle heterogeneous collaborative LLM fine-tuning.

%\textbf{Effectiveness of Higher-Rank LoRA Modules.}
\textbf{Enhancing adaptability via higher-rank LoRA modules:}
The foundational study by \citet{hu2021lora} demonstrated that small ranks can be sufficient for certain tasks; however, they also acknowledge that small rank LoRA modules may not work universally, especially when the downstream task differs significantly from pretraining. Following this, several works explored the effect of increasing the rank in LoRA modules. In a centralized setup, \citet{kalajdzievski2023rankstabilizationscalingfactor} and  \citet{shuttleworth2024lora} showed that higher-rank LoRA models can closely approximate full fine-tuning under rsLoRA. In a federated LLM fine-tuning regime, \citet{bai2024federatedfinetuninglargelanguage} demonstrated improved performance with larger ranks under FlexLoRA. 
%Similarly, \citet{cho2024heterogeneouslorafederatedfinetuning} reported that increasing LoRA rank can lead to better performance initially, but eventually will overfit. They also showed that by applying some tricks, e.g., rank-pruning, overfitting can be mitigated and HeteroLoRA can enjoy the benefit of larger ranks. 
Similarly, \citet{cho2024heterogeneouslorafederatedfinetuning} reported that, with proper overfitting control, HeteroLoRA can also benefit from larger ranks. Orthogonal to these rank-scaling approaches, Uni-LoRA~\citep{li2026uni} reduces the trainable parameter space by reconstructing LoRA parameters from a projected low-dimensional subspace, highlighting another direction for improving LoRA efficiency.
Overall, while small ranks may suffice for simpler tasks or strong base models, higher-rank modules are necessary to compensate for limited backbone capability, such as in lightweight LLMs, and to enable effective adaptation to more complex downstream tasks. 

\textbf{Sketching-based optimization:} Sketching is an efficient technique for mitigating the complexity of high-dimensional optimization, with its earliest applications in least-squares regression \citep{sarlos2006improved,wang2022iterative}. Beyond this, gradient sketching has been employed to construct preconditioners for gradient descent methods \citep{gower2015randomized}. Building on these foundations, recent work has applied sketching to distributed optimization. In particular, \citet{charalambides2024distributed} proposed hybrid local-global sketching for distributed least-squares, while \citet{demidovich2023mast} developed a distributed sparsified training framework based on sketching. \citet{shrivastava2024sketching} demonstrated that sketching substantially reduces communication in distributed training of overparameterized deep models without sacrificing accuracy.
%\citet{nicolas2025communication} introduced correlation-aware sketches that simultaneously reduce communication and differentially private noise. 
More recently, \citet{nicolas2025communication} investigated sketching-based differential privacy and demonstrated its compatibility with secure aggregation.
Despite these advances, sketching strategies tailored to structured low-rank adaptation modules such as LoRA remain largely unexplored.

\section{Problem Background}\label{sec:background}
%\vspace{-1.5mm}

%\subsection{Preliminaries}
\subsection{LoRA-based Federated LLM Fine-tuning}\label{sec:background:fedlora}
The federated LoRA fine-tuning problem can be formulated as 
\vspace{-0.05in}
\begin{equation}\label{fed_lora_formulation}
\begin{aligned}
    \min_{\bf{B}, \bf{A}} f(\bf B,{\bf A}) &\eqdef \frac{1}{N} \sum_{i=1}^{N} f_i(\bf B,{\bf A}) \\
    f_i(\bf B,{\bf A}) &\eqdef  \mathbb{E}_{\xi \sim \mathcal{D}_i} \left[ \ell ({\bf W}_{0} + {\bf B} {\bf A}, \xi)\right],
\end{aligned}
\end{equation}
where \({\bf W}_{0}\) denotes the frozen base model, ${\bf B} \in \mathbb{R}^{m\times r},  {\bf A}\in \mathbb{R}^{r\times n}$ are LoRA modules, \(N\) denotes the number of clients, \(\xi\) denotes a data sample, and \(\mathcal{D}_i\) is the local dataset on client \(i\). 
\(\ell\), \(f_i\), and \(f\) are the sample loss function, the local loss for client \(i\), and the global loss, respectively. 

Problem \eqref{fed_lora_formulation} aligns with the conventional federated optimization formulation, which thus can be solved using the FedAvg algorithm. 
Based on the FedAvg framework, \citet{zhang2024towards} developed federated LoRA, which applies a uniform rank \(r\) across all clients, overlooking resource heterogeneity. 
This one-size-fits-all approach leads to resource mismatches, where computationally constrained clients may struggle, while more powerful clients remain underutilized with a fixed rank.

\subsection{Aren't the Existing Solutions Good Enough?}\label{sec:limitations_of_existing}
%\subsection{Limitations of Other Federated LoRA Approaches}
%\vspace{-1mm}

To address this issue, researchers have proposed heterogeneous federated LoRA approaches, where clients maintain non-uniform LoRA modules with varying ranks. They also introduce mechanisms to overcome the challenges of directly aggregating matrices with different dimensions.   
However, these methods often lack theoretical foundation or incur additional computational and memory overhead.
%However, these methods either lack theoretical justification or introduce additional computational and memory overhead, as outlined below.

%\textbf{HeteroLoRA \citep{cho2024heterogeneouslorafederatedfinetuning}} tries to addresses the aforementioned challenges by allowing \textcolor{blue}{edge devices} to update non-uniform LoRA modules. \textcolor{blue}{Devices} with smaller rank modules pad their updates to match the size of the largest modules before aggregation. During global model dissemination, resource-constrained devices receive a truncated version of the global model. HeteroLoRA, while easy to implement, is largely heuristic and lacks a rigorous theoretical foundation. Additionally, its dependence on zero-padding diminishes optimization efficiency, potentially limiting overall performance.
\textbf{HeteroLoRA \citep{cho2024heterogeneouslorafederatedfinetuning}} lets the server pad the updates from the clients with smaller ranks to match the size of the largest rank during aggregation. During model dissemination, clients receive a truncated version of the global LoRA modules from the server. Although easy to implement, HeteroLoRA is primarily heuristic in nature and lacks a rigorous theoretical foundation, potentially limiting its performance, as we will see in Section \ref{sec:experiment}.
%HeteroLoRA, while easy to implement, lacks a solid theoretical foundation. Additionally, its dependence on zero-padding diminishes optimization efficiency, potentially limiting overall performance.

%\textbf{FlexLoRA \citep{bai2024federatedfinetuninglargelanguage}} introduces an alternative approach where the server first collects the individual LoRA matrices $\mathbf{B}_i $ and $ \mathbf{A}_i$ from edge devices and then computes the product $\mathbf{B}_i \mathbf{A}_i$ for each device and averages these products. To enable non-uniform initialization of LoRA modules, the server applies truncated singular value decomposition (SVD) to the averaged product $\frac{1}{N} \sum_{i=1}^N \mathbf{B}_i \mathbf{A}_i$. Although FlexLoRA supports heterogeneous LoRA updates, the use of truncated SVD introduces significant computational and memory overhead on the server. Additionally, the approximation errors resulting from the truncated SVD could undermine its effectiveness in certain scenarios.

\textbf{FlexLoRA \citep{bai2024federatedfinetuninglargelanguage}} requires the server to collect the individual LoRA matrices $\mathbf{B}_i $ and $ \mathbf{A}_i$ from the clients and then computes their product $\mathbf{B}_i \mathbf{A}_i$. To support the initialization of non-uniform LoRA modules, the server applies truncated SVD to the averaged product $\frac{1}{N} \sum_{i=1}^N \mathbf{B}_i \mathbf{A}_i$. However, this approach introduces extra computational and memory overhead on the server due to truncated SVD, and the associated error can limit the performance as demonstrated in Section \ref{sec:experiment}.
   
%\textbf{FLoRA \citep{wang2024flora}} employs a stacking mechanism where the server concatenates the LoRA modules from distributed devices \(\mathbf{B}_1, \ldots, \mathbf{B}_N\) along the second dimension and \(\mathbf{A}_1, \ldots, \mathbf{A}_N\) along the first dimension. The concatenated matrices are then disseminated to the devices, where each device computes their product and merges it into the base model. Following that, the devices randomly initialize new LoRA modules for the next round of local fine-tuning. However, the communication complexity of FedStack grows linearly with the number of the devices due to the need to transmit these concatenated matrices.  Additionally, the devices are required to compute the product of the concatenated matrices and merge them into the base model, leading to increased computation and memory demands. This approach also compromises the inherent flexibility of LoRA, which allows a base model to accommodate multiple LoRA adapters tailored to different tasks.

\textbf{FLoRA \citep{wang2024flora}} introduces a stacking mechanism where the server concatenates LoRA modules from the clients. The concatenated matrices are then sent back to the clients, which compute their product and merge it into the base model before initializing new LoRA modules for the next fine-tuning round. However, this approach increases communication complexity linearly with the number of clients, imposes higher computation and memory demands on the clients, and compromises LoRA's flexibility to support multiple adapters for different tasks.

More detailed comparisons on computation, memory, and communication are presented in Appendix \ref{appen:com_memo_comparison}.
In summary, a theoretically-grounded solution that preserves LoRA's benefits while effectively addressing resource heterogeneity across distributed clients remains lacking. 
%\vspace{-1.5mm}

\section{Federated Sketching LoRA}\label{sec:fslora_formulation}
%\vspace{-0.5mm}

Motivated by the limitations of existing methods, we propose a reformulation for federated LoRA. Building on this foundation, we develop FSLoRA, a heterogeneous LoRA algorithm that preserves LoRA's flexibility while accommodating client resource heterogeneity.

%Motivated by the limitations of existing methods, we proposed a resource-adaptive algorithm, FSLoRA, which allows each device to update submatrices of the full modules \( \mathbf{B} \) and \( \mathbf{A} \) at each round. The introduced sketching mechanism enables the adaptive selection of submatrix sizes for the devices based on their resource constraints.

%The introduced sketching mechanism enables the adaptive selection of submatrix sizes for edge devices based on their resource constraints, which will be detailed in Section \ref{sec:heterogeneous}. The overview of FSLoRA is shown in Figure \ref{fig:Motivation} and Algorithm \ref{alo_flora}. 

\subsection{Our Formulation}

We propose a sketching-based LoRA formulation for collaborative LLM fine-tuning as follows: 
%\vspace{-0.05in}
\begin{equation}\label{lora_formulation_ours}
\begin{aligned}
\min_{\bf B,{\bf A}} f^{\mathcal{S}}(\bf B,{\bf A}) &\eqdef  \frac{1}{N} \sum_{i=1}^N f_i^{\mathcal{S}} (\bf B,{\bf A})\\
f_i^{\mathcal{S}}(\bf B,{\bf A}) &\eqdef  \mathbb{E}_{ {\bf S} \sim \mathcal{S}_i; \xi \sim \mathcal{D}_i} \left[ \ell ({\bf W}_{0} + {\bf B}\bf S{\bf A}, \xi)\right],
\end{aligned}
\end{equation}
where ${\bf B} \in \mathbb{R}^{m\times r},  {\bf A}\in \mathbb{R}^{r\times n}$ are LoRA modules, $f_i^{\mathcal{S}}$ is the local loss function at client $i$ with sketching, and \({\bf S}\) denotes a sketching matrix randomly sampled from the diagonal matrix set \(\mathcal{S}_i = \mathcal{S}(r, k_i)\). The set \(\mathcal{S}(r, k_i)\) comprises diagonal matrices of size \(r \times r\) with exactly \(k_i\) non-zero entries. 
%where ${\bf W}_{0}$ denotes the pre-trained model, ${\bf B} \in \mathbb{R}^{m\times r},  {\bf A}\in \mathbb{R}^{r\times n}$ are LoRA modules, \(\ell\) is the sample-wise loss function, \(\mathcal{D}_i\) is the local dataset of device \(i\), and \(\mathcal{S}_i\) represents the sketching matrix set for device $i$. 
The formal definition of \(\mathcal{S}(r, k)\) is provided below:
\begin{definition}[Random-$k$ sketching]\label{skeching_def}
A random-$k$ diagonal matrix set is defined as:  
\begin{equation}
    \mathcal{S}(r, k) \!=\! \Bigg \{ \! \mathbf{S} \!\mid \! \mathbf{S} \!=\! \frac{r}{k} \! \sum_{j \in \mathcal{I}} \! \mathbf{e}_j \mathbf{e}_j^\top, \, \mathcal{I} \subseteq \{1, \ldots, r\}, \, |\mathcal{I}| \!=\! k \Bigg \}, \nonumber
\end{equation}
where \(\mathbf{e}_1, \dots, \mathbf{e}_r \in \mathbb{R}^r\) are standard unit basis vectors and index set \(\mathcal{I}\) is a random subset of \([r] \eqdef \{1,2,\ldots, r\} \) sampled uniformly from all subsets of \([r]\) with cardinality \(k\).
\end{definition}

With \(\mathbf{S}\) being a matrix sampled from \(\mathcal{S}_i\), we have $
   \mathbf{BSA} = \frac{r}{k_i} \sum_{j \in \mathcal{I}_i} \mathbf{B} \mathbf{e}_j \mathbf{e}_j^\top \mathbf{A},
   $
where \(\mathcal{I}_i\) corresponds to the index set of non-zero diagonal entries of \(\mathbf{S}\). \(\mathbf{B} \mathbf{e}_j\) extracts the \( j \)-th column of \(\mathbf{B}\) while \(\mathbf{e}_j^\top \mathbf{A}\) extracts the \( j \)-th row of \(\mathbf{A}\). In other words, only \(k_i\) columns and rows in the LoRA modules \(\bf B\) and \(\bf A\) are activated by the sketching matrix in the loss \(\ell ({\bf W}_{0} + {\bf B}\bf S{\bf A}, \xi)\) at client \(i\).  On the other hand, the sketching matrix \(\bf S\) satisfies \(\mathbb{E}_{\mathbf{S} \sim \mathcal{S}_i}[\mathbf{S}] = \mathbf{I}_r\) where $\mathbf{I}_r$ is a $r$-dimensional identity matrix. Based upon this property, ${\bf W}_0 + {\bf B} {\bf S} {\bf A}$ can be treated as an unbiased estimate of ${\bf W}_0 + \bf B{\bf A}$. 
%The proposed formulation (i.e., \eqref{lora_formulation_ours}) serves as an approximation of the original federated LoRA formulation (i.e., \eqref{fed_lora_formulation}). 
%
%The idea behind sketching shares some similarities with LoRA Dropout \citep{lin2024loradropoutsparsityregularizer}. LoRA Dropout randomly drops some columns and rows of the LoRA modules to achieve sparsity regularization, effectively mitigating overfitting. Our goal is to leverage the sketching matrix to strike a balance between the large rank of global LoRA modules and the limitations of resource-constrained devices. Specifically, when \({\bf S}\) is a random diagonal sketching matrix, the local gradients with respect to the LoRA modules computed on the devices exhibit structured sparsity. This structured sparsity effectively reduces both the computational and communication overhead on the devices, as elaborated in the following subsection. 
%

\textbf{Intuition:} A larger rank allows LoRA modules to be more expressive, leading to better performance \citep{bai2024federatedfinetuninglargelanguage,kalajdzievski2023rankstabilizationscalingfactor,shuttleworth2024lora}. However, resource-constrained clients cannot afford the computational or communication demands of large-rank modules. Our formulation \eqref{lora_formulation_ours} leverages the sketching matrix to balance the expressiveness of high-rank LoRA modules with the resource constraints of different clients. With the sketching mechanism introduced, the local gradients with respect to the LoRA modules on the clients will exhibit structured sparsity. By adjusting the sketching ratio \(k_i/r\), we can tailor the sparsity of the gradient to match the capabilities of each client, ensuring affordable training while maintaining performance across heterogeneous systems, as elaborated in the following subsection. Overall, compared to \eqref{fed_lora_formulation}, our formulation offers a more flexible framework, tailored to address the diverse capabilities of heterogeneous clients.
%Overall, compared to \eqref{fed_lora_formulation}, our formulation offers a more flexible framework for resource-heterogeneous system.

%\textcolor{magenta}{Write a sentence something like: Compared to the original FL formulation in \eqref{fed_lora_formulation}, our new objective .....(provide some insights on the advantage? or say the it's detailed in the following subsection?) }

\subsection{Sparsity in the Gradients}
In this subsection, we analyze the gradient structure of LoRA modules and highlight the gradients' sparsity properties under a given sketching matrix. 
%under our reformulation
To begin, we present the gradient expressions for the LoRA modules in the following lemma. The proof is provided in Appendix \ref{proof_compute_grad}.
\begin{lemma}[Gradient Formulation]\label{compute_grad}
For a given sketching matrix \(\bf S\), the gradients of  \( \ell(\bf W_0 + {\bf B}\bf S{\bf A}, \xi)\) with respect to ${\bf B}$ and ${\bf A}$ take the following form
%\tilde{\ell}(\bf B,{\bf A}, \xi) \eqdef
\begin{equation}\label{gradient_matrix}
\resizebox{0.9\hsize}{!}{$
\begin{aligned}
    {\nabla}_{\bf B} \ell({\bf W}_0 + {\bf B} {\bf S} {\bf A}, \xi) &= \nabla \ell({\bf W}_0 + {\bf B} {\bf S} {\bf A}, \xi)  {\bf A}^{\top} {\bf S}^{\top} \\
    {\nabla}_{\bf A} \ell({\bf W}_0 + {\bf B} {\bf S} {\bf A}, \xi) &= {\bf S}^{\top} {\bf B}^{\top} \nabla \ell(\bf W_0 + {\bf B}\bf S{\bf A}, \xi),
\end{aligned}
$}
\end{equation}
where \({\nabla}_{\bf B} \ell({\bf W}_0 + {\bf B} {\bf S} {\bf A}, \xi)\), \({\nabla}_{\bf A} \ell({\bf W}_0 + {\bf B} {\bf S} {\bf A}, \xi)\), and \(\nabla \ell(\bf W_0 + {\bf B}\bf S{\bf A}, \xi)\) denote the gradients of \(\ell(\bf W_0 + {\bf B}\bf S{\bf A}, \xi)\) to \(\bf B\), \(\bf A\), and \(\bf W_0 + {\bf B}\bf S{\bf A}\), respectively. 
%where \({\nabla}_{\bf B} \ell(\cdot, \xi)\), \({\nabla}_{\bf A} \ell(\cdot, \xi)\), and \(\nabla \ell(\cdot, \xi)\) denote the gradient of \(\ell(\cdot, \xi)\) to \(\bf B\), \(\bf A\), and \(\bf W_0 + {\bf B}\bf S{\bf A}\), respectively. 
\end{lemma}
% \begin{proof}
% The proof is provided in Appendix \ref{proof_compute_grad}.
% \end{proof}

In particular, a random-$k$ diagonal sketching matrix selectively samples $k$ rows or columns of a matrix through left product or right product, respectively. With \({\bf S}\) being a random-$k$ diagonal matrix, the gradients of $\ell({\bf W}_0 + {\bf B} {\bf S} {\bf A}, \xi)$ with respect to LoRA modules \(\bf B\) and \(\bf A\), as shown in \eqref{gradient_matrix}, naturally become structurally sparse matrices. 
%This sparsity reduces the computational and memory overhead during training, allowing for faster gradient computation and parameter updates. Additionally, sparse training enables better scalability across distributed clients by reducing communication costs, as only the non-zero elements need to be transmitted. 
This sparsity reduces computational and memory overhead during training, enabling faster gradient computation and parameter updates, while alleviating communication overhead across distributed clients by transmitting only non-zero elements.

%The sparsity level of these gradients at each device is determined by the corresponding sketching matrix set \(\mathcal{S}_i\). 

%Unlike LoRA Dropout, which uses fixed sparsity parameters, our approach incorporates an adaptive mechanism, allowing devices to adjust the sparsity level of the sketching matrix set \(\mathcal{S}_i\) based on their resource constraints, as detailed follows.

\begin{remark}[Sparsity Level Control]
A key advantage of our formulation is its flexible control over the sparsity level of local gradients, achieved by configuring the parameter \(k_i\) of the sketching matrix set \( \mathcal{S}_i = \mathcal{S}(r, k_i) \). This mechanism allows each client to tailor its local updates according to its communication and computation resource constraints, ensuring efficient and scalable fine-tuning in heterogeneous federated systems. Lowering \( k_i \) helps resource-constrained clients reduce computation and communication overhead, while more capable clients can increase \( k_i \) to conduct more informative local updates. 
%Additionally, the distinction in sparsity level control between the proposed FSLoRA and the FedBCGD algorithm \citep{liu2024fedbcgd} is elaborated in Appendix \ref{appen:difference_to_fbcgd}.
\end{remark}

\begin{remark}[Justification for the Choice of Random-$k$ Sketching]
We adopt Random-$k$ sketching because it is unbiased, induces structured sparsity, and exhibits strong empirical performance (see Section~\ref{sec:experiment}). In contrast, identifying an appropriate metric for quantifying the importance of individual LoRA rows or columns remains an open problem, even in the centralized LoRA setting. This lack of a principled importance measure prevents the reliable development of Top-$k$-style or importance-based sketching strategies. 
We empirically evaluate several heuristic importance-based sketching variants and find that they consistently underperform Random-$k$ sketching. 
Detailed theoretical intuition and empirical results can be found in Appendix~\ref{appen:discussion_alternative_sketching}.
\end{remark}

\begin{comment}
In addition, there is a close connection between the configuration of \(\{k_i\}_{i=1}^N\) with the smoothness of function \(
f^{\mathcal{S}}(\mathbf{B}, \mathbf{A}) = \frac{1}{N} \sum_{i=1}^N \mathbb{E}_{{\bf S} \sim \mathcal{S}_i}f_i({\bf B} {\bf S}, {\bf A}),
\) as demonstrated in the following proposition.
\begin{proposition}\label{smooth_derivation}
    Suppose that function \(f_i ({\bf B}, {\bf A})\) is \(L\)-smooth with respect to \([{\bf B}; {\bf A}]\) for each \(i\). Then \(f^{\mathcal{S}} ({\bf B}, {\bf A})\) is smooth to \([{\bf B}; {\bf A}]\) with parameter \(\left( \frac{1}{N} \sum_{i=1}^N \frac{r}{k_i} \right) L \).
\end{proposition}
The proof details of Proposition \ref{smooth_derivation} are provided in Appendix \ref{proof_smooth}.

Proposition \ref{smooth_derivation} highlights the influence of the random sketching matrices \({\bf S} \sim \mathcal{S}_i\) on the smoothness of \(f^{\mathcal{S}} ({\bf B}, {\bf A})\), where the factor \(\frac{r}{k_i}\) reflects the sketching factor introduced by the sketching mechanism. 
\end{comment}

\subsection{FSLoRA Algorithm}

Based on the formulation in \eqref{lora_formulation_ours}, we propose a resource-adaptive algorithm, termed FSLoRA. In each round, FSLoRA enables each client to update only a client-specific pair of submatrices from the global LoRA modules $\mathbf{B}$ and $\mathbf{A}$.
The server maintains the global LoRA modules \(\mathbf{B}\) and \(\mathbf{A}\) and periodically updates them by aggregating sparse local updates from distributed clients. The per-round FSLoRA procedure is detailed below.
\begin{itemize}[topsep=0pt,itemsep=4pt,parsep=0pt, partopsep=0pt, leftmargin=*]
    \item \textbf{Sketch Sampling \& Broadcast.} The server begins by sampling sketching matrices \(\{{\bf S}_i^t \sim \mathcal{S}_i\}_{i=1}^{N}\) for all clients, where \(\mathcal{S}_i\) represents the set of possible sketching matrices for client \(i\). These sketches are then sent to the corresponding clients. Additionally, the server broadcasts the current global LoRA modules \([\mathbf{B}^{t}; \mathbf{A}^{t}]\) to all clients. Note that the communication load introduced by transmitting the sketching matrix is negligible compared to global LoRA modules, as it involves only \emph{binary sketching indices}.
    \item \textbf{Local Sparse Training.} Clients initialize local LoRA modules from the global model and perform sketch-guided updates using ${\bf S}_i^t$, given by:
    %Clients perform local fine-tuning using sketch \({\bf S}_i^t\). Specifically, guided by sketching matrix \({\bf S}_i^t\), the update at client \(i\) during the \(h\)-th iteration of the \(t\)-th round is given by:
    \begin{equation}\label{gradient_client_i}
\begin{bmatrix}
\mathbf{B}_i^{t,h+1} \\
\mathbf{A}_i^{t,h+1}
\end{bmatrix}
=
\begin{bmatrix}
\mathbf{B}_i^{t,h} \\
\mathbf{A}_i^{t,h}
\end{bmatrix}
-
\gamma 
\begin{bmatrix}
{\Delta} {\bf B}_i^{t,h} ({\bf S}_i^t)^{\top}\\
({\bf S}_i^t)^{\top} {\Delta} {\bf A}_i^{t,h}
\end{bmatrix},
\end{equation}
 where \(\gamma\) denotes the learning rate and \([{\Delta} {\bf B}_i^{t,h}; {\Delta} {\bf A}_i^{t,h}]\) is a shorthand representation for:
\begin{equation} 
\resizebox{0.9\hsize}{!}{$
\begin{bmatrix}
{\Delta} {\bf B}_i^{t,h} \\
{\Delta} {\bf A}_i^{t,h}
\end{bmatrix}
\!=\!
\begin{bmatrix}
\nabla \ell({\bf W}_0 + {\bf B}_i^{t,h} {\bf S}_i^t {\bf A}_i^{t,h}, \xi_i^{t,h})  ({\bf A}_i^{t,h})^{\top} \\
({\bf B}_i^{t,h})^{\top} \nabla \ell({\bf W}_0 + {\bf B}_i^{t,h} {\bf S}_i^t {\bf A}_i^{t,h}, {\xi}_i^{t,h})
\end{bmatrix}.
$} \nonumber
\end{equation}
The update direction in \eqref{gradient_client_i} corresponds to the negative stochastic gradient of \(\ell({\bf W}_0 + {\bf B} {\bf S} {\bf A}, \xi)\) with respect to \([{\bf B}; {\bf A}]\) for a given sketch \( {\bf S}_i^t\), as established in Lemma \ref{compute_grad}.   
The total update for client \(i\) in one round of training, consisting of \(H\) local steps, can be expressed as follows:
 \[ 
 \begin{bmatrix}
\mathbf{B}_i^{t,H} \!-\! \mathbf{B}_i^{t,0} \\
\mathbf{A}_i^{t,H} \!-\! \mathbf{A}_i^{t,0}
\end{bmatrix} 
\!=\! 
\begin{bmatrix}
\gamma \left( \sum_{h=0}^{H\!-\!1} {\Delta} {\bf B}_i^{t,h} \right) ({\bf S}_i^t)^{\top} \\
\gamma ({\bf S}_i^t)^{\top} \left(\sum_{h=0}^{H\!-\!1} {\Delta} {\bf A}_i^{t,h} \right) 
\end{bmatrix}. \]
%It's evident that during each round, only a specific subset of columns of module $\bf B$ and a subset of rows of module $\bf A$ are updated during the \(t\)-th round at edge client \(i\).
%Note that only the columns of \(\mathbf{B}\) corresponding to the nonzero entries of \(\mathbf{S}_i^t\) and the rows of \(\mathbf{A}\) corresponding to the nonzero entries of \(\mathbf{S}_i^t\) are updated during the \(t\)-th round at edge client \(i\). In other words, \(\mathbf{S}\) activates only certain columns of \(\mathbf{B}\) and rows of \(\mathbf{A}\) at each round. Subsequently, edge clients send the non-zero columns and rows of sparse model updates to the server.
From the above equation, we see that only the columns of \(\mathbf{B}\) and the rows of \(\mathbf{A}\) corresponding to the nonzero entries of \(\mathbf{S}_i^t\) are updated during the \(t\)-th round at client \(i\). In essence, \(\mathbf{S}_i^t\) \emph{selectively activates} specific columns of \(\mathbf{B}\) and rows of \(\mathbf{A}\) for each round. Afterward, clients transmit these nonzero columns and rows of the sparse local updates to the server. 
    \item   \textbf{Sparse Update Reconstruction~\&~Aggregation.} Using the sketch information, the server reconstructs the full-size sparse updates from the received columns and rows via zero padding, and aggregates them to update the global LoRA modules:
\begin{equation}\label{model_aggregation}
%\resizebox{1\hsize}{!}{$
\begin{bmatrix}
\mathbf{B}^{t+1} \\
\mathbf{A}^{t+1}
\end{bmatrix}
=
\begin{bmatrix}
\mathbf{B}^{t} \\
\mathbf{A}^{t}
\end{bmatrix}
+
\frac{1}{N} \sum_{i=1}^N
\begin{bmatrix}
\mathbf{B}_i^{t,H} - \mathbf{B}_i^{t,0} \\
\mathbf{A}_i^{t,H} - \mathbf{A}_i^{t,0}
\end{bmatrix}.
%$}
\end{equation}
%Following this, the updated global LoRA modules will be broadcast to all edge devices to initialize the next round of local fine-tuning. 
\end{itemize}
%The above procedure is repeated for $t = 0, 1, \ldots, T\!-\!1$ across \(T\) rounds. 

Algorithm~\ref{alo_flora} summarizes the procedure of FSLoRA under full participation, while the extension to partial participation is provided in Appendix~\ref{appen:broader_heterogeneity}.

Intuitively, Random-$k$ sketching ensures uniform expected exposure of all columns and rows of the global LoRA modules to each client, closely mirroring the behavior of vanilla Federated LoRA. As training progresses, every component in the global LoRA is sampled with sufficient frequency in expectation, and thus adequately optimized. The convergence of FSLoRA and the impact of Random-$k$ sketching on convergence are analyzed in Section~\ref{sec:analysis}.

\begin{remark}[Computation, memory, and communication]\label{remark:commnucation} 
Compared to vanilla Federated LoRA~\cite{zhang2024towards}, FSLoRA introduces only lightweight additional operations, including Random-$k$ sketch sampling, sketch index broadcasting, and zero padding. 
Consequently, FSLoRA incurs negligible additional overhead at the server, while maintaining comparable or lower overall complexity than existing heterogeneous LoRA baselines~\citep{bai2024federatedfinetuninglargelanguage,wang2024flora}. More detailed server- and client-side comparisons with baselines are provided in Appendix~\ref{appen:com_memo_comparison}.
\end{remark}

\begin{remark}[Aggregation]
%There are two different aggregation patterns in the existing works on federated LoRA: 1) aggregating $[\mathbf{B}; \mathbf{A}]$ (e.g., vanilla Federated LoRA) and 2) aggregating $\mathbf{B}\mathbf{A}$ via federated averaging (e.g., FLoRA), respectively. Both aggregations work as evidenced by the promising performance of vanilla Federated LoRA and FLoRA. In this work, we stick to the former as it could maintain the advantage of LoRA in flexibility and simplicity.  Additionally, it is worth noting that as $[\mathbf{B}_i^t; \mathbf{A}_i^t]$ converges, i.e., $[\mathbf{B}_1^t, \mathbf{A}_1^t ] = [\mathbf{B}_2^t, \mathbf{A}_2^t ] = \cdots = [\mathbf{B}_N^t, \mathbf{A}_N^t ]$, will be proved in Section \ref{sec:analysis}, the two aggregations are equivalent, i.e., \( \frac{1}{N} \sum_{i} {\bf{B}}_i^t {\bf{A}}_i^t = (\frac{1}{N} \sum_{i=1}^{N} {\bf{B}}_i^t) (\frac{1}{N} \sum_{i=1}^{N} {\bf{A}}_i^t) \).
Existing works on federated LoRA primarily adopt two aggregation strategies: (1) aggregating the LoRA modules as $[\mathbf{B}; \mathbf{A}]$ (e.g., vanilla Federated LoRA~\cite{zhang2024towards}), and (2) aggregating the product $\mathbf{B} \mathbf{A}$ (e.g., FlexLoRA \cite{bai2024federatedfinetuninglargelanguage}). Both methods have demonstrated effectiveness, as evidenced by their promising performance in prior studies.  In this work, we adopt the former, as it introduces minimal overhead and retains the simplicity of LoRA. 
%Additionally, we establish the convergence of FSLoRA under this aggregation choice in Section~\ref{sec:analysis}. 
Additionally, we prove that FSLoRA is compatible with secure aggregation in Appendix \ref{appen:secure_aggregation}.
\end{remark}

\begin{algorithm}[t]
\caption{Federated Sketching LoRA (FSLoRA)}
\label{alo_flora}
\begin{algorithmic}[1]
\REQUIRE Base model ${\bf W}_0$, LoRA modules ${\bf B}^0$ and ${\bf A}^0$, learning rate $\gamma$, and sketching set $\{\mathcal{S}_i\}_{i=1}^{N}$
\FOR{$t = 0, 1, \ldots, T-1$}
    \STATE Server samples sketching matrices $\{{\bf S}_i^t \sim \mathcal{S}_i \}_{i=1}^{N}$ and sends them back to the clients
    \STATE Server broadcasts the current global LoRA modules to the clients
    \FOR{$h = 0, 1, \ldots, H\!-\!1$}
        \STATE Clients update the local LoRA modules via \eqref{gradient_client_i}
    \ENDFOR
    \STATE Clients upload the non-zero columns of $({\bf B}_i^{t,H} \!-\! {\bf B}_i^{t,0})$ and the non-zero rows $({\bf A}_i^{t,H} \!-\! {\bf A}_i^{t,0})$
    \STATE Server updates the global LoRA modules via \eqref{model_aggregation}
\ENDFOR
\end{algorithmic}
\end{algorithm}
\vspace{-3mm}

\subsection{Comparison with Communication Compression}
%Although both the sketching approach in FSLoRA and communication compression \citep{kuo2024federatedlorasparsecommunication} reduce communication overhead, the sketching approach fundamentally differs from traditional compression techniques. Notably, these two methods are orthogonal and can be combined to achieve greater efficiency. Specifically, the compression can be applied to the transmission of non-zero columns of \( \mathbf{B} \) and the non-zero rows of \( \mathbf{A} \) in FSLoRA to further enhance communication efficiency. We demonstrate the effectiveness of this combination in Appendix \ref{appen:compression_integration}. Additionally, most prior compression methods focus solely on reducing the transmission load, leaving the gradient computation and model updates unchanged from the vanilla Federated LoRA. FSLoRA goes beyond communication savings by also reducing gradient computation and model update overhead through sparse training. 

Although both the sketching approach in FSLoRA and communication compression \citep{kuo2024federatedlorasparsecommunication} reduce communication overhead, the sketching approach fundamentally differs from traditional compression techniques. Compression methods focus solely on reducing the transmission load, leaving the gradient computation and model updates unchanged from the vanilla Federated LoRA. FSLoRA goes beyond communication savings by also reducing gradient computation and model update overhead through sparse training.  Notably, these two methods are orthogonal and can be combined to achieve greater efficiency. 
%Specifically, the compression can be applied to the transmission of non-zero columns of \( \mathbf{B} \) and the non-zero rows of \( \mathbf{A} \) in FSLoRA to further enhance communication efficiency. 
We validate the effectiveness of this combination in Appendix \ref{appen:compression_integration}.

\section{Analysis} \label{sec:analysis}

Our theoretical analysis focuses on the full-participation setting in order to isolate and clearly characterize the effect of the \emph{sketching mechanism}.
As sketching is independent of client sampling, the impact of partial participation on FSLoRA is analogous to that on FedAvg.
Accordingly, extending our analysis from full to partial participation follows standard techniques in FL~\citep{yang2021achieving}. 

%We conduct analysis based on the following notations and assumptions. 
We conduct our analysis under the following notations and assumptions.
%We now introduce the notations and assumptions used throughout our analysis.

%We show that the FSLoRA algorithm converges to a stationary point of the function \eqref{lora_formulation_ours}. Our analysis relies on the following notations.
 
\textbf{Notations:} We define \(\tilde{\ell} ({\bf B}, {\bf A}, \xi; {\bf S}) = \ell ({\bf W}_{0} + {\bf B} {\bf S} {\bf A}, \xi)\) and
\(\tilde{f}_{i} ({\bf B}, {\bf A}; {\bf S}) = \mathbb{E}_{\xi \sim \mathcal{D}_i} \left[ \ell ({\bf W}_{0} + {\bf B}\bf S{\bf A}, \xi)\right] \) for a given ${\bf S}$ 
%\[\tilde{f}_{i} ({\bf B}, {\bf A}; {\bf S}) = f_i({\bf B} {\bf S}, {\bf A}) \eqdef \mathbb{E}_{\xi \sim \mathcal{D}_i} \left[ \ell ({\bf W}_{0} + {\bf B}\bf S{\bf A}, \xi)\right]\]
and \( f_i^{\mathcal{S}}(\mathbf{B}, \mathbf{A}) = \mathbb{E}_{{\bf S} \sim \mathcal{S}_i} [\tilde{f}_{i} ({\bf B}, \mathbf{A}; {\bf S}) ] \). 
We denote \( \mathbf{X} = [\mathbf{B}; \mathbf{A}] \) and rewrite \( f(\mathbf{B}, \mathbf{A}) \), \( f_i(\mathbf{B}, \mathbf{A}) \), \( f^{\mathcal{S}}(\mathbf{B}, \mathbf{A}) \), \( f_i^{\mathcal{S}}(\mathbf{B}, \mathbf{A}) \), \( \tilde{f}_i({\bf B}, {\bf A}; {\bf S}) \), and \(\tilde{\ell} ({\bf B}, {\bf A}, \xi; {\bf S})\) as \( f(\mathbf{X}) \), \( f_i(\mathbf{X}) \), \( f^{\mathcal{S}}(\mathbf{X}) \), \( f_i^{\mathcal{S}}(\mathbf{X}) \), \( \tilde{f}_{i} (\mathbf{X}; \mathbf{S}) \), and \(\tilde{\ell} ({\bf X}, \xi; {\bf S})\) respectively. Additionally, we use \(\|\!\cdot\!\|\) to represent the Frobenius norm.
\begin{assumption}\label{smoothness_assumption}
    During fine-tuning, the iterates remain in a bounded region \(\mathcal{C}\). 
    For each client \(i\), \(f_i({\bf X})\) is differentiable and \(L\)-smooth on \(\mathcal{C}\), i.e., there exists a positive constant \(L\) such that for any \({\bf X},{\bf Y}\in \mathcal{C}\),
    \begin{equation}
    \|\nabla f_i({\bf X}) - \nabla f_i({\bf Y})\| \leq L \|{\bf X} - {\bf Y}\|, \quad \forall i. \nonumber
    \end{equation}
\end{assumption}
\begin{assumption}\label{varaince_assumption}
$\nabla_{\bf X}  \tilde{\ell} ({\bf X}, \xi; {\bf S}) $ is an unbiased estimate of $\nabla_{\bf X}  f_i^{\mathcal{S}} ({\bf X})$ and its variance is bounded as 
\begin{equation}\mathbb{E}\| \nabla_{\bf X}  \tilde{\ell} ({\bf X}, \xi; {\bf S})  - \nabla_{\bf X}  f_i^{\mathcal{S}} ({\bf X})\|^2 \leq \rho \|\nabla_{\bf X}  f_i^{\mathcal{S}} ({\bf X})\|^2 +  \sigma^2 \nonumber
\end{equation}
where the expectation is computed over $\xi \sim \mathcal{D}_i$ and ${\bf S}\sim \mathcal{S}_i$. 
\end{assumption}
\begin{assumption}\label{assumption:gradient_dissimilarity}
%There exist positive constants \(c_h\), \(\delta_h^2\) such that the gradient dissimilarity between the global loss and local loss can be bound as follows
The gradient dissimilarity between the global loss \(f^{\mathcal{S}}(\bf X)\) and each local loss \( f_i^{\mathcal{S}}(\bf X) \) satisfies
\begin{equation}
\left\| \nabla_{\bf X} f_i^{\mathcal{S}}({\bf X}) \!-\! \nabla_{\bf X} f^{\mathcal{S}}(\bf X) \right\|^2 \! \leq \! c_h \| \nabla_{\bf X} f^{\mathcal{S}}({\bf X}) \|^2  \!+\! \delta_h^2, \nonumber
\end{equation}
where \(c_h \geq 0\) and \(f^{\mathcal{S}}(\mathbf{X}) = \frac{1}{N} \sum_{i=1}^N f_i^{\mathcal{S}}(\mathbf{X}) \).
\end{assumption}

Assumption \ref{smoothness_assumption} is standard in optimization literature \citep{bottou2018optimization,fang2024mtgc}. Assumptions \ref{varaince_assumption} and \ref{assumption:gradient_dissimilarity} are commonly adopted in FL to bound sampling randomness and data heterogeneity  \citep{fang2022communication,yi2022zeroth}. 
We further discuss Assumption \ref{smoothness_assumption} and provide an empirical validation for Assumptions \ref{varaince_assumption} and \ref{assumption:gradient_dissimilarity} in Appendix \ref{appen:justification_assumption}, showing that these assumptions are reasonable within the LLM fine-tuning scenario.
Building on these assumptions, we analyze the convergence behavior of FSLoRA and the results are summarized in the following theorem.
\begin{theorem}[Convergence of FSLoRA]\label{convergence_fed_slora_condense}
    Suppose that Assumptions \ref{smoothness_assumption}-\ref{assumption:gradient_dissimilarity} hold, then there exists a learning rate 
    \[
\resizebox{0.98\linewidth}{!}{$
\gamma \leq \min \Big\{
\frac{N}{24 \rho (c_h + 1) H\bar{L}},
\frac{1}{8\sqrt{\widetilde{L} L (\rho+1) (c_h +1)}\, H},
\frac{1}{H}
\Big\}
$}
\]
such that the iterates \(\{{\bf X}^t\}\) generated by FSLoRA satisfy 
\begin{align}
\frac{1}{T}\sum_{t=0}^{T-1} \mathbb{E} \left\| \nabla_{\bf X} f^{\mathcal{S}} ({\bf X}^t) \right\|^2 \! \leq \! 8 \frac{\sqrt{\bar{L} \mathcal{F}_0 \sigma_{\rho}^2} }{\sqrt{N TH}} \quad  & \nonumber \\
+ 10 (\widetilde{L} L)^{\frac{1}{3}} \left(\frac{\mathcal{F}_0 \sigma_{\rho}}{ T}\right)^{\frac{2}{3}} & + \frac{4 \mathcal{F}_0 }{T},\label{equa_in_thereom}
\end{align}
where $ \sigma_{\rho}^2 = \sigma^2 + 3(\rho+1) \sigma_h^2$, \(\bar{L} = \left(\frac{1}{N} \sum_{i=1}^N \frac{r}{ k_i}\right) L \), \(\widetilde{L} = \left(\frac{1}{N} \sum_{i=1}^N \frac{r^2}{ k^2_i}\right) L\) and \(\mathcal{F}_0 = f^{\mathcal{S}}({\bf X}^0)  -  f^*\) with $f^*$ denoting the lower bound of $f^{\mathcal{S}}({\bf X})$.
\end{theorem}
\textbf{Technical highlights of Theorem \ref{convergence_fed_slora_condense}:}
A key step in the proof of Theorem \ref{convergence_fed_slora_condense} is characterizing the impact of the sketching mechanism on the optimization landscape. Our analysis reveals how the sketching operation modifies the smoothness properties of the objective, introducing scaled smoothness constants, $\frac{r}{k_i} L$ and $\frac{r^2}{k_i^2} L$, which directly influence the convergence behavior.
Further details are presented in Appendix \ref{proof_theorem}.  

%\counterwithin{table}{section}
\newcommand{\com}[1]{\tiny$\pm$#1}

\begin{table*}[t]
%\vspace{-2mm}
\caption{Testing accuracy over 3 independent runs for fine-tuning the RoBERTa model on the GLUE benchmark. 
%The computational time for our experiments is reported in GPU hours.
%For FSLoRA, the rank of global LoRA modules is fixed at $r=64$ while the sketching ratio for each device \(\frac{k_i}{r}\) is independently sampled from \(\{0.125, 0.25, 0.5, 0.75\}\). For a fair comparison, the ranks of local LoRA modules for devices are sampled from $\{8, 16, 32, 48 \}$.
FSLoRA achieves a notable improvement in average performance compared to the baselines.
}
\vspace{-1mm}
\small
\centering
%\resizebox{0.8\textwidth}{!}{
\begin{tabular}{l|c|ccccccc|c}
\toprule
Method  & GPU hours & QNLI & MRPC & CoLA & MNLI & RTE &  SST-2  & QQP  & Avg. \\
\midrule
Base model & - & 49.5 & 52.0 & 30.9 & 32.6 & 47.3 & 50.9 & 63.2 & 46.6 \\
\midrule
HeteroLoRA & 10.7h & 87.5 \com{0.5} & 84.4 \com{0.9} & 75.3 \com{1.2} & 66.3 \com{0.8} & 69.0 \com{1.7} & 89.5 \com{0.0} & 85.3 \com{0.1} & 79.6  \\
FlexLoRA & 12.6h & 88.5 \com{0.2} & 81.2 \com{0.4} & 77.5 \com{1.2} & 63.0 \com{0.5} & 62.2  \com{1.9} & 92.8  \com{0.4} & 87.4  \com{0.1} & 78.9  \\
FLoRA & 12.3h & 87.2 \com{0.3} & 78.1 \com{0.7} & 77.4  \com{1.7} & 74.6 \com{0.5} & 54.4 \com{2.1} & 93.4  \com{0.1} & 87.1  \com{0.3} & 78.9  \\
RAVAN & 12.3h & 83.4 \com{1.1} & 76.2 \com{1.6} & 71.8 \com{1.4}  & 66.3 \com{3.6} & 63.1 \com{2.8} & 92.1 \com{0.9} & 83.4 \com{1.7} & 76.6 \\
\midrule
\rowcolor{ours} FSLoRA & 10.9h & 88.0 \com{0.3} & 87.3  \com{0.2} & 82.2 \com{0.5} & 76.4  \com{0.2} & 69.8 \com{1.2} & 93.5 \com{0.1} & 85.8 \com{0.1} & 83.3  \\
\bottomrule
\end{tabular}
%}
\label{tab:robert_het_accuracy}
%\vspace{-4mm}
\end{table*}
\begin{table*}[h]
\caption{
Testing accuracy over 3 independent runs for fine-tuning the LLaMA-3.2-3B-Instruct on the commonsense reasoning benchmark.
% , where Params refers to the number of trainable parameters while Commun represents the minimum uploading communication cost among edge devices. 
%The LoRA rank and sketching ratio setups are the same as those of Table \ref{tab:robert_het_accuracy}. 
FSLoRA demonstrates consistent performance improvement across these tasks compared to baselines.
}
\label{tab:llama_het_accuracy}
\vspace{-1mm}
\centering
%\tableindent 
\resizebox{\linewidth}{!}{
\begin{tabular}{l|c|cccccccc|c}
    \toprule
    Method & GPU hours & ARC-c & ARC-e & BoolQ & HellaSwag & OBQA & PIQA & SIQA & WinoGrande & Avg. \\
    \midrule
Base model & - & 51.5 & 61.4 & 47.2 & 55.8 & 52.4 & 21.3 & 54.4 & 4.6 & 43.6 \\
\midrule
   HeteroLoRA & 43.7h & 73.4 \com{0.3} & 86.6 \com{0.2} & 65.8 \com{0.5} & 73.0 \com{0.5} & 71.4 \com{0.3} & 80.9 \com{0.7} & 73.8 \com{0.3} & 72.0 \com{0.3} & 74.6 \\
FlexLoRA & 68.3h & 74.2 \com{0.3} & 86.7 \com{0.6} & 68.6 \com{0.8} & 79.4 \com{0.7} & 75.8 \com{0.4} & 81.0 \com{0.3} & 75.9 \com{0.4} & 77.9 \com{0.3} & 77.4 \\
FLoRA & 49.8h & 68.3 \com{0.6} & 83.1 \com{0.5} & 65.8 \com{0.9} & 77.2 \com{0.5} & 74.2 \com{0.3} & 80.5 \com{0.6} & 76.1 \com{0.5} & 71.5 \com{0.5} & 74.6 \\
RAVAN & 51.2h & 71.9 \com{0.9} & 81.6 \com{1.3} & 65.7 \com{3.2} & 76.7 \com{1.4} & 75.6 \com{1.5} & 77.1 \com{1.2} & 74.8 \com{0.6} & 73.8 \com{0.6} & 74.7  \\
\midrule
\rowcolor{ours} FSLoRA & 44.3h & 76.1 \com{0.4} & 87.2 \com{0.5} & 69.3 \com{0.7} & 82.2 \com{1.1} & 80.7 \com{0.6} & 84.0 \com{0.2} & 76.8 \com{0.0} & 79.1 \com{0.2} & 79.4 \\
    \bottomrule
\end{tabular}
}
%\vspace{-2mm}
\end{table*}

\textbf{Discussion:} Theorem \ref{convergence_fed_slora_condense} establishes an upper bound on the convergence rate of the proposed FSLoRA algorithm. The parameters \(\bar{L}\) and \(\widetilde{L}\) provide insight into how client resource heterogeneity influences the convergence. When clients possess sufficient computational resources, they can adopt larger ranks for their local LoRA modules (i.e., larger $k_i$), which leads to faster convergence of FSLoRA (i.e., smaller $\bar{L}$ and $\tilde{L}$). This effect is also empirically validated in Section~\ref{sec:experiment}.
%Conversely, due to the averaging effect across clients (reflected on $\bar{L}$), the convergence behavior remains stable even in the presence of a small fraction of resource-constrained clients, illustrating the robustness of our algorithm. We further validate this robustness empirically by simulating clients with limited capabilities using a heavy-tailed resource distribution (see Appendix~\ref{appen:larger_model}).
Moreover, the derived upper bound vanishes as $T \rightarrow \infty$ for any $k_i>0$, implying that FSLoRA converges to a stationary point under general sketching ratios. This highlights the robustness of FSLoRA to heterogeneous client resource constraints. 
In particular, as the sketching ratio $k_i/r \to 1$(i.e., $\bar{L} = L$), the decay rate of the dominating term (i.e., the first term on the right-hand-side of \eqref{equa_in_thereom}) recovers the convergence rate of FedAvg~\citep{yu2019parallel,karimireddy2020scaffold}. This demonstrates the tightness of our analysis and shows that FSLoRA retains the convergence guarantees of vanilla Federated LoRA in the limiting case.

\begin{remark}[Analysis Extension to Importance-Based Sketching]
    To this end, we provide a theoretically grounded framework for heterogeneous LoRA fine-tuning. Although our algorithm and analysis are based on Random-$k$ sketching, the framework naturally extends to importance-based sketching when a reliable importance metric is available. We investigate this extension in Appendix~\ref{appen:extention_importance}.
\end{remark}

\begin{figure*}
    \centering
    \includegraphics[width=\linewidth]{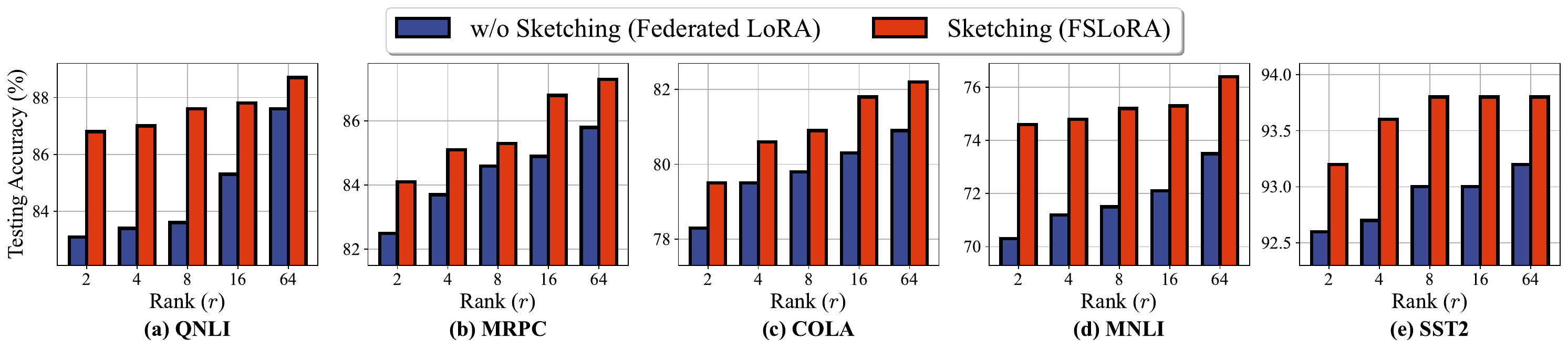}
    \vspace{-5mm}
    \caption{Comparison between FSLoRA with and without sketching (the latter equivalent to Federated LoRA) where the upload budget for clients is set to \(100 \times\) the size of the global LoRA modules at each rank. 
    %The experiment is performed on the GLUE benchmark and the RoBERTa model. 
    FSLoRA obtains a better performance, validating its communication efficiency.
    }
    \vspace{-4mm}
    \label{fig:Results_GLUE_r}
\end{figure*}
 
\begin{figure*}[h]
    \centering
    % First subfigure
    \subfigure[Comparison of FSLoRA with and without sketching, with an upload budget \(80 \times\) the size of the global LoRA modules.]{
        \includegraphics[width=0.3\linewidth]{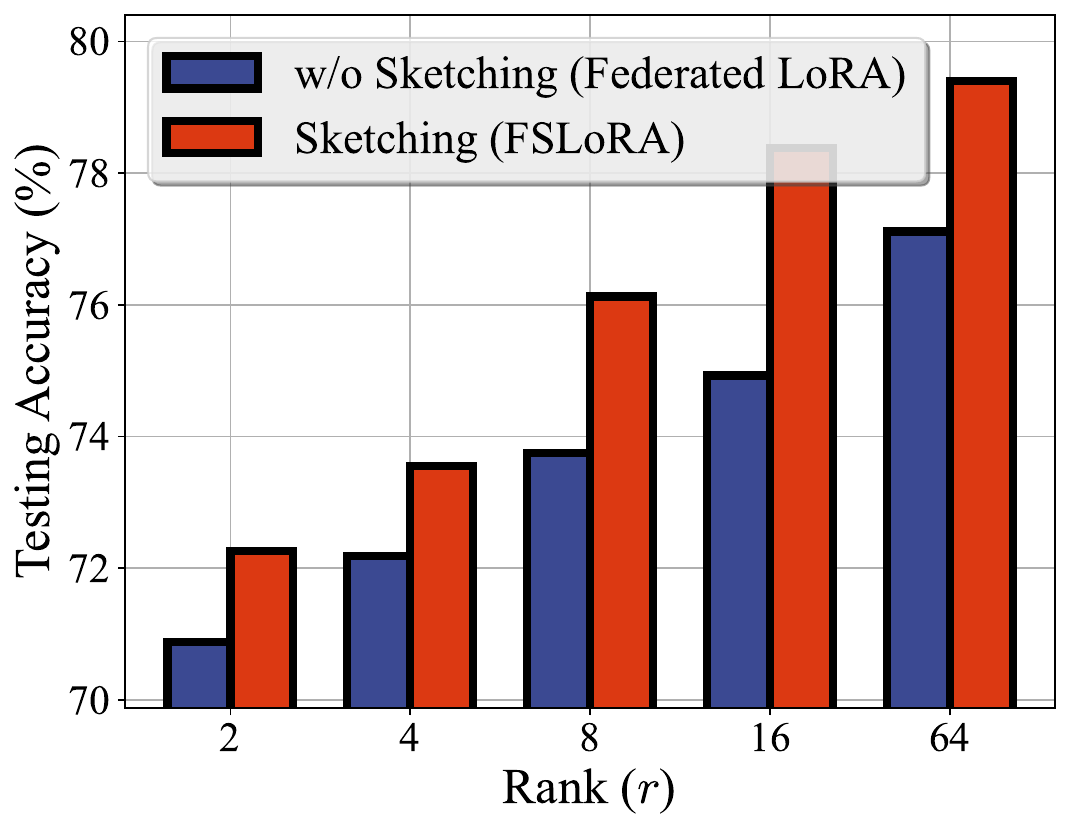} % Replace with your image
        \label{fig:rank_varying}
    }
    \hspace{1mm} % Space between subfigures
    % Second subfigure
    \subfigure[Impact of the rank of global LoRA modules on FSLoRA, given a fixed rank $k_i$ for the updated submatrices at the clients.]{
        \includegraphics[width=0.3\linewidth]{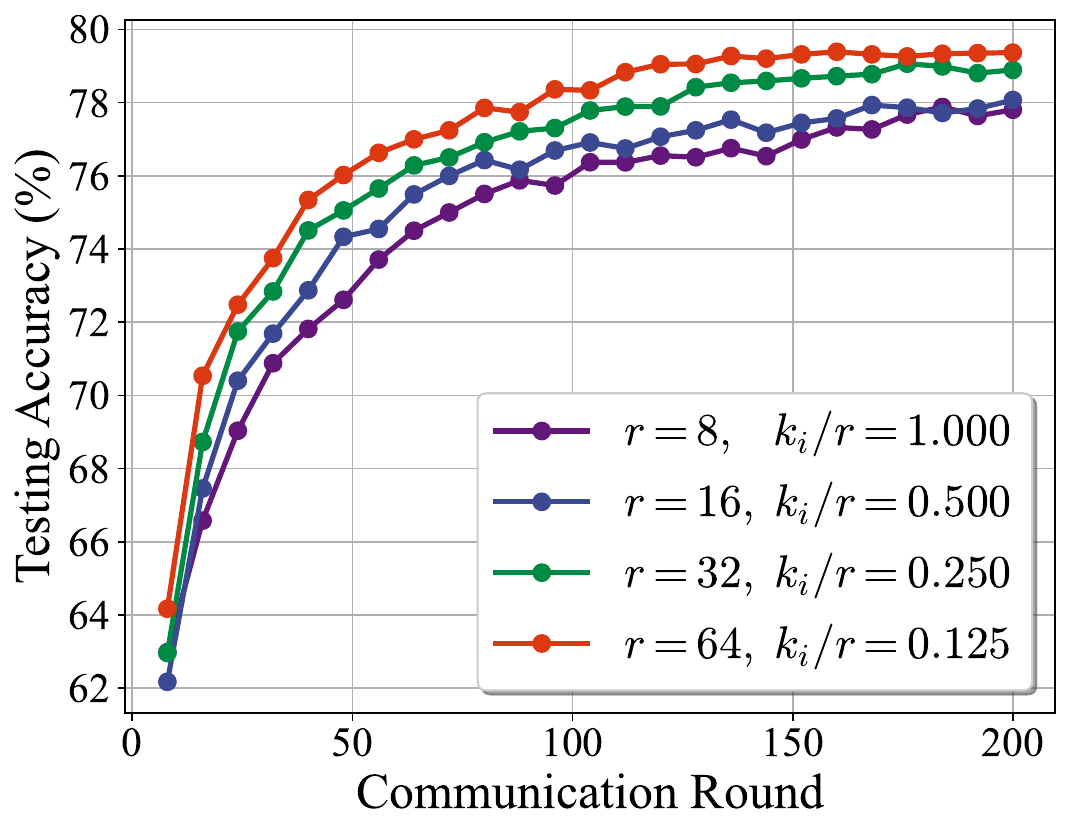} % Replace with your image
        \label{fig:global_rank}
    }
    \hspace{1mm} % Space between subfigures
    % Third subfigure
    \subfigure[Impact of the sketching ratio on FSLoRA's performance under a fixed rank $r=64$ for the global LoRA modules.]{
        \includegraphics[width=0.3\linewidth]{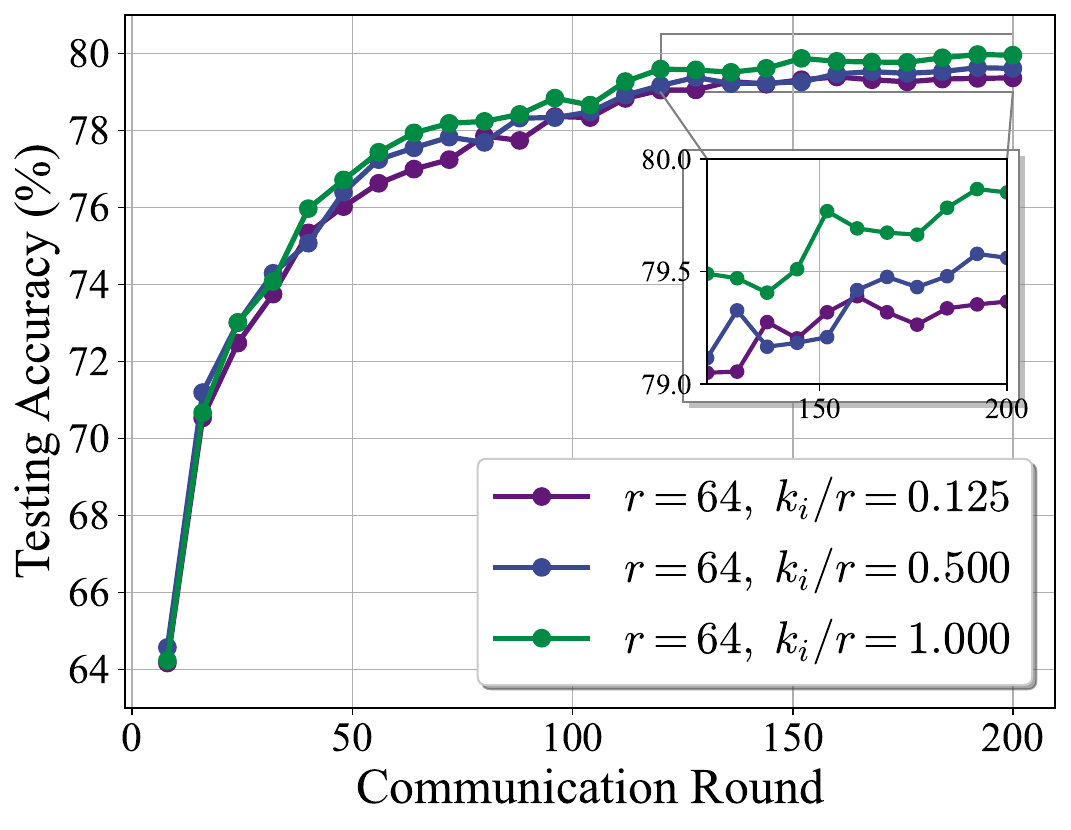} % Replace with your image
        \label{fig:sketching_ratio}
    }
    \vspace{-1mm}
    \caption{Fine-tuning the LLaMA-3.2-3B-Instruct model on the commonsense reasoning benchmark. The results are averaged over eight tasks, illustrating FSLoRA's ability to maintain strong performance while adapting to different rank and sketching configurations.}
    \label{fig:mainfigure}
    \vspace{-1mm}
\end{figure*}

\section{Experiments}\label{sec:experiment}

Our experiments focus on RoBERTa (125M) \citep{liu2019roberta} and LLaMA-3.2-3B-Instruct models \citep{llama3}, which represent typical model sizes suitable for client-side deployment, as well as the LLaMA-7B to reflect large-scale scenarios.
For RoBERTa and LLaMA-3.2-3B-Instruct models, we fine-tune and evaluate them on the GLUE \citep{wang2018glue} and commonsense reasoning benchmarks \citep{hu2023llm}, respectively. 
%For the LLaMA-7B model, we utilize Wizard, Dolly-15k, and Alpaca datasets, where the results are reported in Appendix \ref{appen:larger_model}.
Similar to \citep{zhang2024towards,wang2024flora}, we adopt Dirichlet-based partitioning for dataset splits.
%All the experiments are conducted on a cluster equipped with 4 NVIDIA A100 GPUs, each with 40 GB of memory. 
The number of clients is set to $20$ in the main manuscript and increased to $50$ and $100$ in Appendix \ref{appen:broader_heterogeneity}.
%Other hyperparameter configurations can be found in Appendix \ref{appen:hyperparameters}. 
Further details are provided in the Appendix \ref{appen:implementation_details}. 
%The code is available at \url{https://github.com/wenzhifang/Federated-Sketching-LoRA-Implementation}.
The implementation code for this project is available at \url{https://github.com/wenzhifang/Federated-Sketching-LoRA-Implementation}.

\subsection{Main Results Under Heterogeneous LoRA Setup}

\textbf{Performance comparison with baselines:} 
We consider four state-of-the-art baselines: three representative methods discussed in Section~\ref{sec:limitations_of_existing}, and RAVAN~\citep{raje2026ravan}, a recent federated LoRA fine-tuning method that addresses resource heterogeneity through partial LoRA head selection/freezing.
We also report the model's accuracy before fine-tuning (i.e., the base model). 
For FSLoRA, the rank of the global LoRA modules is fixed to $r=64$, while the sketching ratio for client $i$ is sampled from $\{0.125, 0.25, 0.5, 0.75\}$. For fair comparison, the same rank configuration is applied to all baseline methods.
Tables~\ref{tab:robert_het_accuracy} and~\ref{tab:llama_het_accuracy} present the performance of FSLoRA and the baselines. Across both benchmarks, FSLoRA consistently outperforms the base model in a wide margin, confirming the necessity of collaborative LLM fine-tuning. On the GLUE benchmark with RoBERTa, FSLoRA achieves the highest average accuracy among all methods. On the more challenging commonsense reasoning benchmark with LLaMA, FSLoRA still delivers the best overall performance. Notably, FSLoRA maintains computational efficiency comparable to HeteroLoRA in terms of GPU hours, highlighting its advantage in heterogeneous LoRA fine-tuning.

\textbf{Evaluation under broader heterogeneity, increased number of clients, and other models:} In Appendix \ref{appen:broader_heterogeneity}, we extend our evaluation to $50$ and $100$ clients under partial client participation, incorporating greater diversity in clients' communication and computation capabilities, as well as varying levels of data heterogeneity. In Appendix \ref{appen:larger_model}, we further assess the effectiveness of our method on Qwen2.5-1.5B-Instruct and LLaMA-7B models.

\subsection{Ablation Study}\label{sec:ablation_study}

\textbf{Impact of sketching:} 
In Figures \ref{fig:Results_GLUE_r} and \ref{fig:rank_varying}, we compare the performance of FSLoRA with and without sketching on fine-tuning RoBERTa and LLaMA-3.2-3B-Instruct, respectively. 
Note that FSLoRA without sketching is equivalent to the vanilla Federated LoRA.
For FSLoRA with sketching, we apply a uniform sketching ratio of \(k_i/r = 0.5\) across all clients.
The upload budget for each client is set to \(100\) and \(80\) times the size of the full global LoRA modules at the corresponding rank for the RoBERTa and the LLaMA-3.2-3B-Instruct models, respectively. 
%Specifically, the upload budget for each client is set to \(100\) times the size of the full global LoRA modules at the corresponding rank.
As shown in Figures \ref{fig:Results_GLUE_r} and \ref{fig:rank_varying}, both FSLoRA with and without sketching achieve higher accuracy when the rank $r$ increases due to the availability of more tunable parameters. 
In addition, FSLoRA consistently outperforms its non-sketched counterpart across all the ranks and datasets. The use of sketching increases the communication frequency for clients under the same communication budget, thereby facilitating the optimization process and enhancing fine-tuning efficiency. 

%This demonstrates that FSLoRA can maintain comparable performance even under severe memory and communication constraints, offering a more efficient solution in resource-constrained scenarios.

\textbf{Impact of the global rank:} In Figure \ref{fig:global_rank}, we investigate the impact of the rank of the global LoRA modules on FSLoRA's performance. 
%We consider 4 cases: 1) \(r = 8,~ k_i/r = 1\), 2) \(r = 16,~ k_i/r  = 0.5\), 3) \(r = 32,~ k_i/r = 0.25\), 4) \(r = 64,~ k_i/r = 0.125\). The rank of submatrices updated by the clients is kept consistent (i.e., $k_i=8$) across all configurations. 
%We consider 4 cases: 1) \(r = 8,~ k_i = 8\), 2) \(r = 16,~ k_i  = 8\), 3) \(r = 32,~ k_i = 8\)  4) \(r = 64,~ k_i = 8\). 
We vary the rank of the global LoRA modules while keeping the rank of submatrices updated by the clients to be consistent (i.e., $k_i=8$).
This ensures that the communication and computational resources on the client side remain unchanged.
As illustrated in Figure \ref{fig:global_rank}, FSLoRA maintains stable convergence across all the configurations. Moreover, FSLoRA demonstrates improved performance as the global rank increases.
This observation confirms that the proposed sketching mechanism enables resource-constrained systems to reap the benefits of a higher global rank, striking an effective balance between efficiency and performance.
%This observation validates that using a large rank for the global LoRA modules results in a more expressive model. Moreover, the proposed sketching mechanism enables resource-constrained systems to effectively benefit from a large rank.
%Furthermore, the proposed sketching mechanism enables resource-constrained systems to reap the benefits of a higher global rank, striking an effective balance between efficiency and performance.

\textbf{Impact of sketching ratio:} We investigate the impact of the sketching ratio on FSLoRA's performance by maintaining a constant global LoRA rank $r=64$ while varying the sketching ratio \(k_i/r\) in the range \(\{0.125, 0.5, 1\}\). 
As shown in Figure \ref{fig:sketching_ratio}, there is a slight performance degradation as the sketching ratio decreases, which is consistent with our theoretical analysis. 
This also reflects an tradeoff: while a larger sketching ratio improves convergence and accuracy, a smaller ratio reduces both computational and communication overhead. Notably, the observed degradation remains minor, highlighting FSLoRA's ability to maintain strong performance even under constrained resources. This demonstrates its effectiveness in balancing efficiency and accuracy, making it well-suited for resource-limited scenarios.

%From Figure \ref{fig:sketching_ratio}, we see that there is a slight performance degradation as the sketching ratio decreases, which aligns with our theoretical analysis. This observation reflects the inherent tradeoff: while a larger sketching ratio enables a better convergence result, a smaller sketching ratio reduces both computational and communication overhead. Notably, the slight performance degradation observed demonstrates FSLoRA's ability to effectively balance efficiency and accuracy, underscoring its advantage in scenarios with limited resources.

\textbf{Further experiments:} 
Results with more clients under broader heterogeneity, as well as with other models, are reported in Appendix \ref{appen:broader_heterogeneity} and Appendix \ref{appen:larger_model}, respectively. Appendix \ref{appen:further_experiments} provides detailed per-task comparisons on the commonsense reasoning benchmark corresponding to Figures \ref{fig:rank_varying} and \ref{fig:global_rank}. The impact of varying the number of local updates $H$ is studied in Appendix \ref{appen:local_updates}, while the empirical extension to dynamic sketching ratios is presented in Appendix \ref{appen:dynamic_sketching_ratio}. Finally, Appendix \ref{appen:compression_integration} demonstrates the synergistic effect of integrating compression with sketching.

%Additional results, including detailed comparisons on each task in the commonsense reasoning benchmark corresponding to Figures \ref{fig:rank_varying} and \ref{fig:global_rank}, the integration of communication compression and sketching, and the experiments with more devices, are provided in Appendix \ref{appen:further_experiments}. 

\section{Conclusion}\label{sec:conclusion}
%We have proposed FSLoRA, a novel on-device collaborative LLM fine-tuning framework. Under the assistance of the sketching mechanism, FSLoRA could maintain high-rank LoRA modules on the server to enhance performance while enabling devices to flexibly update only a subset of parameters through sketching matrices, thus balancing efficiency and effectiveness in heterogeneous federated environments. We present a rigorous theoretical analysis that establishes the convergence behavior of FSLoRA and reveals how the heterogeneous ranks of updated submatrices across edge devices influence the convergence rate. Furthermore, extensive experiments across multiple datasets and models confirm the effectiveness of FSLoRA under and showcase its superior performance compared to baseline methods. A limitation of our work is that we only examine FSLoRA in the context of LLM fine-tuning, leaving its performance on pretraining unexplored and pointing to an interesting direction for future investigation.

We proposed FSLoRA, a novel collaborative LLM fine-tuning framework that introduces a sketching mechanism to enhance training efficiency in resource-constrained systems. By maintaining large-rank LoRA modules on the server and allowing clients to selectively update submatrices based on the sketching ratios, FSLoRA effectively adapts to heterogeneous communication and computational constraints. We provided a rigorous convergence analysis of FSLoRA that characterizes how the sketching ratios (i.e., resource heterogeneity) affect the convergence rate. Finally, we confirmed the effectiveness of FSLoRA through extensive experiments across multiple datasets and models.
%Furthermore, we characterize how the sketching ratios influence the convergence rate. Finally, we confirm the effectiveness of FSLoRA through extensive experiments across multiple datasets and models. 
%A limitation of our work is that we only examine FSLoRA in the context of LLM fine-tuning, leaving its performance on pretraining unexplored and pointing to an interesting direction for future investigation.
%A direction for future work is to extend FSLoRA beyond LLM fine-tuning and explore its performance in pretraining, which remains an open area for further investigation.
%A limitation of this work is the absence of a real-world deployment study that evaluates FSLoRA and baseline methods on the actual server and clients due to hardware constraints.

\section*{Limitations}
One limitation of our work is that, although we provide theory and extensive experiments, the analysis focuses on a standard federated setting and does not yet cover more realistic system effects such as asynchronous participation, unstable connections, client stragglers, and latency fluctuations. In addition, we focus on random-$k$ sketching because it is unbiased and theoretically tractable; more adaptive sketching strategies may be interesting future work if reliable importance measures for LoRA components can be developed.

\section*{Impact Statement}
This paper makes contributions to collaborative LLM fine-tuning by developing a resource-adaptive algorithm for resource-constrained distributed LLM systems. The focus of this work is on the technical advancement of LLM fine-tuning algorithms. While this research has potential societal impacts, it primarily addresses technical challenges and does not necessitate a specific discussion on societal consequences.

\section*{Acknowledgments}

This work was supported in part by the National Science Foundation (NSF) under Grants CNS-2146171, CNS-2212565, and ECCS-2543754; by the Office of Naval Research (ONR) under Grant N00014-22-1-2305; by the Air Force Office of Scientific Research (AFOSR) under Grant FA9550-24-1-0083; and by the National Research Foundation of Korea (NRF) grant funded by the Korea government (MSIT) (No. RS-2026-25470983).

\bibliography{refs}
\bibliographystyle{icml2026}

\onecolumn

\newpage
\appendix

\begin{center}
    {\bf\Large Appendix}
\end{center}

\startcontents[sections]
\printcontents[sections]{l}{1}{\setcounter{tocdepth}{3}}

\newpage

\section{Comparison of Computation, Memory, and Communication}\label{appen:com_memo_comparison}
\textbf{Computation and memory.}
Let $P$ and $q$ denote the memory cost of the full model and the global LoRA module (rank $r$), respectively. The computational cost is expressed with the big O notation. Forward and backward computations, as well as activation memory, are omitted as they are identical across all the considered methods. The results are summarized in Tables  \ref{tab:com_memo_comparison_client} and \ref{tab:com_memo_comparison_server}, where $m$ and $n$ denote the shape of the base model, $k_i$ denotes the LoRA rank for client $i$, $H$ denotes the number of iterations per round, and $N$ is the number of clients. Additionally, the results for the vanilla Federated LoRA, denoted as FedLoRA, are reported under the case of homogeneous LoRA ranks, i,e., $k_i = r$. 

\begin{table}[H]
\caption{
Client-side computation load and memory usage comparison.
}
\label{tab:com_memo_comparison_client}
\centering
%\begin{adjustbox}{width=\textwidth}
\begin{tabular}{@{}lcc@{}}
\toprule
{Method} & {Memory} & {Computation (per round)} \\
\midrule
FedLoRA &
$P + q$ &
$\mathcal{O}(Hr(m + n))$ \\
HeteroLoRA &
$P + \frac{k_i}{r}q$ &
$\mathcal{O}(Hk_i(m + n))$ \\
FlexLoRA &
$P + \frac{k_i}{r}q$ &
$\mathcal{O}(Hk_i(m + n))$ \\
{FLoRA} &
$P + \max\left\{\sum_{i=1}^N \frac{k_i}{r}q, P\right\}$ &
$\mathcal{O}\left(Hk_i(m + n)) + (\sum_{i=1}^N k_i)mn + mn \right)$ \\
FSLoRA &
$P \!+\! \frac{k_i}{r}q$ &
$\mathcal{O}(Hk_i(m + n))$ \\
\bottomrule
\end{tabular}
%\end{adjustbox}
\end{table}

\begin{table}[H]
\caption{
Server-side computation load and memory usage comparison.
}
\label{tab:com_memo_comparison_server}
\centering
%\begin{adjustbox}{width=\textwidth}
\begin{tabular}{@{}lcc@{}}
\toprule
{Method} & {Memory} & {Computation (per round)} \\
\midrule
FedLoRA &
$N q$ &
$\mathcal{O}(N(m + n) r)$ \\
{HeteroLoRA} &
$\sum_{i=1}^N \frac{k_i}{r}q$ &
$\mathcal{O}(N(m + n)r)$ \\
{FlexLoRA} &
$\max\left\{\sum_{i=1}^N \frac{k_i}{r}q, 2P\right\}$ &
$\mathcal{O}\left((\sum_{i=1}^N k_i)mn + Nmn + \min\{m,n\}mn \right)$ \\
{FLoRA} &
$\sum_{i=1}^N \frac{k_i}{r}q$ &
$\mathcal{O}\left((\sum_{i=1}^N k_i)(m + n)\right)$ \\
 {FSLoRA} &
$\sum_{i=1}^N \frac{k_i}{r}q$ &
$\mathcal{O}(N(m + n)r)$ \\
\bottomrule
\end{tabular}
%\end{adjustbox}
\vspace{-3mm}
\end{table}

As shown in Tables  \ref{tab:com_memo_comparison_client} and \ref{tab:com_memo_comparison_server}, FSLoRA matches HetLoRA in both computation and memory cost. FLoRA introduces additional client-side overhead due to merging LoRA modules. FlexLoRA incurs extra server-side costs from conducting SVD on the full model. In summary, FSLoRA guarantees convergence with minimum overhead.

\textbf{Communication.}
We detailed the communication load for baselines and our methods in Table \ref{tab:communication_comp}, where $q$ denotes the communication cost of a global LoRA module with rank $r$, $k_i$ denotes the local LoRA rank for client $i$, $m$ and $n$ denote the shape of the base model, and $N$ denotes the number of clients.

\begin{table}[h]
\centering
\caption{Communication complexity, assuming float 32 parameters and binary sketching indices.}
\begin{tabular}{lccccc}
\toprule
& {FedLoRA} & {HeteroLoRA} & {FlexLoRA} & {FLoRA} & {FSLoRA} \\
\midrule
{Uplink} & $q$ & $\frac{k_i}{r}q$ & $\frac{k_i}{r}q$ & $\frac{k_i}{r}q$ & $\frac{k_i}{r}q$ \\
{Downlink} & $q$ & $q$ & $q$ & $\sum_{i=1}^{N} \frac{k_i}{r}q $ & $q(1+\frac{Nr}{32mn})$ \\
\bottomrule
\end{tabular}
\label{tab:communication_comp}
\vspace{-2mm}
\end{table}
%For uplink, the communication overhead for transmitting updated local LoRA modules are the same for these four algorithms. For downlink, FLoRA needs to broadcast stacked LoRA modules, HeteroLoRA and FlexLoRA need to broadcast the updated global LoRA modules, and FSLoRA needs to broadcast the global LoRA modules and sketching matrix. Note that the extra communication load introduced by transmitting the sketching matrix is negligible compared to global LoRA modules, as it involves only \emph{binary sketching indices} (i.e., the diagonal elements of the sketching matrix). For example, considering the LLaMA-3.2-3B model, under the LoRA setup in our experiment, the global LoRA modules have $66060288$ parameters. This translates to $252$ MB for broadcasting (in float32). When the rank is set to $r=64$ for the global LoRA modules, the sketching indices require only $64$ bits per client (applies to all LoRA layers). Even with 100 clients, sketching indices add just $0.78$ KB ($0.0003\%$ over $252$ MB).

For the uplink, all four heterogeneous LoRA algorithms incur the same communication overhead for transmitting updated local LoRA modules, which is lower than that of FedLoRA.
For the downlink, FLoRA requires broadcasting the stacked LoRA modules, while HeteroLoRA and FlexLoRA broadcast the updated global LoRA modules. FSLoRA, on the other hand, broadcasts both the global LoRA modules and additional sketching matrices.
The extra communication introduced by the sketching matrices is negligible compared to that of the global LoRA modules, as it consists only of \emph{binary sketching indices} (i.e., the diagonal elements of the sketching matrix). For instance, in the case of the LLaMA-3.2-3B-Instruct model under our experimental LoRA configuration, the global LoRA modules contain 66,060,288 parameters, equivalent to approximately 252 MB when using float32. With a global rank of $r = 64$, the sketching indices require only 64 bits per client, covering all LoRA layers. Even with 100 clients, the total sketching overhead is merely 0.78 KB, which accounts for only 0.0003\% of the global LoRA modules.

\textbf{Resource-Efficient Implementation Details.} 
We uses a sparse full-rank formulation for analysis in Section~\ref{sec:fslora_formulation}, but the implementation can be written equivalently in dense local form with submatrix. Consider
\begin{equation}\label{formulation_sparse_form}
l(\mathbf{W}_0 + \mathbf{BSA}), \qquad
\mathbf{BSA} = \frac{r}{k}\sum_{j \in I} \mathbf{B} \mathbf{e}_j \mathbf{e}_j^\top \mathbf{A},
\end{equation}
where $I = \{j_1,\dots,j_k\}$ is the selected index set. Let
$\mathbf{E}_I = [\mathbf{e}_{j_1},\dots,\mathbf{e}_{j_k}]$, $\hat{\mathbf{B}} = \mathbf{B} \mathbf{E}_I$, and $\hat{\mathbf{A}} = \mathbf{E}_I^\top \mathbf{A}$, then $\hat{\mathbf{B}}$ and $\hat{\mathbf{A}}$ are exactly the selected local submatrices in dense form, and
$$
\mathbf{BSA}=\frac{r}{k}\hat{\mathbf{B}} \hat{\mathbf{A}}.
$$
Hence, updating the sparse full-rank parameters is equivalent to updating the dense local submatrices with loss
\begin{equation}\label{formulation_dense_form}
l(\mathbf{W}_0 + \frac{r}{k}\hat{\mathbf{B}} \hat{\mathbf{A}}).
\end{equation}
The forward passes in \eqref{formulation_dense_form} and \eqref{formulation_sparse_form} are identical, and the gradients on the selected entries are also identical because both parameterizations refer to the same variables through indexing. In other words, FSLoRA can be implemented in a dense form with submatrix, which does not require clients to maintain identical fixed-rank local LoRA adapters in the implementation.

\section{Justification for \texorpdfstring{Random-$k$}{Random-k} Sketching}\label{appen:discussion_alternative_sketching}

FSLoRA is built upon Random-$k$ diagonal sketching due to two key properties:
\begin{itemize}[leftmargin=*]
    \item \textbf{Submatrix selection.} Given a sparse diagonal matrix $\mathbf{S}_i$, we have
    \begin{equation}
        \mathbf{BS}_i\mathbf{A} = \sum_{j \in \mathcal{I}_i} s_j \mathbf{b}_j \mathbf{a}_j^\top 
        = [\mathbf{b}_j]_{j \in \mathcal{I}_i} \operatorname{diag}\{s_j\}_{j \in \mathcal{I}_i} [\mathbf{a}_j^\top]_{j \in \mathcal{I}_i},
    \end{equation}
    where \(\mathcal{I}_i\) corresponds to the index set of non-zero diagonal entries of \(\mathbf{S}_i\) and $\mathbf{b}_j$ and $\mathbf{a}_j^\top$ denote the $j$-th column of module $\mathbf{B}$ and the $j$-th row of module $\mathbf{A}$, respectively.
    In other words, with Random-$k$ sketching, only a subset of $\mathbf{B}$'s columns and $\mathbf{A}$'s rows are activated for client $i$. 
    The \emph{sparse diagonal structure} effectively reduces local training cost for each client. 
    \item \textbf{Unbiasedness for convergence.} When $\mathbf{S}_i$ is a Random-$k$ diagonal matrix with $k_i$ nonzero diagonal entries $s_j = \frac{r}{k_i}$, we have 
    \begin{equation}
        \mathbb{E}[\mathbf{BS}_i\mathbf{A}] = \mathbf{BA}.
    \end{equation}
    This unbiasedness is critical for our convergence analysis.
\end{itemize}

Random-$k$ sketching is easy to implement, which ensures uniform expected exposure of all columns and rows of the global LoRA modules to each client, closely mirroring the behavior of the standard LoRA. As training progresses, every component in the global LoRA is repeatedly sampled and thus sufficiently optimized in expectation. In particular, we theoretically and empirically validated its effectiveness in Sections~\ref{sec:analysis} and \ref{sec:experiment}.

In contrast, identifying a principled metric to quantify the importance of individual LoRA rows or columns remains an open problem, even in the centralized LoRA setting. This limitation hinders the reliable development of Top-$k$ or importance-based sketching strategies. Existing approaches typically rely on performing an SVD of the composite matrix $\mathbf{BA}$ (e.g., $\mathbf{BA} = \mathbf{U\Sigma V}^{\top}$) and defining importance scores for the columns of $\mathbf{U}$ and the rows of $\mathbf{V}^{\top}$. However, this departs from our intended setting, as we operate directly in the original $\mathbf{B}$ and $\mathbf{A}$ parameter space.

In addition, we empirically compare Random-$k$ sketching with several heuristic importance-based sketching strategies, as a rigorous importance metric is currently unavailable. Specifically, we consider two heuristic importance scores for each LoRA component $\mathbf{b}_j \mathbf{a}_j^\top$: (i) the spectral norm $\|\mathbf{b}_j\|_2 \cdot \|\mathbf{a}_j\|_2$ inspired by matrix analysis, and (ii) a simpler magnitude-based metric $\|\mathbf{b}_j\|_2 + \|\mathbf{a}_j\|_2$.
Using these scores, we sample (sketch) $k_i$ components from $\{\mathbf{b}_j \mathbf{a}_j^\top \}_{j=1}^r$ with probabilities proportional to the corresponding importance scores. 

The results are reported in Table~\ref{tab:alternative_sketching}. We observe that Random-$k$ sketching consistently outperforms these heuristic importance-based alternatives. A potential reason is that these importance measures do not reliably capture the true contribution of individual components. As a result, importance-based sketching may systematically under-sample certain components that are critical for optimization but poorly reflected by such heuristics, leading to biased training. On the other hand, Random-$k$ sketching provides uniform expected exposure of all rank-one components throughout training, which prevents under-training of any individual column or row.
While more principled importance-based sketching strategies may further improve performance and remain an interesting direction for future work, our current empirical evidence favors Random-$k$ sketching.

\begin{table}[h]
\centering
\caption{Accuracy comparison of Random-$k$ sketching and importance-based sketching under the commonsense reasoning benchmark with the LLaMA-3.2-3B-Instruct model. The experimental settings are the same as those in Table~\ref{tab:llama_het_accuracy}. Random-$k$ sketching achieves better performance.}
\small
%\resizebox{\linewidth}{!}{
\begin{tabular}{lccccccccc}
\toprule
\textbf{Importance metric} & ARC-c & ARC-e & BoolQ & HSwag & OBQA & PIQA & SIQA & Wino & Avg. \\
\midrule
Spectral Norm: $\|a\|_2 \|b\|_2$ & 73.4 & 86.2 & 60.5 & 77.9 & 75.8 & 81.6 & 74.6 & 73.1 & 75.4 \\    
Magnitude Sum: $\|a\|_2 + \|b\|_2$ & 72.1 & 86.5 & 57.4 & 75.3 & 73.5 & 81.3 & 72.6 & 69.9 & 73.6 \\
\rowcolor{ours} Random-$k$              & 75.8 & 86.7 & 69.7 & 81.4 & 80.4 & 83.9 & 76.2 & 78.8 & 79.1 \\
\bottomrule
\end{tabular}
%}
\label{tab:alternative_sketching}
\end{table}

Notably, although our algorithm and analysis are based on Random-$k$ sketching, the framework naturally extends to importance-based sketching when a reliable importance metric is available. We theoretically discussed this extension in Appendix~\ref{appen:extention_importance}.

\section{Justification for the Assumptions in the Convergence Analysis}\label{appen:justification_assumption}

\textbf{Smoothness.} Assumption~\ref{smoothness_assumption}, although global $L$-smoothness may be difficult to guarantee for the LoRA fine-tuning objective, it is reasonable to assume local $L$-smoothness within the neighborhood explored during fine-tuning. Since LoRA updates are typically low-rank and initialized near zero, the optimization trajectory often remains in a local region around the pretrained model, where the loss function can be approximated as smooth.

\textbf{Variance and Data Heterogeneity.}
In the context of LLM fine-tuning, both the magnitude of stochastic gradients and gradients in LLM fine-tuning are in a mild range since:
\begin{itemize}[leftmargin=*]
    \item Transformer-based LLMs use \emph{LayerNorm} and \emph{scaled softmax attention}, which stabilize activations and suppress gradient spikes.
    \item Fine-tuning starts from a well-pretrained model already near a local minimum, leading to smaller gradients.
    \item The fine-tuning dataset does not typically contain strong contradictory signals to what the model already knows, resulting in a relatively flat loss surface. 
\end{itemize}
To further support this empirically, we report the statistics of the expected norm of the stochastic gradients 
$\mathbb{E}\|\nabla_{\mathbf{X}} \tilde{\ell}(\mathbf{X},\xi;\mathbf{S})\|$ 
over approximately $4500$ samples for 4 representative clients with different sketching ratios. The table below reports the minimum and maximum expected norm among 30 randomly sampled model states $\mathbf{X} = [\mathbf{B}; \mathbf{A}]$.

\begin{table}[h!]
\caption{Statistics of the expected norm of stochastic gradients across clients.\label{tab:gradient_norm}}
\centering
\small
\begin{tabular}{ccccc}
\toprule
\textbf{Client ID} & \textbf{Number of samples} & \textbf{Rank} & \textbf{Min} & \textbf{Max} \\
\midrule
1 & 4580 & 4  & 0.1286 & 0.9499 \\
2 & 4216 & 19 & 0.1284 & 0.7069 \\
3 & 4873 & 9  & 0.1237 & 0.5774 \\
4 & 5124 & 32 & 0.1499 & 0.8066 \\
\bottomrule
\end{tabular}
\end{table}

As we can see from Table \ref{tab:gradient_norm}, the expected norm, i.e., $\mathbb{E}_{\xi \sim \mathcal{D}_i, \mathbf{S}\sim \mathcal{S}_i}\|\nabla_{\mathbf{X}} \tilde{\ell}(\mathbf{X},\xi;\mathbf{S})\|$, is in a moderate range. Notably, the variance is upper-bounded by the expected squared gradient norm:
\[
\mathbb{E}_{\xi \sim \mathcal{D}_i, \mathbf{S}\sim \mathcal{S}_i} \|\nabla_{\mathbf{X}} \tilde{\ell}(\mathbf{X},\xi;\mathbf{S}) - \nabla_{\mathbf{X}} f_i^{\mathcal{S}}(\mathbf{X})\|^2 
\leq \mathbb{E}_{\xi \sim \mathcal{D}_i, \mathbf{S}\sim \mathcal{S}_i}\|\nabla_{\mathbf{X}} \tilde{\ell}(\mathbf{X},\xi;\mathbf{S})\|^2.
\]
Therefore, it is generally not hard to find a $\sigma$ and $\rho$ to make Assumption~\ref{varaince_assumption} hold.

On the other hand, we have
\begin{equation}
\begin{aligned}
& \left\| \nabla_{\bf X} f_i^{\mathcal{S}}({\bf X}) \!-\! \nabla_{\bf X} f^{\mathcal{S}}(\bf X) \right\|^2 \\
\leq &  2 \left\| \nabla_{\bf X} f_i^{\mathcal{S}}({\bf X})  \right\|^2 +   2 \left\| \nabla_{\bf X} f^{\mathcal{S}}(\bf X) \right\|^2 \\
\leq & 2 \left\| \nabla_{\bf X} f_i^{\mathcal{S}}({\bf X})  \right\|^2 +   2 \frac{1}{N}\sum_{i=1}^N \left\| \nabla_{\bf X} f_i^{\mathcal{S}}({\bf X}) \right\|^2 \\
\leq & 2 \mathbb{E}_{\xi \sim \mathcal{D}_i, \mathbf{S}\sim \mathcal{S}_i}\|\nabla_{\mathbf{X}} \tilde{\ell}(\mathbf{X},\xi;\mathbf{S})\|^2 +   2 \frac{1}{N}\sum_{i=1}^N \mathbb{E}_{\xi \sim \mathcal{D}_i, \mathbf{S}\sim \mathcal{S}_i}\|\nabla_{\mathbf{X}} \tilde{\ell}(\mathbf{X},\xi;\mathbf{S})\|^2,
\end{aligned}
\end{equation}
where the first inequality follows Cauchy-Schwarz inequality, while the last inequality follows Jensen's inequality.
Thus, the deviation can be controlled by the expected gradient norm, which we empirically found to be moderate (see Table \ref{tab:gradient_norm}). Hence, it is reasonable to impose an upper bound on $\left\| \nabla_{\bf X} f_i^{\mathcal{S}}({\bf X}) \!-\! \nabla_{\bf X} f^{\mathcal{S}}(\bf X) \right\|^2 $ as in Assumption \ref{assumption:gradient_dissimilarity}.

\section{Compatibility of FSLoRA with Secure Aggregation}\label{appen:secure_aggregation}
The aggregation of FSLoRA is compatible with secure aggregation. Taking the aggregation of module $\mathbf{B}$ as an example, we illustrate this below. 

In FSLoRA, client updates are sparse matrices with non-zero values only in columns indexed by $\mathcal{I}_i \subset [r]$, size $|\mathcal{I}_i| = k_i$. With secure aggregation, each client apply \emph{additive masking}:
\begin{equation}
\tilde{\mathbf{B}}_i = \Delta \mathbf{B}_i + \mathbf{R}_i,
\end{equation}
where mask $\mathbf{R}_i$ satisfies $\operatorname{supp}(\mathbf{R}_i) \subseteq (u,v): u \in [m], v \in \mathcal{I}_i$ and $\sum_{i=1}^N \mathbf{R}_i = 0$. That is, the mask has non-zero entries only in the client's active columns, and all masks together sum to zero to preserve correctness. Such masks can be constructed following the classical protocol: for each pair of clients $(i,j)$, define a random matrix
\begin{equation}
\mathbf{M}_{ij} = - \mathbf{M}_{ji} \in \mathbb{R}^{m \times r}, 
\quad \operatorname{supp}(\mathbf{M}_{ij}) \subseteq (u,v) \;|\; u \in [m], \, v \in \mathcal{I}_i \cap \mathcal{I}_j,
\end{equation}
and then construct its total mask
\begin{equation}
\mathbf{R}_i = \sum_{j>i} \mathbf{M}_{ij} - \sum_{j<i} \mathbf{M}_{ji}.
\end{equation}

During uploading, client $i$ sends the masked $k_i$ non-zero columns of $\tilde{\mathbf{B}}_i$, and then the server adds the corresponding padding and averages them as:
\begin{equation}
\sum_{i=1}^N \tilde{\mathbf{B}}_i 
= \sum_{i=1}^N \left( \Delta \mathbf{B}_i + \mathbf{R}_i \right) 
= \sum_{i=1}^N \Delta \mathbf{B}_i,
\end{equation}
which matches the aggregation of module $\mathbf{B}$ in \eqref{model_aggregation}. The aggregation of module $\mathbf{B}$ in FSLoRA is thus compatible with secure aggregation. 

We can draw the same conclusion for module $\mathbf{A}$ under the same derivation.

\section{Evaluation under Broader Heterogeneity and Increased Number of Clients with Partial Participation}\label{appen:broader_heterogeneity}

To accommodate a larger number of clients, we extend FSLoRA (Algorithm~\ref{alo_flora}) to support \emph{partial client participation}. Specifically, at each round, the server samples a subset of clients, distributes the current global LoRA modules to them, and aggregates only the updates from these clients. The detailed algorithm procedure is summarized in Algorithm~\ref{alo_flora_pp}. 

\begin{algorithm}[h]
\caption{Federated Sketching LoRA (FSLoRA) with Partial Participation}
\label{alo_flora_pp}
\begin{algorithmic}[1]
\REQUIRE Base model ${\bf W}_0$, LoRA modules ${\bf B}^0$ and ${\bf A}^0$, learning rate $\gamma$, sketching sets $\{\mathcal{S}_i\}_{i=1}^{N}$, client sampling size $M$
\FOR{$t = 0, 1, \ldots, T-1$}
    \STATE Server samples a participating client set $\mathcal{C}_t \subseteq \{1,\ldots,N\}$ with $|\mathcal{C}_t|=M$
    \STATE Server samples sketching matrices $\{{\bf S}_i^t \sim \mathcal{S}_i\}_{i \in \mathcal{C}_t}$ and sends them to participating clients
    \STATE Server broadcasts the current global LoRA modules $({\bf B}^t,{\bf A}^t)$ to all $i \in \mathcal{C}_t$
    \FORALL{clients $i \in \mathcal{C}_t$ \textbf{in parallel}}
        \STATE Initialize local LoRA modules: ${\bf B}_i^{t,0} \leftarrow {\bf B}^t,\;\; {\bf A}_i^{t,0} \leftarrow {\bf A}^t$
        \FOR{$h = 0, 1, \ldots, H\!-\!1$}
            \STATE Client $i$ updates local LoRA modules: $\begin{bmatrix}
\mathbf{B}_i^{t,h+1};
\mathbf{A}_i^{t,h+1}
\end{bmatrix}
\leftarrow
\begin{bmatrix}
\mathbf{B}_i^{t,h};
\mathbf{A}_i^{t,h}
\end{bmatrix}
-
\gamma 
\begin{bmatrix}
{\Delta} {\bf B}_i^{t,h} ({\bf S}_i^t)^{\top}; 
({\bf S}_i^t)^{\top} {\Delta} {\bf A}_i^{t,h}
\end{bmatrix}$
        \ENDFOR
        \STATE Client $i$ uploads the non-zero columns of $({\bf B}_i^{t,H} - {\bf B}_i^{t,0})$ and the non-zero rows of $({\bf A}_i^{t,H} - {\bf A}_i^{t,0})$
    \ENDFOR
    \STATE Server updates the global LoRA modules: $\begin{bmatrix}
\mathbf{B}^{t+1};
\mathbf{A}^{t+1}
\end{bmatrix}
\leftarrow
\begin{bmatrix}
\mathbf{B}^{t};
\mathbf{A}^{t}
\end{bmatrix}
+
\frac{1}{M} \sum_{i\in \mathcal{C}_t}
\begin{bmatrix}
\mathbf{B}_i^{t,H} - \mathbf{B}_i^{t,0};
\mathbf{A}_i^{t,H} - \mathbf{A}_i^{t,0}
\end{bmatrix}$.
\ENDFOR
\end{algorithmic}
\end{algorithm}

Our theoretical analysis focuses on the full-participation setting to isolate the effect of the sketching mechanism. Since sketching is independent of client sampling, partial participation affects FSLoRA in the same manner as FedAvg. Extending the analysis therefore follows standard techniques widely studied in the FL literature~\citep{yang2021achieving,wang2020tackling}, and we empirically validate FSLoRA in this setting.

Throughout this section, we fix the partial participation size to $M=10$, i.e., 10 clients are sampled in each round.

\subsection{Increasing Resource Heterogeneity and the Number of Clients}\label{appen:more_devices_more_heterogeneity}

We extend our experiments on LLaMA-3.2-3B-Instruct with the commonsense reasoning benchmark to $50$ clients. We adopt Dirichlet-based partitioning for dataset splits.
Specifically, the commonsense reasoning benchmark includes $8$ tasks, and we partitioned them based on the Dirichlet distribution to construct task heterogeneity among $50$ clients. The Dirichlet concentration parameter is set to $\alpha = 0.1$.
We simulate client resource heterogeneity via different LoRA rank distributions (beyond the limited sketching ratio considered in Section \ref{sec:experiment}). More capable clients are assigned higher ranks, reflecting varying compute capacities. We consider two different rank distributions: normal and heavy-tail distributions in the range $[4, 64]$.

\textbf{Normal distribution:} Ranks are sampled from a normal distribution with mean $\mu = \frac{a + b}{2}$ and standard deviation $\sigma = \frac{b - a}{6}$, where $a=4$ and $b=64$.
This models a balanced distribution of client capabilities centered around the middle of the range.

%\textbf{Uniform distribution:} Local ranks are uniformly sampled from the range $[4, 64]$, simulating a scenario where client capabilities are evenly distributed across a fixed range.

\textbf{Heavy-tail distribution:} We sample ranks using an inverse log-normal distribution. Specifically, we draw $x_i \sim \text{LogNormal}(\mu, \sigma)$ with $\mu = \log\left(\frac{a + b}{4}\right)$ and $\sigma = 1.0$, then set $k_i = 1/x_i$ and apply min-max normalization to scale values into the range $[a, b]$. This results in a heavy-tailed distribution where most clients receive low ranks, reflecting a scenario with many low-capability clients and a few high-capability ones.

\begin{table}[h]
\centering
\caption{Accuracy comparison under different client heterogeneity settings. FSLoRA outperforms baseline methods across both normal and heavy-tail LoRA rank distributions.}
\small
%\resizebox{\linewidth}{!}{
\begin{tabular}{llccccccccc}
\toprule
{Rank setup} & Method & ARC-c & ARC-e & BoolQ & HellaSwag & OBQA & PIQA & SIQA & WinoGrande & Avg. \\
\midrule

\multirow{4}{*}{{Normal}} 
& HeteroLoRA     & 73.38 & 85.82 & 62.17 & 71.23 & 77.40 & 80.14 & 74.72 & 72.53 & 74.67 \\
& FlexLoRA       & 74.23 & 87.84 & 68.37 & 79.77 & 76.00 & 82.97 & 75.90 & 78.13 & 77.90 \\
& FLoRA   & 68.17 & 83.75 & 64.93 & 75.67 & 71.40 & 77.20 & 71.24 & 70.09 & 72.81 \\
\rowcolor{ours}& {FSLoRA}       & {75.77} & {86.95} & {69.67} & {81.53} & {80.60} & {84.06} & {76.20} & {78.85} & {79.20} \\

\midrule

\multirow{4}{*}{{Heavy-tail}} 
& HeteroLoRA     & 72.44 & 86.78 & 63.60 & 73.10 & 72.00 & 81.34 & 71.65 & 68.75 & 73.71 \\
%& FlexLoRA       & 72.95 & 86.74 & 62.13 & 75.43 & 77.40 & 80.90 & 74.72 & 72.85 & 75.39 \\
& FlexLoRA       & 73.04 & 86.70 & 62.23 & 75.57 & 78.00 & 81.12 & 74.77 & 73.32 & 75.59 \\
& FLoRA   & 67.92 & 81.90 & 64.90 & 72.77 & 74.00 & 80.41 & 75.28 & 70.24 & 73.43 \\
\rowcolor{ours}& {FSLoRA}       & {75.77} & {86.70} & {69.67} & {81.40} & {80.40} & {83.90} & {76.15} & {78.77} & {79.10} \\

\bottomrule
\end{tabular}
%}
\label{tab:divice_het}
\end{table}

\begin{comment}
\begin{table*}[h]
\centering
\caption{Accuracy comparison on commonsense reasoning benchmarks under different data heterogeneity settings. FSLoRA maintains its advantage over the baselines as the data heterogeneity increases.}
\small
\resizebox{\linewidth}{!}{
\begin{tabular}{llccccccccc}
\toprule
{Data setup} & Method & ARC-c & ARC-e & BoolQ & HellaSwag & OBQA & PIQA & SIQA & WinoGrande & Avg. \\
\midrule

\multirow{4}{*}{Dir(0.5)} 
& HeteroLoRA     & 70.3 & 86.5 & 65.9 & 70.0 & 74.0 & 82.1 & 71.5 & 68.0 & 73.5 \\
& FlexLoRA       & 70.1 & 85.3 & 65.8 & 68.8 & 71.4 & 80.8 & 68.9 & 69.3 & 72.5 \\
& FLoRA   & 68.3 & 85.8 & 65.5 & 68.5 & 69.0 & 74.9 & 66.5 & 69.1 & 71.0 \\
& {FSLoRA} & {72.7} & {87.2} & {67.2} & {72.1} & {76.8} & {81.4} & {74.0} & {71.0} & {75.3} \\

\midrule

\multirow{4}{*}{Dir(0.1)} 
& HeteroLoRA     & 69.9 & 85.0 & 65.3 & 69.0 & 72.3 & 81.1 & 70.7 & 67.3 & 72.6 \\
& FlexLoRA       & 69.9 & 84.0 & 65.1 & 68.0 & 70.6 & 79.5 & 68.7 & 68.7 & 71.8 \\
& FLoRA   & 67.6 & 84.8 & 64.5 & 67.8 & 67.8 & 73.8 & 65.9 & 68.0 & 70.0 \\
& {FSLoRA} & {72.0} & {86.2} & {67.2} & {71.3} & {75.6} & {80.8} & {73.4} & {70.6} & {74.6} \\

\bottomrule
\end{tabular}
}
\label{tab:data_het}
\end{table*}
\end{comment}

As shown in Table \ref{tab:divice_het}, FSLoRA outperforms other methods under both heterogeneity settings. As we move from normal to heavy-tail, where more clients are low-resource, overall performance decreases for all methods. However, FSLoRA exhibits the smallest drop, demonstrating stronger robustness to extreme client heterogeneity.

%Figure \ref{fig:comvergence_comparison_device} compares the convergence behavior of FSLoRA and three baseline methods in different device heterogeneity scenarios. Notably, in the heavy-tail distribution case where most devices are low-resource, FSLoRA maintains stable and robust convergence while baseline methods suffer significant degradation. 

\begin{figure}[h]
    \centering
    \includegraphics[width=0.8\linewidth]{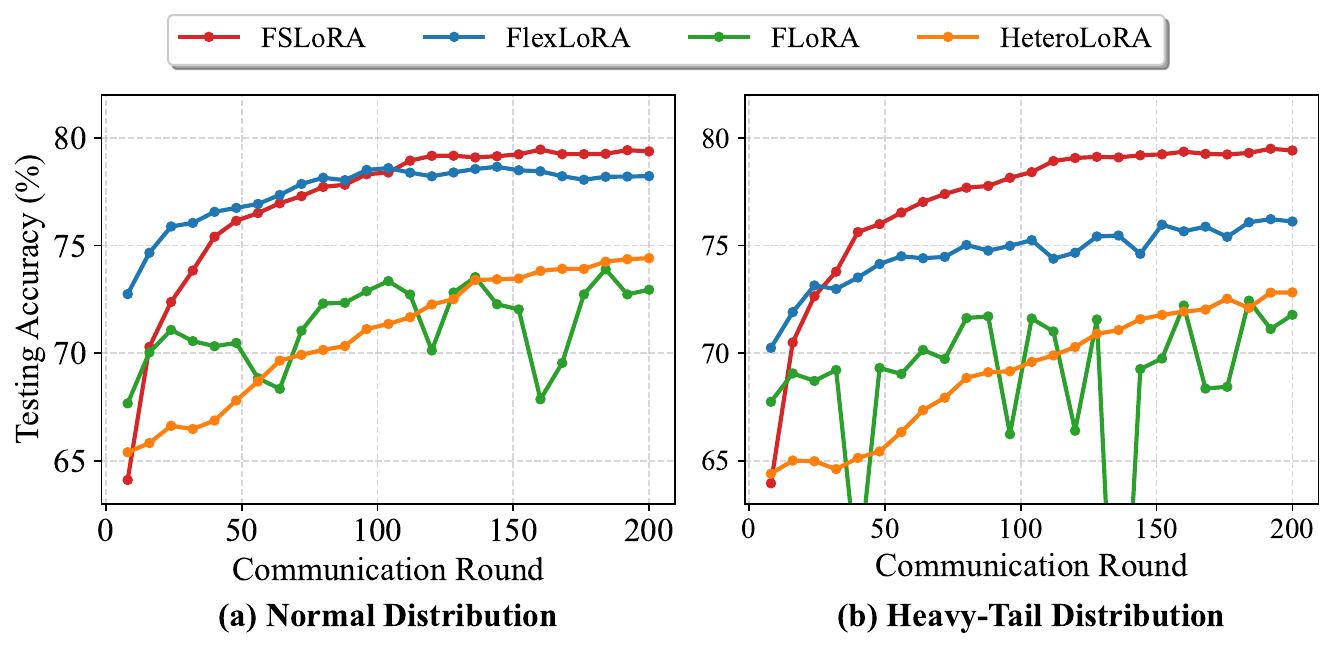}
    %\vspace{-5mm}
    \caption{Convergence behavior of FSLoRA and baselines on the commonsense reasoning benchmark with the LLaMA-3.2-3B-Instruct model. Notably, FSLoRA's per-round communication cost is at most equal to the baselines (as detailed in Appendix \ref{appen:com_memo_comparison}). Testing accuracy is averaged over eight tasks.
    }
    %\vspace{-5mm}
    \label{fig:comvergence_comparison_device}
\end{figure}
In Figure \ref{fig:comvergence_comparison_device}, we compare the convergence behavior of FSLoRA and three baseline methods under the aforementioned two types of client heterogeneity. Under the normal distribution, FlexLoRA exhibits fast initial progress but falls behind FSLoRA in final accuracy, likely due to approximation errors introduced by truncated SVD. This issue is exacerbated in the heavy-tail distribution, where low-rank clients dominate and SVD truncation causes greater distortion, further degrading FlexLoRA's performance. Similarly, HeteroLoRA's reliance on zero-padding reduces optimization efficiency, preventing it from achieving higher accuracy. 
FLoRA fails to show steady improvement as communication progresses. One potential reason is that frequent model merging and random reinitialization of LoRA modules in each round disrupt the convergence continuity. %Notably, the official implementation of FLoRA also recommends using a small number of communication rounds \cite{wang2024flora}.
%FLoRA struggles to converge in both settings. 
%one potential reason is the frequent random reinitialization of LoRA modules at each round, which disrupts convergence continuity.
In contrast, FSLoRA demonstrates robust and stable convergence across both scenarios, achieving the highest overall accuracy.

\subsection{Further Increasing the Number of Clients}
%We further evaluated the performance of FSLoRA by increasing the number of devices from $50$ to $100$. The results are presented in Table \ref{tab:more_devices}. In this setting, local ranks follow a heavy-tailed distribution as described in the previous subsection, and all other experimental configurations remain unchanged. As shown in the table, FSLoRA maintains stable performance across all benchmarks when scaling to more devices, with only a slight drop in average accuracy. This minor decline can be attributed to increased data dispersion across a larger number of devices, which potentially increases the difficulty of local training.

We further evaluated the performance of FSLoRA by increasing the number of clients to $100$. The results are presented in Table \ref{tab:more_devices}. In this setting, local ranks follow a heavy-tailed distribution as described in the previous subsection, and all other experimental configurations remain unchanged. As shown in the table, FSLoRA maintains its advantage in terms of the average performance when scaling to more clients.

\begin{table}[h]
\centering
\caption{Accuracy comparison when the number of clients is $N=100$. FSLoRA maintains its advantage in terms of the average accuracy.}
\small
%\resizebox{\linewidth}{!}{
\begin{tabular}{lccccccccc}
\toprule
Method & ARC-c & ARC-e & BoolQ & HellaSwag & OBQA & PIQA & SIQA & WinoGrande & Avg. \\
\midrule

 HeteroLoRA     & 71.76 & 86.24 & 62.57 & 68.07 & 76.60 & 79.38 & 74.10 & 69.69 & 73.55 \\
 FlexLoRA       & 73.38 & 87.54 & 69.03 & 75.27 & 78.60 & 80.47 & 74.16 & 73.80 & 76.53 \\
 FLoRA   & 69.97 & 83.25 & 67.10 & 71.67 & 73.60 & 78.94 & 72.21 & 70.80 & 73.44 \\
\rowcolor{ours} {FSLoRA}  & {74.40} & {87.54} & {70.13} & {79.90} & {79.40} & {83.57} & {76.51} & {78.93} & {78.80} \\
\bottomrule
\end{tabular}
%}
\label{tab:more_devices}
\end{table}

\subsection{Varying the Level of Data Heterogeneity}

In Table \ref{tab:data_het}, we investigate the impact of the degree of data heterogeneity on performance. We increase the heterogeneity by decreasing the Dirichlet concentration parameter from $\alpha = 1$ to $\alpha = 0.1$. 
The local ranks follow the heavy-tail distribution described in the previous subsection, and all other experimental configurations remain consistent with Appendix \ref{appen:more_devices_more_heterogeneity}. 
As observed from Table \ref{tab:data_het}, the performance of all methods degrades as heterogeneity increases. FSLoRA consistently achieves higher accuracy.

\begin{table}[h]
\centering
\caption{Accuracy comparison under different data heterogeneity settings. FSLoRA maintains its advantage over the baselines as the data heterogeneity increases. The number of clients is set to $50$.}
\small
%\resizebox{\linewidth}{!}{
\begin{tabular}{llccccccccc}
\toprule
{Data setup} & Method & ARC-c & ARC-e & BoolQ & HellaSwag & OBQA & PIQA & SIQA & WinoGrande & Avg. \\
\midrule

\multirow{4}{*}{Dir(1)} 
& HeteroLoRA     & 72.18 & 86.11 & 62.57 & 73.10 & 77.60 & 79.82 & 74.26 & 69.46 & 74.39 \\
& FlexLoRA       & 74.06 & 87.25 & 65.67 & 74.90 & 78.80 & 81.01 & 74.16 & 74.27 & 76.27 \\
& FLoRA   & 70.14 & 83.29 & 67.27 & 71.60 & 73.60 & 78.73 & 72.16 & 70.96 & 73.47 \\
\rowcolor{ours}& {FSLoRA}       & {75.85} & {87.50} & {70.93} & {81.47} & {81.00} & {82.86} & {76.66} & {78.53} & {79.35} \\

\midrule

\multirow{4}{*}{Dir(0.1)} 
& HeteroLoRA     & 72.44 & 86.78 & 63.60 & 73.10 & 72.00 & 81.34 & 71.65 & 68.75 & 73.71 \\
%& FlexLoRA       & 72.95 & 86.74 & 62.13 & 75.43 & 77.40 & 80.90 & 74.72 & 72.85 & 75.39 \\
& FlexLoRA       & 73.04 & 86.70 & 62.23 & 75.57 & 78.00 & 81.12 & 74.77 & 73.32 & 75.59 \\
& FLoRA   & 67.92 & 81.90 & 64.90 & 72.77 & 74.00 & 80.41 & 75.28 & 70.24 & 73.43 \\
\rowcolor{ours}& {FSLoRA}       & {75.77} & {86.70} & {69.67} & {81.40} & {80.40} & {83.90} & {76.15} & {78.77} & {79.10} \\

\bottomrule
\end{tabular}
%}
\label{tab:data_het}
\end{table}

% \begin{figure}[h]
%     \centering
%     \includegraphics[width=0.8\linewidth]{Figure/convergence_comparison_data.pdf}
%     %\vspace{-5mm}
%     \caption{Convergence behavior of FSLoRA and baselines on the commonsense reasoning benchmark with the LLaMA-3.2-3B model. Testing accuracy is averaged over eight tasks.
%     }
%     %\vspace{-5mm}
%     \label{fig:comvergence_comparison_data}
% \end{figure}

% In addition, Figure \ref{fig:comvergence_comparison_data} presents the convergence behavior of FSLoRA and baseline methods under varying degrees of data heterogeneity. 

\section{Experiments on Qwen2.5-1.5B-Instruct and LLaMA-7B}\label{appen:larger_model}

We further extended our evaluation to the Qwen model. Specifically, we fine-tuned Qwen2.5-1.5B-Instruct on the commonsense reasoning benchmark using the same setup as in Table~\ref{tab:llama_het_accuracy}. As shown in Table~\ref{tab:qwen_1_5b}, FSLoRA achieves competitive performance compared with the heterogeneous LoRA baselines, suggesting its effective across different model architectures.
\begin{table}[h]
\centering
\caption{Performance comparison on the Qwen2.5-1.5B-Instruct commonsense reasoning benchmark. FSLoRA consistently outperforms the baselines, demonstrating robust effectiveness across model architectures.}
\small
%\resizebox{\linewidth}{!}{
\begin{tabular}{lccccccccc}
\toprule
Method & ARC-c & ARC-e & BoolQ & HellaSwag & OBQA & PIQA & SIQA & WinoGrande & Avg. \\
\midrule
 HeteroLoRA & 70.36 & 86.53 & 65.50 & 68.55 & 79.01 & 74.93 & 68.25 & 68.04 & 72.65 \\
FlexLoRA   & 73.41 & 88.07 & 63.32 & 71.30 & 76.44 & 78.02 & 70.94 & 68.56 & 73.76 \\
FLoRA      & 70.28 & 84.51 & 60.90 & 68.82 & 69.58 & 73.45 & 67.53 & 66.27 & 70.17 \\
\rowcolor{ours} {FSLoRA}  & {75.82} & {88.16} & {65.08} & {74.93} & {79.82} & {80.72} & {72.62} & {71.29} & {76.06} \\
\bottomrule
\end{tabular}
%}
\label{tab:qwen_1_5b}
\end{table}
\begin{table}[h]
\centering
\caption{Performance comparison on the LLaMA-7B model across Wizard, Dolly-15k, and Alpaca datasets. FSLoRA achieves the highest average performance, validating its scalability to larger and more complex models.}
\begin{tabular}{lcccc}
\toprule
{Method} & {Wizard} & {Dolly-15k} & {Alpaca} & {Avg} \\
\midrule
HeteroLoRA         & 27.15           & 26.70              & 28.74           & 27.53        \\
FlexLoRA         & 28.25           & 35.60             & 30.40           & 31.42        \\
FLoRA    & 27.91           & 28.50              & 29.54           & 28.65        \\
\rowcolor{ours} FSLoRA & 30.33  & 40.79     & 30.68  & 33.93 \\
\bottomrule
\end{tabular}
\label{tab:llama7b_results}
\end{table}

Although our primary focus is on models suitable for client-side deployment, such as RoBERTa and LLaMA-3.2-3B-Instruct models, we also include experiments on the larger LLaMA-7B model to demonstrate the scalability of FSLoRA in more complex models. Specifically, we fine-tune the LLaMA-7B model on the Wizard, Dolly-15k, and Alpaca datasets and evaluate it on $1444$ MMLU samples (available at: https://github.com/ATP-1010/FederatedLLM). For Wizard and Dolly-15k, we adopt the same heterogeneous data partitioning as \citep{wang2024flora}. Since the Alpaca dataset lacks a clear task or domain structure, we apply a uniform random partitioning strategy to distribute the data across clients. We tune the \text{q\_proj} and \text{v\_proj} modules and set the local LoRA ranks $k_i\!=\![64, 32, 16, 16, 8, 8, 4, 4, 4, 4]$ for $10$ clients. The parameter settings are aligned with those in \citep{wang2024flora}. 
%Table \ref{tab:llama7b_results} reports the results, with FSLoRA numbers sourced directly from \citep{wang2024flora}.
As shown in Table \ref{tab:llama7b_results}, FSLoRA achieves the highest average performance across all three datasets compared to baselines. These results demonstrate FSLoRA's potential for effective fine-tuning with the large-scale LLaMA-7B model under heterogeneous client settings.

\section{Further Experiments}

In this section, we provide additional results, including detailed per-task comparisons from the commonsense reasoning benchmark corresponding to Figures \ref{fig:rank_varying} and \ref{fig:global_rank}. In addition, we further investigate the impact of the number of local updates $H$ on the convergence, the robustness of FSLoRA under dynamic sketching ratio, and the integration of communication compression and sketching.

%compare the communication costs required to achieve a preset testing accuracy

\subsection{Further Details on the Ablation Study}\label{appen:further_experiments}
\begin{figure}[h]
    \centering
    \includegraphics[width=\linewidth]{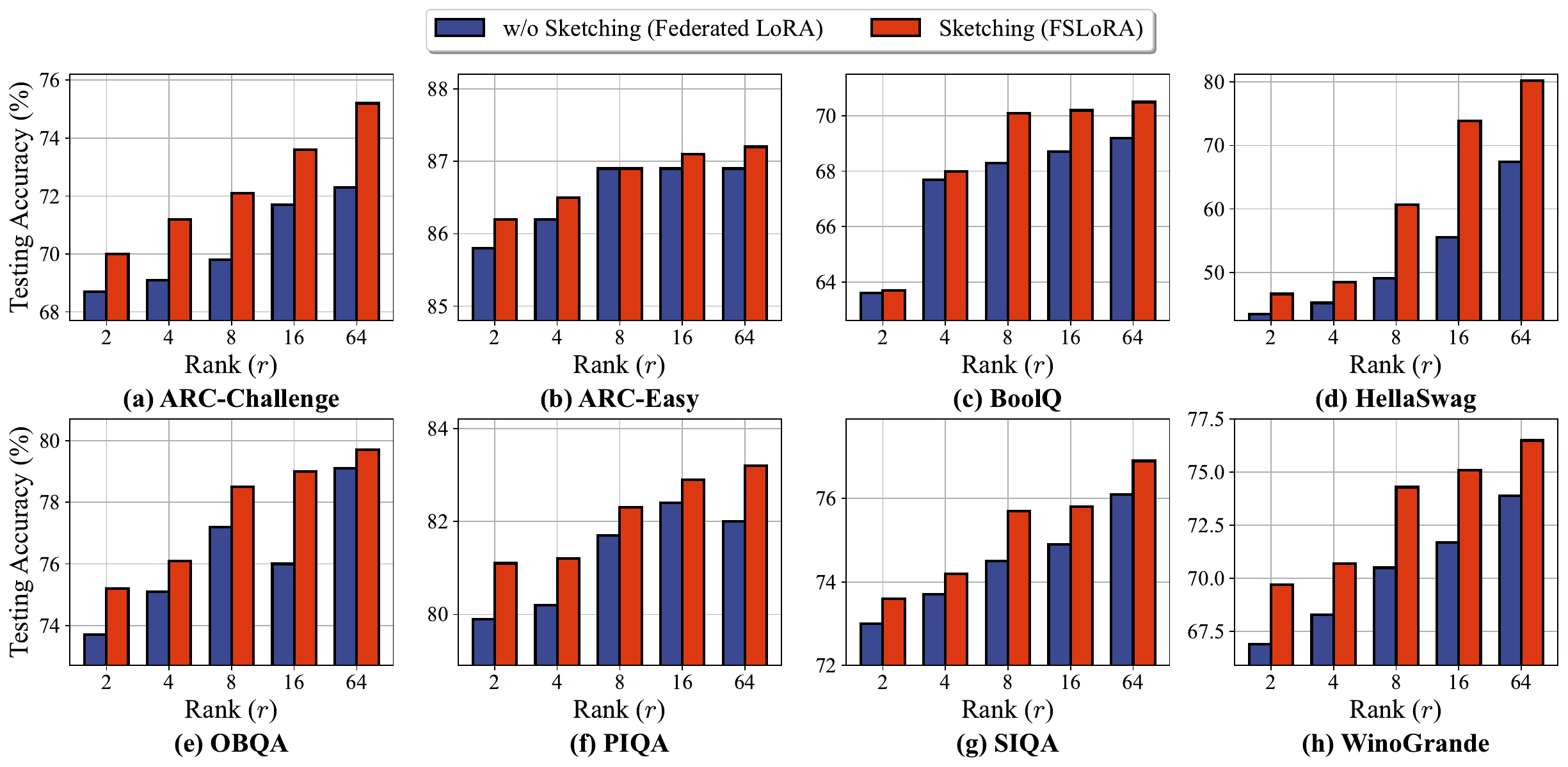}
    \caption{Comparison of FSLoRA with and without sketching, with an upload budget \(200 \times\) the global LoRA module size at each rank. This is based on the commonsense reasoning benchmark and the LLaMA-3.2-3B-Instruct model. We observe that the sketching mechanism improves performance across all considered tasks. The average accuracy of the eight tasks is shown in Figure \ref{fig:rank_varying}.}
    \label{fig:llama_detail_fix_budget}
\end{figure}

\textbf{Impact of sketching:} In Figure \ref{fig:llama_detail_fix_budget}, we compare the performance of FSLoRA with and without sketching on eight tasks from the commonsense reasoning benchmark using the LLaMA-3.2-3B-Instruct model. Notably, FSLoRA without sketching is equivalent to the vanilla Federated LoRA. For FSLoRA with sketching, we apply a uniform sketching ratio of \(k_i/r = 0.5\) across all distributed clients. 
The uploading budget for each client is set to \(200\) times the size of the full global LoRA modules at the corresponding rank. It is clear that FSLoRA with sketching consistently outperforms its non-sketched counterpart across these eight tasks, demonstrating the effectiveness of sketching in improving performance.

\begin{figure}[h!]
    \vspace{-2mm}
    \centering
    \includegraphics[width=\linewidth]{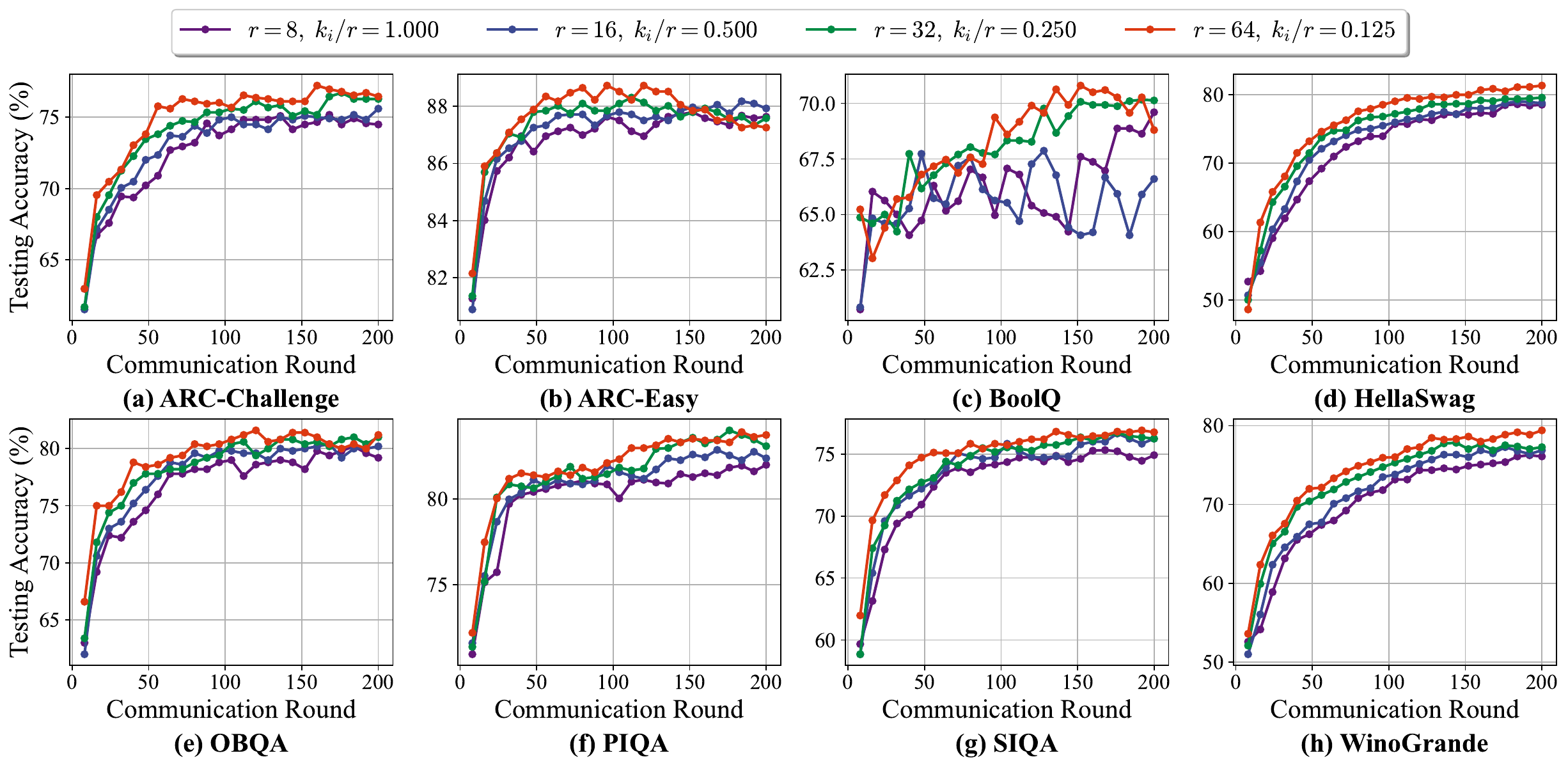}
    \caption{Impact of the rank of global LoRA modules on FSLoRA, given a fixed rank for the updated submatrices at the clients. This is based on the commonsense reasoning benchmark and the LLaMA-3.2-3B-Instruct model. Overall, FSLoRA demonstrates improved performance as the global rank increases. The average accuracy of the eight tasks is shown in Figure \ref{fig:global_rank}.
    }
    \label{fig:llama_detail_global_rank}
    %\vspace{-4mm}
\end{figure}

\textbf{Impact of the global rank:} In Figure \ref{fig:llama_detail_global_rank}, we present the impact of the rank of the global LoRA modules on FSLoRA's performance across eight tasks from the commonsense reasoning benchmark. We consider four configurations: 1) \(r = 8,~ k_i/r = 1\), 2) \(r = 16,~ k_i/r = 0.5\), 3) \(r = 32,~ k_i/r = 0.25\), and 4) \(r = 64,~ k_i/r = 0.125\). The rank of submatrices updated by the clients at each iteration remains consistent across all configurations (i.e., \(k_i = 8\)), ensuring that the communication and computational resources on the client side are kept fixed for all cases. In the ARC-Easy task, performance decreases as the rank increases to $64$, potentially due to overfitting. 
%We see that the third subfigure exhibits oscillations as the sketching ratio increases. One potential explanation for this behavior is that the BoolQ task may be more sensitive to variations in the sketching ratio. 
In general, FSLoRA shows improved performance as the rank increases. 

%As shown in Figure \ref{fig:llama_detail_global_rank}, FSLoRA demonstrates improved performance overall as the global rank increases, highlighting that a larger rank for the global LoRA modules leads to a more expressive model. Furthermore, the proposed sketching mechanism enables resource-constrained systems to effectively leverage the benefits of a large global rank. We also see that the third subfigure (\textbf{BoolQ}) exhibits significant oscillations as the sketching ratio increases. One potential explanation for this behavior is that the \textbf{BoolQ} task may be more sensitive to variations in the sketching ratio, potentially due to its specific data distribution or the nature of the task requiring a more stable representation. This suggests that sensitivity to the sketching ratio may vary across tasks and should be taken into account when designing global rank configurations for different tasks.
\subsection{Impact of Local Updates}\label{appen:local_updates}
Based on the commonsense reasoning benchmark and the LLaMA-3.2-3B-Instruct model, we evaluated the convergence behavior of FSLoRA under varying numbers of local updates (i.e., $H$). The experimental results are presented in Figure~\ref{fig:impact_H}. In the low-to-moderate regime of local updates (i.e., $H \in {10, 20, 100}$), FSLoRA demonstrates a clear acceleration in convergence as $H$ increases. For example, moving from $H=10$ to $H=20$ substantially reduces the number of communication rounds required to reach the same level of testing accuracy, while further increasing $H$ to $100$ yields even faster progress toward convergence. This observation indicates that a moderate increase in local updates allows clients to improve communication efficiency. However, when the number of local updates is pushed beyond this range (e.g., $H=200$), no additional convergence gain is observed. These findings align well with our theoretical analysis in Section \ref{sec:analysis}, which shows that FSLoRA can achieve a convergence speedup under certain conditions on $H$. 
\begin{figure}[h!]
    \centering
    \includegraphics[width=0.4\linewidth]{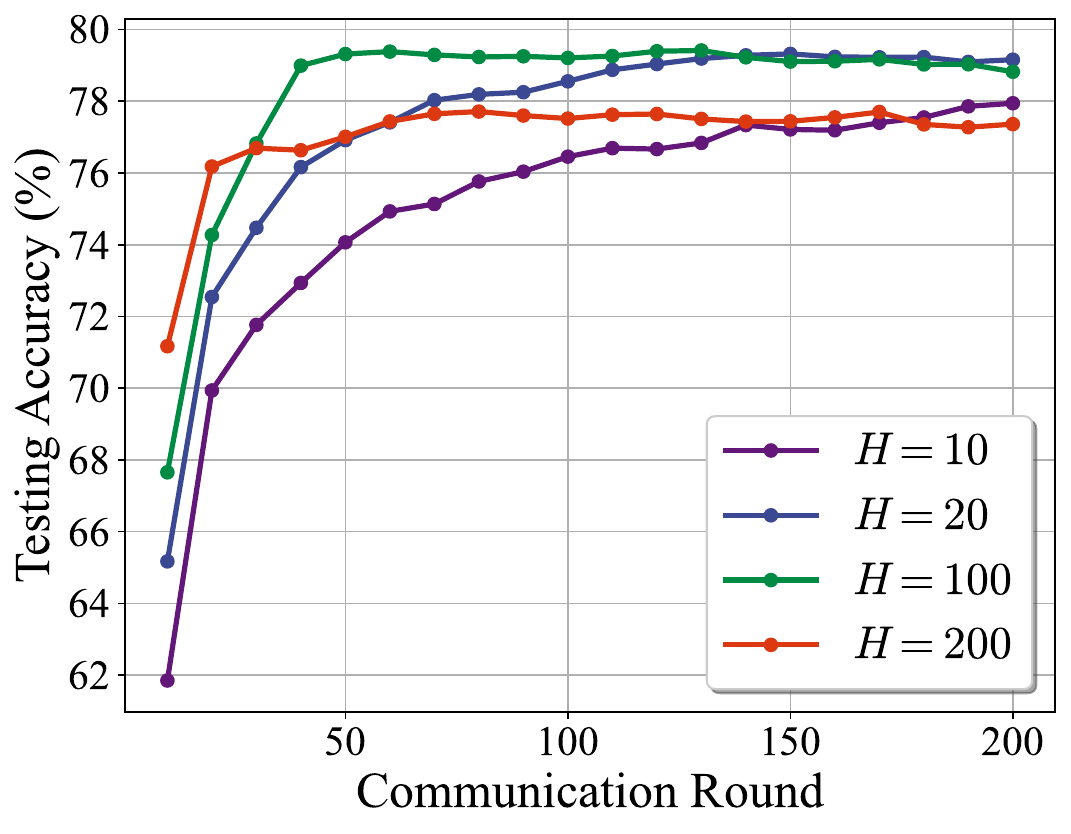}
    \caption{Impact of the number of local updates on FSLoRA's convergence. This is based on the commonsense reasoning benchmark and the LLaMA-3.2-3B-Instruct model. In a certain range, i.e., from 10 to 100, FSLoRA achieves a fast convergence as $H$ increases.
    }
\label{fig:impact_H}
\end{figure}

\subsection{Dynamic Sketching Ratios}\label{appen:dynamic_sketching_ratio}

While our primary focus is on developing a heterogeneous federated LoRA method under a standard static setup, following prior works \citep{wang2024flora,cho2024heterogeneouslorafederatedfinetuning,bai2024federatedfinetuninglargelanguage}, the proposed FSLoRA algorithm can be naturally extended to dynamic, time-varying resource environments. The modification is straightforward: we allow the sparsity levels, corresponding to the sketching ratios of FSLoRA, of the matrices in the sketching set $\mathcal{S}_i$ in Algorithm~\ref{alo_flora} to become time-varying, while keeping the remaining steps unchanged. 

\begin{table}[h]
\centering
\caption{The performance of FSLoRA under static and dynamic sketching ratios. This is based on the commonsense reasoning benchmark and the LLaMA-3.2-3B-Instruct model. FSLoRA maintains comparable performance when moving from the static to the dynamic case.}
\label{tab:dynamic_ratios}
\begin{tabular}{lccccccccc}
\toprule
Ratios & ARC-c & ARC-e & BoolQ & HSwag & OBQA & PIQA & SIQA & Wino & Avg. \\
\midrule
Static  & 76.1 & 87.1 & 70.0 & 81.7 & 81.4 & 82.6 & 76.4 & 78.9 & 79.3 \\
\rowcolor{ours} Dynamic & 75.5 & 87.7 & 69.2 & 81.3 & 81.2 & 82.2 & 76.0 & 78.8 & 79.0 \\
\bottomrule
\end{tabular}
\end{table}

We empirically validate the effectiveness of FLoRA under this dynamic setting. In the simulation, we group clients into three capability levels \emph{low}, \emph{medium}, and \emph{high}, assigned sketching ratio ranges $[0.125, 0.25]$, $[0.25, 0.5]$, and $[0.5, 1.0]$, respectively, to balance local training latencies across groups. Within each range, the sketching ratios are allowed to vary dynamically. The results, reported in Table~\ref{tab:dynamic_ratios}, show that FSLoRA maintains comparable performance when moving from the static to the dynamic case, demonstrating its robustness under time-varying sketching ratios.
\subsection{Comparison with FFA-LoRA}

FFA-LoRA~\citep{sun2024improving} freezes the LoRA matrix $\mathbf{A}$ and tunes only matrix $\mathbf{B}$, which reduces training cost. However, FFA-LoRA requires all clients to use the same LoRA rank as the server, and therefore does not apply to our targeted heterogeneous-rank setting, where clients may use different local ranks. Hence, we compare the proposed FSLoRA with FFA-LoRA under the homogeneous setting based on Llama-3.2-3B. Since FFA-LoRA requires the client rank to be the same as the server rank, while FSLoRA allows different local and global ranks, we carefully configured the local/global rank combinations so that FSLoRA, FFA-LoRA, and FedLoRA (i.e., vanilla federated LoRA) have the same number of locally trainable parameters for a fair comparison.

\begin{table}[t]
\centering
\caption{Comparison with FFA-LoRA under the homogeneous-rank setting on Llama-3.2-3B}
\label{tab:performance_comparison_FFA_LoRA}
\begin{tabular}{lccccccccc}
\toprule
Method & ARC-c & ARC-e & BoolQ & HellaSwag & OBQA & PIQA & SIQA & WinoGrande & Avg. \\
\midrule
FFA-LoRA, $k=8; r=8$ & 68.17 & 84.47 & 64.07 & 64.20 & 69.20 & 78.51 & 67.40 & 63.30 & 69.92 \\
FedLoRA, $k=4; r=4$ & 74.66 & 87.33 & 67.57 & 77.53 & 77.26 & 81.50 & 75.49 & 75.93 & 77.16 \\
\rowcolor{ours} FSLoRA, $k=4; r=8$ & 75.00 & 87.75 & 68.90 & 77.70 & 77.60 & 82.32 & 75.69 & 76.16 & 77.64 \\
\bottomrule
\end{tabular}
\end{table}

As shown in Table~\ref{tab:performance_comparison_FFA_LoRA}, FSLoRA outperforms FFA-LoRA under the same local budget in the homogeneous setting. The potential reason is that freezing one LoRA factor limits the expressive power of the adaptation module. In contrast, FSLoRA randomly selects partial rows and columns for training in each round, so that all rows and columns are trained with approximately uniform exposure across clients over the course of training.

\subsection{Integration of Sketching and Top-k Compression}\label{appen:compression_integration} 
Building on the commonsense reasoning benchmark and the LLaMA-3.2-3B-Instruct model, we further explore the integration of two orthogonal techniques, sketching and top-k compression, to further reduce the uplink communication overhead of clients in FSLoRA. 

\begin{figure}[h]
    \centering
    \includegraphics[width=0.4\linewidth]{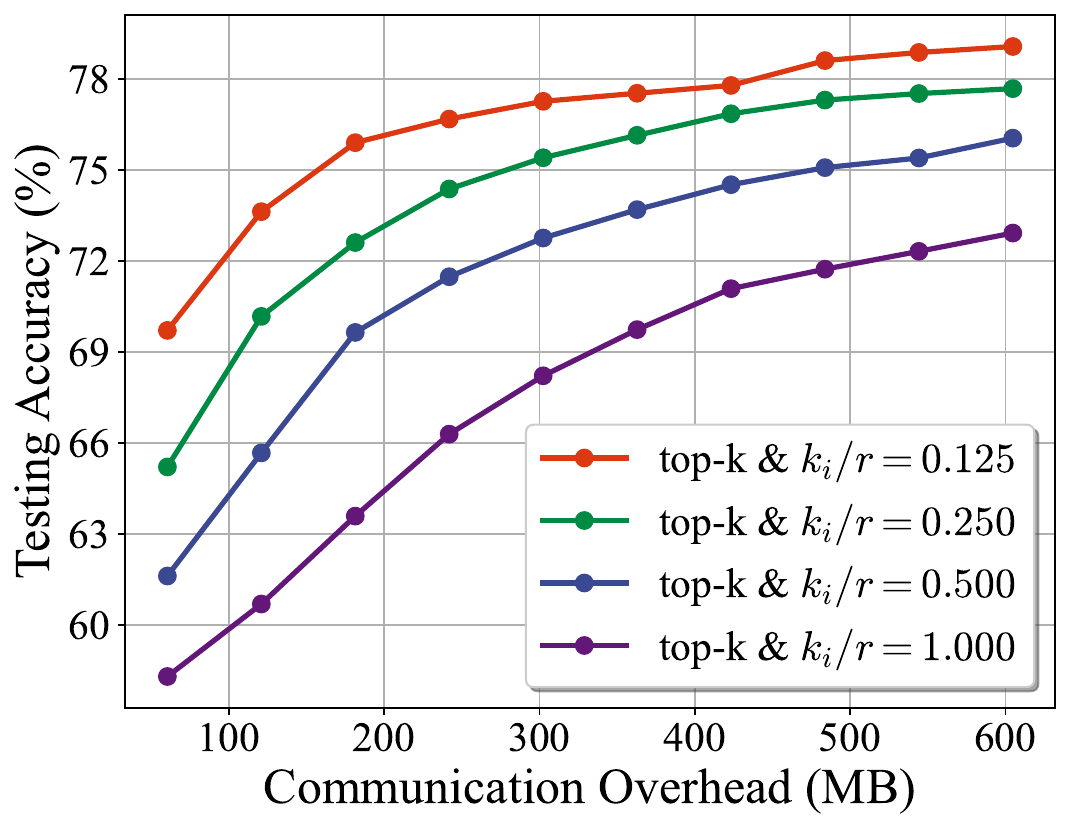}
    \caption{Testing accuracy versus communication overhead using float 32 precision. Lower sketching ratios achieve higher accuracy at the same communication cost, demonstrating that combining sketching with top-$k$ compression leads to more communication-efficient training.
    }
\label{fig:integration_compression_sketching}
\end{figure}

Specifically, with sketching, each client activates and updates submatrices of the global LoRA weights, 
$[\mathbf{b}_j]_{j \in \mathcal{I}_i}, [\mathbf{a}_j^\top]_{j \in \mathcal{I}_i}$, which are selected at the beginning of each round:
\begin{equation}
\mathbf{BS}_i\mathbf{A} = \sum_{j \in \mathcal{I}_i} \frac{r}{k_i} \mathbf{b}_j \mathbf{a}_j^\top 
= \frac{r}{k_i} [\mathbf{b}_j]_{j \in \mathcal{I}_i} [\mathbf{a}_j^\top]_{j \in \mathcal{I}_i},
\end{equation}
where $\mathbf{b}_j$ and $\mathbf{a}_j^\top$ denote the $j$-th column of module $\mathbf{B}$ and the $j$-th row of module $\mathbf{A}$, respectively. By limiting updates to submatrices $[\mathbf{b}_j]_{j \in \mathcal{I}_i}$ and $[\mathbf{a}_j^\top]_{j \in \mathcal{I}_i}$, 
FSLoRA reduces communication and computation. To further reduce communication cost, we can apply Top-$k$ compression to the uploading stage. For instance, instead of sending the full 
$\Delta [\mathbf{b}_j]_{j \in \mathcal{I}_i}$, each client transmits the compressed update 
$\operatorname{Top}_k(\Delta [\mathbf{b}_j]_{j \in \mathcal{I}_i})$.  
Sketching selects the update submatrix at the beginning of each round, 
while compression further reduces its transmission cost at the uploading stage. 
These two techniques operate at different stages and can jointly improve 
communication efficiency.

In our setup, the compression ratio is fixed at \(0.5\) for all methods, while the sketching ratio \(k_i/r\) varies over \(\{0.125, 0.25, 0.5, 1\}\). Notably, FSLoRA with sketching ratio \(k_i/r = 1\) corresponds to the vanilla Federated LoRA (i.e., without sketching).
Figure \ref{fig:integration_compression_sketching} plots testing accuracy versus communication overhead, where the x-axis represents the amount of data uploaded per client (in MB), assuming parameters are stored in float 32 precision. 
The results clearly show that integrating sketching with top-k compression further improves communication efficiency: methods with lower sketching ratios consistently achieve higher accuracy under the same communication budget, highlighting the potential of FSLoRA for scalable and communication-efficient collaborative LLM fine-tuning.

\section{LoRA structure over training}

To characterize the learned LoRA structure, we compare the stable-rank dynamics of the LoRA adapter under the proposed FSLoRA and representative baselines. Specifically, we analyze the stable rank of the server-side global LoRA update on Llama-3.2-3B under the same heavy-tail setting as in Table~\ref{tab:divice_het}. The stable rank is defined as
\begin{equation}
\mathrm{srank}(\Delta \mathbf{W}_t)
=
\frac{\|\Delta \mathbf{W}_t\|_F^2}{\|\Delta \mathbf{W}_t\|_2^2},
\qquad
\Delta \mathbf{W}_t = \mathbf{B}_t \mathbf{A}_t.
\end{equation}

For representative illustration, we report a middle layer (L15, \texttt{q\_proj} and \texttt{down\_proj}) with global rank $r=32$.

\begin{table}[t]
\centering
\caption{Stable-rank dynamics of the server-side global LoRA update at layer 15 of Llama-3.2-3B with global rank $r=32$.}
\label{tab:stable_rank_dynamics}
\resizebox{\linewidth}{!}{
\begin{tabular}{lcccccccccc}
\toprule
\multicolumn{11}{c}{\textbf{L15, \texttt{q\_proj}}} \\
\midrule
Rounds & 20 & 40 & 60 & 80 & 100 & 120 & 140 & 160 & 180 & 200 \\
\midrule
HeterLoRA & 1.158 & 1.115 & 1.097 & 1.098 & 1.102 & 1.105 & 1.107 & 1.112 & 1.121 & 1.124 \\
FlexLoRA  & 1.348 & 1.308 & 1.109 & 1.030 & 1.016 & 1.008 & 1.009 & 1.006 & 1.005 & 1.002 \\
\rowcolor{ours} FSLoRA    & 1.493 & 1.694 & 1.781 & 1.870 & 1.958 & 2.047 & 2.102 & 2.124 & 2.131 & 2.137 \\
\midrule
\multicolumn{11}{c}{\textbf{L15, \texttt{down\_proj}}} \\
\midrule
Rounds & 20 & 40 & 60 & 80 & 100 & 120 & 140 & 160 & 180 & 200 \\
\midrule
HeterLoRA & 1.402 & 1.327 & 1.423 & 1.532 & 1.665 & 1.790 & 1.911 & 2.009 & 2.105 & 2.183 \\
FlexLoRA  & 1.243 & 1.167 & 1.093 & 1.028 & 1.012 & 1.006 & 1.007 & 1.006 & 1.005 & 1.002 \\
\rowcolor{ours} FSLoRA    & 1.608 & 1.859 & 2.032 & 2.153 & 2.272 & 2.421 & 2.493 & 2.517 & 2.531 & 2.536 \\
\bottomrule
\end{tabular}
}
\end{table}

As we can see from Table~\ref{tab:stable_rank_dynamics}, the stable-rank trajectories differ markedly across methods. FlexLoRA quickly collapses toward rank $1$, while FSLoRA maintains a richer higher-rank adaptation structure throughout training.

\section{Implementation Details}
\label{appen:implementation_details}
\subsection{Details on Hyperparameters}
%\label{appen:hyperparameters}
Unless stated otherwise, the hyperparameters used in this work are as follows.
\begin{table}[H]
\caption{
The hyperparameters for RoBERTa \& GLUE and LLaMA-3.2-3B-Instruct \& commonsense reasoning benchmarks.
}
\label{tab:hyperparameters}

\centering
\resizebox{\linewidth}{!}{
\begin{tabular}{l|cc}
    \toprule
    Hyperparameter & RoBERTa \& GLUE & LLaMA-3.2-3B-Instruct \& commonsense reasoning \\
   \midrule
   Dirichlet parameter & 0.1 & 0.1 \\
   \midrule
    Batch size & 16 & 16 \\
    \midrule
   LoRA dropout rate & 0.1 & 0.1 \\ %\multicolumn{2}{c}{0.1}
    \midrule
    Learning rate, $\gamma$ & 5e-4 & 3e-4 \\
    \midrule
    Communication round, $T$ & 200 & 200 \\
    \midrule
    Local iteration number, $H$& 50 & 20 \\
    \midrule
    % Number of edge devices, $N$& 20 & 20 \\
    % \midrule
    Target module  & [``query'', ``value'', ``classification head''] & [``q\_proj'', ``k\_proj'', ``v\_proj'', ``up\_proj'', ``down\_proj''] \\
    \bottomrule
\end{tabular}
}
\end{table}

%\newpage

\subsection{Details on Datasets}
%\label{appen:benchmarks}

\subsubsection{GLUE Benchmark}
GLUE is a widely recognized benchmark designed to assess the natural language understanding capabilities of language models \citep{wang2018glue}.
\begin{itemize}[topsep=0pt,itemsep=4pt,parsep=0pt, partopsep=0pt, leftmargin=*]
    \item \textbf{CoLA} focuses on whether a given sentence is acceptable according to linguistic rules. It evaluates a model's ability to recognize well-formed sentences.
    \begin{itemize}
        \item[$\rhd$] Input: A single sentence.
        \item[\ding{73}] Output: A label indicating whether the sentence is acceptable or unacceptable.
    \end{itemize}
    \item \textbf{SST-2} is designed for sentiment classification on movie reviews or short texts. It tests whether a model can correctly identify positive or negative sentiment in a given sentence.
    \begin{itemize}
        \item[$\rhd$] Input: A single sentence. 
        \item[\ding{73}] Output: A label indicating positive or negative sentiment.
    \end{itemize}
    \item \textbf{MRPC} checks if two sentences are paraphrases of each other, i.e., if they mean the same thing.
    \begin{itemize}
        \item[$\rhd$] Input: Two sentences (`sentence1' and `sentence2'). 
        \item[\ding{73}] Output: A label indicating either equivalent or not equivalent. 
    \end{itemize}
    \item \textbf{QQP} tests a model's ability to determine if two questions ask the same thing.
    \begin{itemize}
        \item[$\rhd$] Input: Two questions. 
        \item[\ding{73}] Output: A label indicating duplicate or not duplicate.
    \end{itemize}
    \item \textbf{MNLI} tests whether a given hypothesis is entailed, contradicted, or neutral with respect to a premise. 
    \begin{itemize}
        \item[$\rhd$] Input: A premise (first sentence) and a hypothesis (second sentence). 
        \item[\ding{73}] Output: A label indicating entailment, contradiction, or neutral.
    \end{itemize}
    \item \textbf{QNLI} aims to determine if a context sentence correctly answers a given question. 
    \begin{itemize}
        \item[$\rhd$] Input: A question and a sentence.  
        \item[\ding{73}] Output: A label indicating the sentence answers the question or it does not.
    \end{itemize}
    \item \textbf{RTE} provides pairs of sentences to see if one implies the other. 
    \begin{itemize}
        \item[$\rhd$] Input: Two sentences (`sentence1' and `sentence2')
        \item[\ding{73}] Output: A label indicating whether the meaning of one sentence is entailed from the other one.
    \end{itemize}
\end{itemize}

%\newpage

\begin{table}[H]
\caption{The prompt template of the commonsense reasoning datasets \citep{hu2023llm}.}
\label{tab:input_template}

\begin{tabularx}{\textwidth}{lX}
\toprule
Dataset & Input Template \\
\midrule

ARC-c/e & 
\begin{minipage}{\linewidth}
\begin{tcolorbox}[colback=gray!10,colframe=black!75!black, left=1mm, right=1mm, top=1mm, bottom=1mm]
Please choose the correct answer to the question: [QUESTION] 

Answer1: [ANSWER\_1]

Answer2: [ANSWER\_2]

Answer3: [ANSWER\_3]

Answer4: [ANSWER\_4]

Answer format: answer1/answer2/answer3/answer4

the correct answer is [ANSWER]
\end{tcolorbox}
\end{minipage} \\
\midrule

BoolQ & 
\begin{minipage}{\linewidth}
\begin{tcolorbox}[colback=gray!10,colframe=black!75!black, left=1mm, right=1mm, top=1mm, bottom=1mm]
Please answer the following question with true or false, question: [QUESTION]

Answer format: true/false

the correct answer is [ANSWER]
\end{tcolorbox}
\end{minipage} \\
\midrule

HellaSwag & 
\begin{minipage}{\linewidth}
\begin{tcolorbox}[colback=gray!10,colframe=black!75!black, left=1mm, right=1mm, top=1mm, bottom=1mm]
Please choose the correct ending to complete the given sentence: [ACTIVITY\_LABEL]: [CONTEXT]

Ending1: [ENDING\_1] 

Ending2: [ENDING\_2] 

Ending3: [ENDING\_3] 

Ending4: [ENDING\_4] 

Answer format: ending1/ending2/ending3/ending4 

the correct answer is [ANSWER]
\end{tcolorbox}
\end{minipage} \\
\midrule

OBQA & 
\begin{minipage}{\linewidth}
\begin{tcolorbox}[colback=gray!10,colframe=black!75!black, left=1mm, right=1mm, top=1mm, bottom=1mm]
Please choose the correct answer to the question: [QUESTION] 

Answer1: [ANSWER\_1]

Answer2: [ANSWER\_2]

Answer3: [ANSWER\_3]

Answer4: [ANSWER\_4]

Answer format: answer1/answer2/answer3/answer4

the correct answer is [ANSWER]
\end{tcolorbox}
\end{minipage} \\
\midrule

PIQA & 
\begin{minipage}{\linewidth}
\begin{tcolorbox}[colback=gray!10,colframe=black!75!black, left=1mm, right=1mm, top=1mm, bottom=1mm]
Please choose the correct solution to the question: [QUESTION]

Solution1: [SOLUTION\_1]

Solution2: [SOLUTION\_2]

Answer format: solution1/solution2 

the correct answer is [ANSWER]
\end{tcolorbox}
\end{minipage} \\
\midrule

SIQA & 
\begin{minipage}{\linewidth}
\begin{tcolorbox}[colback=gray!10,colframe=black!75!black, left=1mm, right=1mm, top=1mm, bottom=1mm]
Please choose the correct answer to the question: [QUESTION]

Answer1: [ANSWER\_1]

Answer2: [ANSWER\_2]

Answer3: [ANSWER\_3]

Answer format: answer1/answer2/answer3 

the correct answer is [ANSWER]
\end{tcolorbox}
\end{minipage} \\
\midrule

WinoGrande & 
\begin{minipage}{\linewidth}
\begin{tcolorbox}[colback=gray!10,colframe=black!75!black, left=1mm, right=1mm, top=1mm, bottom=1mm]
Please choose the correct answer to fill in the blank to complete the given sentence: [SENTENCE]

Option1: [OPTION\_1]

Option2: [OPTION\_2]

the correct answer is [ANSWER]
\end{tcolorbox}
\end{minipage} \\
\bottomrule

\end{tabularx}
\end{table}

\subsubsection{Commonsense Reasoning Benchmark}
The training set of the commonsense reasoning benchmark is a mixture of multiple datasets including about 170K training samples from ARC-c/e \citep{clark2018think}, BoolQ \citep{clark2019boolq}, HellaSwag \citep{zellers2019hellaswag}, OBQA \citep{mihaylov2018can}, PIQA \citep{bisk2020piqa}, SIQA \citep{sap2019socialiqa}, and WinoGrande \citep{sakaguchi2021winogrande} datasets. 
\begin{itemize}[leftmargin=*]
    \item \textbf{ARC-c/e} contains the challenge and easy question set from the ARC dataset of genuine grade-school level, multiple-choice science questions.
    \item \textbf{BoolQ} is a question-answering dataset with yes/no questions derived from natural, real-world scenarios. 
    \item \textbf{HellaSwag} includes questions for commonsense natural language inference, where a context and multiple endings are given, requiring the most coherent ending to be selected.
    \item \textbf{OBQA} involves multi-step problem-solving that combines commonsense knowledge, reasoning, and comprehension of accompanying textual information.
    \item \textbf{PIQA} focuses on questions requiring physical commonsense to solve. Each question offers two answer choices. 
    \item \textbf{SIQA} targets reasoning about human actions and their social implication. 
    \item \textbf{WinoGrande} is designed as a binary-choice fill-in-the-blank task, this dataset evaluates the ability to resolve ambiguous sentences through commonsense reasoning.
\end{itemize}
The input template, i.e., prompt format for these datasets is detailed in Table \ref{tab:input_template}.

\section{Proof of the Theoretical Results}
\subsection{Preliminaries}
Before presenting the proof of the main results, we first introduce some preliminary facts that will be used later. Throughout this work, $\|\cdot\|$ denotes the Frobenius norm when applied to a matrix and the $\ell_2$ norm when applied to a vector.

\begin{lemma}\label{pre:lemma:variable:0}
    Suppose a sequence of independent random matrices $\left\{{\bf P}_i\right\}_{i=1}^N$ satisfy $\mathbb{E}[{\bf P}_i] = {\bf 0}, \forall i.$ Then, 
    \begin{equation}
        \mathbb{E}\left\|\frac{1}{N} \sum_{i=1}^N {\bf P}_i\right\|^2 = \frac{1}{N^2} \sum_{i=1}^N \mathbb{E} \left\|{\bf P}_i\right\|^2.
    \end{equation}
\end{lemma}

\begin{lemma} {\normalfont \citep[Lemma 2]{jianyu}}\label{pre:lemma:jianyu}
Suppose a sequence of random matrices $\left\{{\bf P}_i\right\}_{i=1}^N$ satisfy $\mathbb{E} \left[{\bf P}_i \mid {\bf P}_{i-1}, \!{\bf P}_{i-2}, \ldots, {\bf P}_1 \right]={\bf 0}, \forall i.$ Then,
\begin{equation}
\mathbb{E}\left[\left\|\sum_{i=1}^N {\bf P}_i\right\|^2\right]=\sum_{i=1}^N \mathbb{E}\left[\left\|{\bf P}_i\right\|^2\right].
\end{equation}
\end{lemma}

\begin{lemma}\label{lemma:tune_stepsize}
{\normalfont \citep[Lemma 17]{koloskova2020unified}}\label{converging_iterates} For any $a_0 \geq 0, b\geq 0, c\geq 0, d >0$, there exist a constant $\eta \leq \frac{1}{d}$ such that
\begin{align}
\frac{a_0}{T \eta } + b \eta + c \eta^2 \leq 2 \left( \frac{a_0 b}{T} \right)^{\frac{1}{2}} + 2 c^{\frac{1}{3}} \left(\frac{a_0}{T}\right)^{\frac{2}{3}} + \frac{d a_0}{T}.
\end{align}
\end{lemma}

\begin{lemma}[Random sketching bounds]\label{lamma:sketching_bound}
Let $\mathbf{S}$ be a random diagonal sketching matrix of the form
\begin{align}
\mathbf{S}
\;=\;
\frac{r}{k}\,\sum_{j \in \mathcal{I}} \mathbf{e}_j\,\mathbf{e}_j^{\top},
\end{align}
where \(\mathbf{e}_1, \dots, \mathbf{e}_r \in \mathbb{R}^r\) are standard unit basis vectors and $\mathcal{I}\subseteq \{1,\dots,r\}$ is chosen uniformly at random with $|\mathcal{I}|=k$. Then for any matrix $\mathbf{X}$, we have
\begin{equation}\label{diag:worstcase_bound}
\|\mathbf{X}\,\mathbf{S}\|^{2}
~\le~
\frac{r^{2}}{k^{2}}
\,\|\mathbf{X}\|^{2},
\end{equation}
and in expectation we have
\begin{equation}\label{diag:expected_bound}
\mathbb{E}_{\mathbf{S}}\!\Bigl[\|\mathbf{X}\,\mathbf{S}\|^{2}\Bigr]
~\le~
\frac{r}{k}
\,\|\mathbf{X}\|^{2}.
\end{equation}
\end{lemma}

\begin{proof}
Since $\mathbf{S}$ is diagonal with exactly $k$ diagonal entries equal to $\tfrac{r}{k}$ and the rest zero, its largest eigenvalue is $\tfrac{r}{k}$.  Squaring gives
\begin{align}
\mathbf{S}\,\mathbf{S}^{\!\top}
~=~
\mathbf{S}^{2}
~\preceq~
\frac{r^{2}}{k^{2}}
\,\mathbf{I},
\end{align}
Equivalently,
\begin{align}
\mathbf{x}^{\top}
\bigl(\mathbf{S}\,\mathbf{S}^{\!\top}\bigr)
\mathbf{x}
~\le~
\frac{r^{2}}{k^{2}}
\,\|\mathbf{x}\|^{2},
 \forall \mathbf{x}.
\end{align}
Setting $\mathbf{x} = \mathbf{x}_j$ to be the $j$-th column of $\mathbf{X}$ and summing over $j$ implies
\begin{align}
\|\mathbf{X}\,\mathbf{S}\|^{2}
~
=\;
\sum_{j}
\,\|\mathbf{S}^{\!\top}\,\mathbf{x}_j\|^{2}
~
=\;
\sum_{j}
\,\mathbf{x}_j^{\top}
\,(\mathbf{S}\,\mathbf{S}^{\!\top})
\,\mathbf{x}_j
~\le~
\frac{r^{2}}{k^{2}}
\,\sum_{j}
\,\|\mathbf{x}_j\|^{2}
~
=\;
\frac{r^{2}}{k^{2}}
\,\|\mathbf{X}\|^{2},
\end{align}
which proves~\eqref{diag:worstcase_bound}.

For the expected bound~\eqref{diag:expected_bound}, note that each diagonal index $j\in\{1,\dots,r\}$ is included in $\mathcal{I}$ with probability $\tfrac{k}{r}$.  Hence the expectation of $\mathbf{S}^{2}$ satisfies
\begin{align}
\mathbb{E}_{\mathbf{S}}\bigl[\mathbf{S}^{2}\bigr]
~
=\;
\frac{r^{2}}{k^{2}}
\,\mathbb{E}\Bigl[\sum_{j\in \mathcal{I}} \mathbf{e}_j\,\mathbf{e}_j^{\top}\Bigr]
~
=\;
\frac{r^{2}}{k^{2}}
\,
\frac{k}{r}\,\mathbf{I}
~
=\;
\frac{r}{k}\,\mathbf{I}.
\end{align}
Thus for any vector $\mathbf{x}$,
\begin{align}
\mathbb{E}_{\mathbf{S}}
\!\Bigl[\|\mathbf{S}^{\!\top}\mathbf{x}\|^{2}\Bigr]
~
=\;
\mathbb{E}_{\mathbf{S}}
\!\Bigl[\mathbf{x}^{\top}\,\mathbf{S}\,\mathbf{S}^{\!\top}\,\mathbf{x}\Bigr]
~
=\;
\mathbf{x}^{\top}
\Bigl(\mathbb{E}[\mathbf{S}^{2}]\Bigr)
\mathbf{x}
~
=\;
\frac{r}{k}\,\|\mathbf{x}\|^{2}.
\end{align}
Summing over columns of $\mathbf{X}$ again establishes
\begin{align}
\mathbb{E}_{\mathbf{S}}
\!\bigl[\|\mathbf{X}\,\mathbf{S}\|^{2}\bigr]
~
=\;
\sum_{j}
\,\mathbb{E}_{\mathbf{S}}
\!\bigl[\|\mathbf{S}^{\!\top}\mathbf{x}_j\|^{2}\bigr]
~
=\;
\sum_{j}
\,\mathbf{x}_j^{\top}
\Bigl(\mathbb{E}[\mathbf{S}^{2}]\Bigr)
\mathbf{x}_j
~
=\;
\frac{r}{k}
\,\|\mathbf{X}\|^{2}.
\end{align}
This completes the proof of Lemma \ref{lamma:sketching_bound}.
\end{proof}

\subsection{Proof of Lemma \ref{compute_grad}}\label{proof_compute_grad}
From the chain rule for matrix calculus, we know that: 
    \begin{align}
    \nabla_{\bf Y} g({\bf X}{\bf Y}) ={\bf X}^{\top} \nabla g({\bf X}{\bf Y}), ~\nabla_{\bf X} g({\bf X}{\bf Y}) =  \nabla g({\bf X} {\bf Y}){\bf Y}^{\top}, 
    \end{align}
    %$\nabla_{\bf Y} g({\bf X}{\bf Y}) ={\bf X}^{\top} \nabla g(\bf X{\bf Y}) $ and $\nabla_{\bf X} g({\bf X}{\bf Y}) =  \nabla g({\bf X} {\bf Y}){\bf Y}^{\top}$, 
    where \(\nabla g({\bf X} {\bf Y})\) denotes the gradient of \(g\) to \({\bf X} {\bf Y}\). Applying this to \(\ell({\bf W}_0 + {\bf B} {\bf S} {\bf A}, \xi)\), %we can derive its gradient with respect to \(\bf B\) by setting \({\bf Y} = \bf S{\bf A}\) and \(\bf X = \bf B\). Similarly, we can obtain its gradient to \(\bf A\) by setting \(\bf X = {\bf B}\bf S\) and \({\bf Y} ={\bf A}\). 
    we proceed as follows: \\
    To compute the gradient with respect to \(\mathbf{B}\), set \({\bf X} = \mathbf{B}\) and \({\bf Y} = \mathbf{S}\mathbf{A}\):
   \begin{align}
   \nabla_{\mathbf{B}} \ell({\bf W}_0 + {\bf BSA}, \xi) = \nabla \ell({\bf W}_0 + {\bf BSA}, \xi) (\mathbf{S}\mathbf{A})^\top.
   \end{align}
   Similarly, to compute the gradient with respect to \(\mathbf{A}\), set \({\bf X} = \mathbf{B}\mathbf{S}\) and \({\bf Y} = \mathbf{A}\):
   \begin{align}
   \nabla_{\mathbf{A}} \ell({\bf W}_0 + {\bf BSA}, \xi) = \mathbf{S}^\top  \mathbf{B}^\top  \nabla \ell({\bf W}_0 + {\bf BSA}, \xi).
   \end{align}

\subsection{Proof of Theorem \ref{convergence_fed_slora_condense}}\label{proof_theorem}
The proof of Theorem \ref{convergence_fed_slora_condense} relies on the following proposition. 
\begin{proposition}\label{smooth_derivation}
    Under Assumption \ref{smoothness_assumption}, \(\tilde{f}_{i} ({\bf X}; {\bf S}) = f_i({\bf B} {\bf S}, {\bf A}),~{\bf S} \in \mathcal{S}_i \), \(f_i^{\mathcal{S}}({\bf X}) = \mathbb{E}_{{\bf S} \sim \mathcal{S}_i} [\tilde{f}_{i} ({\bf X}; {\bf S}) ]\), and \(f^{\mathcal{S}} ({\bf X}) = \frac{1}{N} \sum_{i=1}^N f_i^{\mathcal{S}}(\mathbf{X})\) are smooth with parameters \(L \frac{r^2}{k_i^2}\), \(L \frac{r}{k_i}\), and \(\left( \frac{1}{N} \sum_{i=1}^N \frac{r}{k_i} \right) L \), respectively. 
\end{proposition}
The proof of Proposition \ref{smooth_derivation} is deferred to Appendix \ref{proof_smooth}. With this proposition, we are ready to prove Theorem \ref{convergence_fed_slora_condense}. 

In FSLoRA, the update direction in \eqref{gradient_client_i} corresponds to the negative stochastic gradient of \(\ell({\bf W}_0 + {\bf B} {\bf S} {\bf A}, \xi)\) with respect to \([{\bf B}; {\bf A}]\) for a given sketch \( {\bf S}_i^t\).
We have defined \(\tilde{\ell} ({\bf X}, \xi; {\bf S}) = \ell({\bf W}_0 + {\bf B}{\bf S}{\bf A}, \xi)\). The iterative equation for the proposed FSLoRA algorithm thus can be written as
\begin{equation}
   {\bf X}^{t+1} ={\bf X}^{t} - \gamma \frac{1}{N} \sum_{i=1}^{N} \sum_{h=0}^{H-1} \nabla_{\bf X} \tilde{\ell} ({\bf X}^{t,h}_i, \xi_i^{t,h}; {\bf S}_i^t),
\end{equation}
where \({\bf g}^{t,h}_i\) denotes the stochastic gradient \(\nabla_{\bf X} \tilde{\ell} ({\bf X}^{t,h}_i, \xi_i^{t,h}; {\bf S}_i^t)\).
Based on the smoothness of $f^{\mathcal{S}}(\bf X)$, i.e., Proposition \ref{smooth_derivation}, we have
\begin{align}\label{expansion_smooth}
   \mathbb{E} [f^{\mathcal{S}}({\bf X}^{t+1})] \leq \mathbb{E}[f^{\mathcal{S}}({\bf X}^t)] \underbrace{-\mathbb{E}\left \langle
    \nabla_{\bf X} f^{\mathcal{S}} ({\bf X}^t) , \gamma \frac{1}{N} \sum_{i=1}^{N} \sum_{h=0}^{H-1} {\bf g}^{t,h}_i
    \right \rangle}_{T_1} +
    \frac{\gamma^2 \bar{L}}{2} \underbrace{ \mathbb{E}\left \| \frac{1}{N} \sum_{i=1}^{N} \sum_{h=0}^{H-1}{\bf g}^{t,h}_i \right \|^2}_{T_2},
\end{align}
%where \(\nabla_{\bf X} f^{\mathcal{S}}({\bf X}^t)\) denote the gradient of global empirical loss at \({\bf X}^t\).
where \(\bar{L} = \left( \frac{1}{N} \sum_{i=1}^N \frac{r}{k_i} \right) L \).

For \(T_1\), we have 
\begin{align}
    T_1 =& -H\mathbb{E}\left \langle
    \nabla_{\bf X} f^{\mathcal{S}} ({\bf X}^t), \gamma \frac{1}{NH} \sum_{i=1}^{N} \sum_{h=0}^{H-1} {\bf g}^{t,h}_i
    \right \rangle \nonumber \\
    %=& -H\mathbb{E} \left \langle \nabla_{\bf X} f^{\mathcal{S}} ({\bf X}^t), \gamma \frac{1}{N} \sum_{i=1}^{N} \sum_{h=0}^{H-1} \nabla_{\bf X} \tilde{f}_i ({\bf X}_i^{t,h}; {\bf S}_i^t) \right \rangle \nonumber\\
    =& -H\mathbb{E} \left \langle
    \nabla_{\bf X} f^{\mathcal{S}} ({\bf X}^t), \gamma \frac{1}{NH} \sum_{i=1}^{N} \sum_{h=0}^{H-1} \nabla_{\bf X} f_i^{\mathcal{S}} ({\bf X}_i^{t,h})
    \right \rangle \nonumber\\
    =&  -\frac{\gamma H}{2} \mathbb{E} \left\| \nabla_{\bf X} f^{\mathcal{S}} ({\bf X}^t) \right\|^2 - \frac{\gamma H}{2} \mathbb{E} \left\| \frac{1}{NH} \sum_{i=1}^{N} \sum_{h=0}^{H-1} \nabla_{\bf X} f_i^{\mathcal{S}} ({\bf X}_i^{t,h}) \right\|^2
    \nonumber\\
    & + \frac{\gamma H}{2} \mathbb{E} \left \|
    \nabla_{\bf X} f^{\mathcal{S}} ({\bf X}^t) - \frac{1}{NH} \sum_{i=1}^{N} \sum_{h=0}^{H-1} \nabla_{\bf X} f_i^{\mathcal{S}} ({\bf X}_i^{t,h}) 
    \right \|^2 \nonumber\\
    \leq&  - \frac{\gamma H}{2} \mathbb{E} \left\| \nabla_{\bf X} f^{\mathcal{S}} ({\bf X}^t) \right\|^2 - \frac{\gamma H}{2} \mathbb{E} \left\| \frac{1}{N} \sum_{i=1}^{N} \sum_{h=0}^{H-1} \nabla_{\bf X} f_i^{\mathcal{S}} ({\bf X}_i^{t,h}) \right\|^2
    \nonumber\\
    & + \frac{\gamma}{2} \sum_{h=0}^{H-1} \mathbb{E} \left \|\frac{1}{N} \sum_{i=1}^{N} \nabla_{\bf X} f_i^{\mathcal{S}} ({\bf X}^t) - \frac{1}{N} \sum_{i=1}^{N} \nabla_{\bf X} f_i^{\mathcal{S}} ({\bf X}_i^{t,h}) 
    \right \|^2 \nonumber\\
    \leq&  - \frac{\gamma H}{2} \mathbb{E} \left\| \nabla_{\bf X} f^{\mathcal{S}} ({\bf X}^t) \right\|^2 - \frac{\gamma }{2H} \mathbb{E} \left\| \frac{1}{N} \sum_{i=1}^{N} \sum_{h=0}^{H-1} \nabla_{\bf X} f_i^{\mathcal{S}} ({\bf X}_i^{t,h}) \right\|^2 \nonumber \\
    & + \frac{\gamma H L^2}{2}  \frac{1}{NH} \sum_{i=1}^{N} \frac{r^2}{k_i^2} \sum_{h=0}^{H-1} \mathbb{E} \left \|{\bf X}_i^{t,h} -{\bf X}^t
    \right \|^2, \label{T_1}
\end{align}
where the last inequalities follow Jensen's inequality and Proposition \ref{smooth_derivation}. 

For $T_2$, we have
\begin{align}
    T_2 = & \mathbb{E}\left \| \frac{1}{N} \sum_{i=1}^{N} \sum_{h=0}^{H-1} ({\bf g}^{t,h}_i \mp \nabla_{\bf X} f_i^{\mathcal{S}} ({\bf X}_i^{t,h}))\right \|^2 \nonumber\\
    \leq & \frac{2}{N^2} \sum_{i=1}^{N} \mathbb{E}\left \|  \sum_{h=0}^{H-1} ({\bf g}^{t,h}_i -  \nabla_{\bf X} f_i^{\mathcal{S}} ({\bf X}_i^{t,h})) \right \|^2 + 2\mathbb{E}\left \| \frac{1}{N} \sum_{i=1}^{N} \sum_{h=0}^{H-1} \nabla_{\bf X} f_i^{\mathcal{S}} ({\bf X}_i^{t,h}) \right \|^2, 
\end{align}
where the inequality follows the fact that $\mathbb{E}[\sum_{h=0}^{H-1} ({\bf g}^{t,h}_i -  \nabla_{\bf X} f_i^{\mathcal{S}} ({\bf X}_i^{t,h}))] = 0$ and the independence between clients.

Furthermore, we bound the first term on the right-hand side of the above inequality as
\begin{align}
    \mathbb{E}\left \|  \sum_{h=0}^{H-1} ({\bf g}^{t,h}_i -  \nabla_{\bf X} f_i^{\mathcal{S}} ({\bf X}_i^{t,h})) \right \|^2 = \sum_{h=0}^{H-1} \mathbb{E}\left \|{\bf g}^{t,h}_i -  \nabla_{\bf X} f_i^{\mathcal{S}} ({\bf X}_i^{t,h}) \right \|^2 
    \leq H \sigma^2 + \rho \sum_{h=0}^{H-1} \mathbb{E} \left \|\nabla_{\bf X} f_i^{\mathcal{S}} ({\bf X}_i^{t,h}) \right \|^2, \nonumber
\end{align}
where the equality follows Lemma \ref{pre:lemma:jianyu} and the inequality follows Assumption \ref{varaince_assumption}. 
% We thus obtain
% \begin{align}
%     T_2 \leq  2H\frac{\sigma^2}{N} + \frac{2 \rho}{N^2} \sum_{i=1}^{N} \sum_{h=0}^{H-1} \mathbb{E} \left \|\nabla_{\bf X} f_i^{\mathcal{S}} ({\bf X}_i^{t,h}) \right \|^2 + 2\mathbb{E}\left \| \frac{1}{N} \sum_{i=1}^{N} \sum_{h=0}^{H-1} \nabla_{\bf X} f_i^{\mathcal{S}} ({\bf X}_i^{t,h}) \right \|^2.
% \end{align}
For $\left \|\nabla_{\bf X} f_i^{\mathcal{S}} ({\bf X}_i^{t,h}) \right \|^2$, utilizing Assumption \ref{assumption:gradient_dissimilarity} and Proposition \ref{smooth_derivation}, we have
\begin{align}\label{bound_gradient_i}
& \left \|\nabla_{\bf X} f_i^{\mathcal{S}} ({\bf X}_i^{t,h}) \right \|^2 = \left \|\nabla_{\bf X} f_i^{\mathcal{S}} ({\bf X}_i^{t,h}) \mp \nabla_{\bf X} f_i^{\mathcal{S}} ({\bf X}^{t})  \mp \nabla_{\bf X} f^{\mathcal{S}} ({\bf X}^t)   \right \|^2 \nonumber \\
\leq & 3 \left \|\nabla_{\bf X} f_i^{\mathcal{S}} ({\bf X}_i^{t,h}) - \nabla_{\bf X} f_i^{\mathcal{S}} ({\bf X}^{t}) \right \|^2 + 3 \left \|\nabla_{\bf X} f_i^{\mathcal{S}} ({\bf X}^{t})  - \nabla_{\bf X} f^{\mathcal{S}} ({\bf X}^t)   \right \|^2 + 3 \left \|\nabla_{\bf X} f^{\mathcal{S}} ({\bf X}^t)   \right \|^2 \nonumber \\
\leq & 3 \frac{r^2}{k_i^2} L^2 \left \|{\bf X}_i^{t,h} - {\bf X}^{t} \right \|^2 + 3 (c_h + 1) \| \nabla_{\bf X} f^{\mathcal{S}}({\bf X}^t) \|^2  \!+\! 3 \rho \delta_h^2.
\end{align}
Combining the above three inequalities gives rise to
\begin{align}\label{T_2}
    T_2 
    \leq & \frac{2H}{N} (\sigma^2 + 3 \rho \delta_h^2) + 2\mathbb{E}\left \| \frac{1}{N} \!\sum_{i=1}^{N} \sum_{h=0}^{H-1} \nabla_{\bf X} f_i^{\mathcal{S}} ({\bf X}_i^{t,h}) \right \|^2 + \frac{6\rho (c_h + 1) H}{N} \mathbb{E} \| \nabla_{\bf X} f^{\mathcal{S}}({\bf X}^t) \|^2 \nonumber \\
    &+ \frac{6 \rho H L^2}{N} T_3,
\end{align}
where $T_3 = \frac{1}{NH} \sum_{i=1}^{N} \frac{r^2}{k_i^2} \sum_{h=0}^{H-1} \mathbb{E} \left \|{\bf X}_i^{t,h} -{\bf X}^t
    \right \|^2$.
Combining \eqref{expansion_smooth}, \eqref{T_1}, and \eqref{T_2} yields
\begin{align}
   \mathbb{E} [f^{\mathcal{S}}({\bf X}^{t+1})] \leq & \mathbb{E}[f^{\mathcal{S}}({\bf X}^t)]   - (\frac{\gamma H}{2} - 3\gamma^2 \rho (c_h+1) \frac{H}{N}\bar{L} ) \mathbb{E} \left\| \nabla_{\bf X} f^{\mathcal{S}} ({\bf X}^t) \right\|^2 + \gamma^2 \bar{L} \frac{H}{N} (\sigma^2 + 3\rho \sigma_h^2)
    \nonumber\\
    & - (\frac{\gamma}{2H} -  \gamma^2 \bar{L} ) \mathbb{E} \left\| \frac{1}{N} \sum_{i=1}^{N} \sum_{h=0}^{H-1} \nabla_{\bf X} f_i^{\mathcal{S}} ({\bf X}_i^{t,h}) \right\|^2 
   + (\frac{\gamma H L^2}{2} +  3 \gamma^2 \rho \bar{L} L^2 \frac{H}{N} ) T_3, 
\end{align}
where \(\bar{L} = \left( \frac{1}{N} \sum_{i=1}^N \frac{r}{k_i} \right) L \).
Let $\gamma \leq \min \{\frac{N}{24 \rho (c_h + 1) H\bar{L}}, \frac{1}{2H\bar{L}}, \frac{N}{6\rho \bar{L}} \}$, we have
\begin{align}\label{T_1_T_2}
   \mathbb{E} [f^{\mathcal{S}}({\bf X}^{t+1})] \leq & \mathbb{E}[f^{\mathcal{S}}({\bf X}^t)]   - \frac{3\gamma H}{8} \mathbb{E} \left\| \nabla_{\bf X} f^{\mathcal{S}} ({\bf X}^t) \right\|^2 + \gamma^2 \bar{L} \frac{H}{N} (\sigma^2 + 3\rho \sigma_h^2)
   + \frac{5\gamma}{8} H L^2 T_3.
\end{align}

For $T_3$, we have
\begin{align}
    T_3 =& \frac{1}{NH} \sum_{i=1}^{N} \frac{r^2}{k_i^2} \sum_{h=0}^{H-1} \mathbb{E} \left \| \gamma \sum_{\tau=0}^{h-1} {\bf g}^{t,\tau}_i
    \right \|^2 \nonumber\\
    =& \gamma^2 \frac{1}{NH} \sum_{i=1}^{N} \frac{r^2}{k_i^2} \sum_{h=0}^{H-1} \mathbb{E} \left \| \sum_{\tau=0}^{h-1} ( {\bf g}^{t,\tau}_i \mp  \nabla_{\bf X} f_i^{\mathcal{S}} ({\bf X}_i^{t,\tau}) )
    \right \|^2 \nonumber\\
    \leq & 2 \gamma^2 \frac{1}{NH} \sum_{i=1}^{N} \frac{r^2}{k_i^2} \sum_{h=0}^{H-1}  \sum_{\tau=0}^{h-1} \mathbb{E} \left \| {\bf g}^{t,\tau}_i - \nabla_{\bf X} f_i^{\mathcal{S}} ({\bf X}_i^{t,\tau})
    \right \|^2 +  2\gamma^2 \frac{1}{NH} \sum_{i=1}^{N} \frac{r^2}{k_i^2} \sum_{h=0}^{H-1} h \sum_{\tau=0}^{h-1} \mathbb{E} \left \|\nabla_{\bf X} f_i^{\mathcal{S}} ({\bf X}_i^{t,\tau})
    \right \|^2 \nonumber\\
    \leq & 2\gamma^2 H \sigma^2 \left( \frac{1}{N} \sum_{i=1}^{N} \frac{r^2}{k_i^2} \right) 
    +  \frac{2 \rho \gamma^2}{NH} \sum_{i=1}^{N} \frac{r^2}{k_i^2} \sum_{h=0}^{H-1}  \sum_{\tau=0}^{h-1} \mathbb{E} \left \|\nabla_{\bf X} f_i^{\mathcal{S}} ({\bf X}_i^{t,\tau})
    \right \|^2 \nonumber \\
    &+ \frac{2\gamma^2}{NH} \sum_{i=1}^{N} \frac{r^2}{k_i^2} \sum_{h=0}^{H-1} h \sum_{\tau=0}^{h-1} \mathbb{E} \left \| \nabla_{\bf X} f_i^{\mathcal{S}} ({\bf X}_i^{t,\tau})
    \right \|^2 \nonumber\\
    \leq & 2\gamma^2 H \sigma^2 \left( \frac{1}{N} \sum_{i=1}^{N} \frac{r^2}{k_i^2} \right) 
    + \frac{2 (\rho+1) \gamma^2 H}{N} \sum_{i=1}^{N} \frac{r^2}{k_i^2} \sum_{h=0}^{H-1} \mathbb{E} \left \| \nabla_{\bf X} f_i^{\mathcal{S}} ({\bf X}_i^{t,\tau})
    \right \|^2. \label{T_3_intermiate_1}
\end{align}

Plugging inequality \eqref{bound_gradient_i} into inequality \eqref{T_3_intermiate_1} yeilds
    \begin{align}
    T_3 \leq & 2 \gamma^2 H \left( \frac{1}{N} \sum_{i=1}^{N} \frac{r^2}{k_i^2} \right) \sigma^2 + 6 (\rho+1) \gamma^2 H^2 \left( \frac{1}{N} \sum_{i=1}^{N} \frac{r^2}{k_i^2} \right) \sigma^2_h  \nonumber \\
    & + 6 (\rho+1) \gamma^2 L^2 H^2 T_3 + 6(\rho+1) \left( \frac{1}{N} \sum_{i=1}^{N} \frac{r^2}{k_i^2} \right) (c_h + 1)\gamma^2 H^2\mathbb{E} \left \| \nabla_{\bf X} f^{\mathcal{S}} ({\bf X}^t)
    \right \|^2.
\end{align}
Denote $\kappa = \frac{1}{N} \sum_{i=1}^{N} \frac{r^2}{k_i^2}$, we simplify the above inequality as 
\begin{align}
    (1- 6(\rho +1) \gamma^2 L^2 H^2) T_3 \leq 2 \kappa \gamma^2 H^2 (\sigma_g^2 + 3(\rho+1) \sigma_h^2) + 6\kappa (\rho+1)(c_h +1) \gamma^2 H^2 \mathbb{E} \left \| \nabla_{\bf X} f^{\mathcal{S}} ({\bf X}^t)  
    \right \|^2. \nonumber
\end{align}
Let $\gamma \leq \frac{1}{\sqrt{12(\rho+1)}HL}$, we get the bound for $T_3$
\begin{align} \label{T_3}
    T_3 \leq 4 \kappa \gamma^2 H^2 (\sigma^2 + 3(\rho+1) \sigma_h^2) + 12 \kappa (\rho+1) (c_h +1) \gamma^2 H^2 \mathbb{E} \left \| \nabla_{\bf X} f^{\mathcal{S}} ({\bf X}^t)
    \right \|^2.
\end{align}
Plugging the bound for $T_3$ into inequality \eqref{T_1_T_2} gives rise to
\begin{align}
   \mathbb{E} [f^{\mathcal{S}}({\bf X}^{t+1})] \leq & \mathbb{E}[f^{\mathcal{S}}({\bf X}^t)]   - ( \frac{3\gamma H}{8} - \frac{5\gamma H}{8} L^2 \left(12 \kappa (\rho+1) (c_h +1) \gamma^2 H^2) \right) \mathbb{E} \left\| \nabla_{\bf X} f^{\mathcal{S}} ({\bf X}^t) \right\|^2 \nonumber \\
   &+ \gamma^2 \bar{L} \frac{H}{N} (\sigma^2 + 3\rho \sigma_h^2) + \frac{5\gamma}{8} H L^2 \cdot 4 \kappa \gamma^2 H^2 (\sigma^2 + 3(\rho+1) \sigma_h^2).
\end{align}
Let $\gamma \leq \frac{1}{8\sqrt{\kappa (\rho+1) (c_h +1)} HL}$, we obtain
\begin{align}
   \mathbb{E} [f^{\mathcal{S}}({\bf X}^{t+1})] \leq & \mathbb{E}[f^{\mathcal{S}}({\bf X}^t)]   - \frac{\gamma H}{4} \mathbb{E} \left\| \nabla_{\bf X} f^{\mathcal{S}} ({\bf X}^t) \right\|^2 + \gamma^2 \bar{L} \frac{H}{N} \sigma_{\rho}^2 + \frac{5}{2} \kappa \gamma^3 H^3 L^2 \sigma_{\rho}^2,
\end{align}
where $ \sigma_{\rho}^2 = \sigma^2 + 3(\rho+1) \sigma_h^2$.

Telescoping the above inequality from \(t=0\) to \(T-1\), we have
\begin{align}
 \frac{1}{T}\sum_{t=0}^{T-1} \!\mathbb{E} \left\| \nabla_{\bf X} f^{\mathcal{S}} ({\bf X}^t) \right\|^2  \leq 4 \frac{ f^{\mathcal{S}}({\bf X}^0)  -  f^*}{\gamma TH}  + \gamma \frac{4 \bar{L}}{N} \sigma_{\rho}^2 + 10 \gamma^2 H^2 \widetilde{L} L \sigma_{\rho}^2,
\end{align}
where $f^*$ denotes the lower bound of $f^{\mathcal{S}}({\bf X})$ and $\widetilde{L} = \kappa L$. 
% This completes the proof of Theorem \ref{convergence_fed_slora}.

% \subsection{Proof of Corollary \ref{coro_convergence_fed_slora}}\label{proof_coro}

% Applying Lemma \ref{converging_iterates} to the bound derived in Theorem \ref{convergence_fed_slora} and letting $d=H$, 
Applying Lemma \ref{converging_iterates} to the above inequality and letting $d=H$, 
it follows that there exists a learning rate 
$$\gamma \leq \min \{\frac{N}{24 \rho (c_h + 1) H\bar{L}}, \frac{1}{8\sqrt{\widetilde{L} L (\rho+1) (c_h +1)} H}, \frac{1}{H}\}$$ 
such that
\begin{align}
\frac{1}{T}\sum_{t=0}^{T-1} \!\mathbb{E} \left\| \nabla_{\bf X} f^{\mathcal{S}} ({\bf X}^t) \right\|^2 \! \leq \! 8 \frac{\sqrt{\bar{L} \mathcal{F}_0 \sigma_{\rho}^2} }{\sqrt{N TH}} +
    10 (\widetilde{L} L)^{\frac{1}{3}} \left(\frac{\mathcal{F}_0 \sigma_{\rho}}{ T}\right)^{\frac{2}{3}} + \frac{4 \mathcal{F}_0 }{T}.
\end{align}
%This completes the proof of Corollary \ref{coro_convergence_fed_slora}.
This completes the proof of Theorem \ref{convergence_fed_slora_condense}.

\subsection{Proof of Proposition  \ref{smooth_derivation}}\label{proof_smooth}
%The proof of this lemma was inspired by \citep{demidovich2023mast}, where the author was the first to thoroughly study this type of sketching formulation. 

%The proof of Lemma \ref{smooth_derivation} relies on the following lemma. 

%Now, we are ready to prove Lemma \ref{smooth_derivation}.

i) For illustration, we need to recover \(\bf X\) to \([\bf B; {\bf A}]\) in this proof. According to the definition of \(\tilde{f}_{i} ({\bf X}; {\bf S})\) and \(f_{i} ({\bf B},{\bf A})\), we have
\begin{align}
    \tilde{f}_{i} ({\bf X}; {\bf S}) =& \tilde{f}_{i} ({\bf B}, {\bf A}; {\bf S}) \label{BS_B_A_S_added_1} \\ 
    =&  \mathbb{E}_{\xi \sim \mathcal{D}_i} \left[ \ell ({\bf W}_{0} + {\bf B} {\bf S}{\bf A}, \xi)\right]  \nonumber \\
    = & f_{i} ({\bf B} {\bf S}, {\bf A}). \label{BS_B_A_S}
\end{align}
As \(f_{i} ({\bf B},{\bf A}) \) is \(L\)-smooth, we have
\begin{align}\label{deriving_smoothness_S}
    f_{i} ({\bf B} {\bf S} + \Delta {\bf B} {\bf S}, {\bf A}+ \Delta{\bf A}) \leq f_{i} ({\bf B} {\bf S},{\bf A}) + 
    \left \langle \begin{bmatrix}
       \nabla_{{\bf B}\bf S} f_{i} ({\bf B} {\bf S},{\bf A}) \\
       \nabla_{\bf A} f_{i} ({\bf B} {\bf S},{\bf A}) 
    \end{bmatrix}, 
    \begin{bmatrix}
        \Delta {\bf B} {\bf S}\\
         \Delta {\bf A}  
    \end{bmatrix}
    \right \rangle + \frac{L}{2} \left \| \begin{bmatrix}
       \Delta {\bf B} {\bf S}\\
         \Delta {\bf A}
    \end{bmatrix} \right \|^2. 
\end{align}
According to \eqref{BS_B_A_S_added_1} and \eqref{BS_B_A_S}, we have \(\tilde{f}_{i} ({\bf B} + \Delta {\bf B} , {\bf A}+ \Delta{\bf A}; {\bf S}) =  f_{i} ({\bf B} {\bf S} + \Delta {\bf B} {\bf S}, {\bf A}+ \Delta{\bf A})\) and \( \tilde{f}_{i} ({\bf B},{\bf A}; {\bf S}) = f_{i} ({\bf B} {\bf S},{\bf A})\). Combining these with \eqref{deriving_smoothness_S} gives rise to
\begin{align}\label{B_A_S_new}
    \tilde{f}_{i} ({\bf B} + \Delta {\bf B} , {\bf A}+ \Delta{\bf A}; {\bf S}) \leq \tilde{f}_{i} ({\bf B},{\bf A}; {\bf S}) + 
    \left \langle \begin{bmatrix}
       \nabla_{{\bf B}\bf S} f_{i} ({\bf B} {\bf S},{\bf A}) \\
       \nabla_{\bf A} f_{i} ({\bf B} {\bf S},{\bf A}) 
    \end{bmatrix}, 
    \begin{bmatrix}
        \Delta {\bf B} {\bf S}\\
         \Delta {\bf A}  
    \end{bmatrix}
    \right \rangle + \frac{L}{2} \left \| \begin{bmatrix}
       \Delta {\bf B} {\bf S}\\
         \Delta {\bf A} 
    \end{bmatrix} \right \|^2. 
\end{align}
We denote 
\begin{align}\label{notation_function_L}
    L({\bf W}_{0} + {\bf B}{\bf S}{\bf A}) = \tilde{f}_{i} ({\bf B}, {\bf A}; {\bf S}) = \mathbb{E}_{\xi \sim \mathcal{D}_i} \left[ \ell ({\bf W}_{0} + {\bf B} {\bf S}{\bf A}, \xi)\right].
\end{align}
Note that \(\nabla_{{\bf B}{\bf S}} f_{i} ({\bf B} {\bf S},{\bf A}) = \nabla L({\bf W}_{0} + {\bf B}\bf S{\bf A}){\bf A}^{\top}\) and \(\nabla_{\bf A} f_{i} ({\bf B} {\bf S},{\bf A}) = {\bf S}^{\top} {\bf B}^{\top}  \nabla L({\bf W}_{0} + {\bf B} {\bf S} {\bf A})\). We thus have
\begin{align}
 \left \langle \begin{bmatrix}
       \nabla_{{\bf B}{\bf S}} f_{i} ({\bf B} {\bf S}; {\bf A}) \\
       \nabla_{\bf A} f_{i} ({\bf B} {\bf S}; {\bf A}) 
    \end{bmatrix}, 
    \begin{bmatrix}
        \Delta {\bf B} {\bf S}\\
        \Delta {\bf A}  
    \end{bmatrix}
    \right \rangle  
    =& \left \langle \begin{bmatrix}
      \nabla L({\bf W}_{0} + {\bf B}{\bf S}{\bf A}){\bf A}^{\top} \\
       {\bf S}^{\top} \bf B^{\top} \nabla L({\bf W}_{0} + {\bf B}{\bf S}{\bf A})
    \end{bmatrix}, 
    \begin{bmatrix}
        \Delta {\bf B} {\bf S}\\
         \Delta {\bf A} 
    \end{bmatrix}
    \right \rangle  \nonumber\\
    =&  \left \langle \begin{bmatrix}
      \nabla L({\bf W}_{0} + {\bf B}\bf S{\bf A}){\bf A}^{\top} {\bf S}^{\top} \\
       {\bf S}^{\top} \bf B^{\top} \nabla L({\bf W}_{0} + {\bf B}\bf S{\bf A})
    \end{bmatrix}, 
    \begin{bmatrix}
        \Delta {\bf B} \\
         \Delta {\bf A} 
    \end{bmatrix}
    \right \rangle  \nonumber\\
    =&  \left \langle \begin{bmatrix}
      \nabla_{\bf B} \tilde{f}_i({\bf B},{\bf A}; {\bf S}) \\
        \nabla_{\bf A} \tilde{f}_i({\bf B}, {\bf A}; {\bf S})
    \end{bmatrix}, 
    \begin{bmatrix}
        {\Delta {\bf B}} \\
        {\Delta {\bf A}} 
    \end{bmatrix}
    \right \rangle,  \label{gradient_transformation}
\end{align}
where the last equality follows the fact that \(\tilde{f}_i({\bf B},{\bf A}; {\bf S}) = L({\bf W}_{0} + {\bf B}\bf S{\bf A}) \) defined in \eqref{notation_function_L} and
\begin{align}
\begin{bmatrix}
      \nabla_{\bf B} \tilde{f}_i({\bf B},{\bf A}; {\bf S}) \\
        \nabla_{\bf A} \tilde{f}_i({\bf B},{\bf A}; {\bf S})
    \end{bmatrix}
=
\begin{bmatrix}
      \nabla L({\bf W}_{0} + {\bf B} {\bf S} {\bf A}){\bf A}^{\top} {\bf S}^{\top} \\
       {\bf S}^{\top} {\bf B}^{\top} \nabla L({\bf W}_{0} + {\bf B} {\bf S} {\bf A})
    \end{bmatrix}.
\end{align}
Plugging \eqref{gradient_transformation} into \eqref{B_A_S_new} gives rise to
\begin{align}\label{later_derivation_S}
    \tilde{f}_{i} ({\bf B} + \Delta {\bf B} , {\bf A}+ \Delta{\bf A}; {\bf S}) \leq \tilde{f}_{i} ({\bf B},{\bf A}; {\bf S}) + 
    \left \langle \begin{bmatrix}
      \nabla_{\bf B} \tilde{f}_i({\bf B},{\bf A}; {\bf S})\\
        \nabla_{\bf A} \tilde{f}_i({\bf B},{\bf A}; {\bf S})
    \end{bmatrix}, 
    \begin{bmatrix}
        {\Delta {\bf B}} \\
         \Delta {\bf A}  
    \end{bmatrix}
    \right \rangle + \frac{L}{2} \left \| \begin{bmatrix}
       {\Delta {\bf B}} {\bf S}\\
        \Delta {\bf A}
    \end{bmatrix} \right \|^2.  
\end{align}
In particular, \(\left \| \begin{bmatrix}
       \Delta {\bf B} {\bf S}\nonumber\\
        \Delta {\bf A}
    \end{bmatrix} \right \|^2 = \|{\Delta {\bf B}} {\bf S}\|^2 + \|\Delta {\bf A}\|^2. \) 
% For \(\|\Delta {\bf B} {\bf S} \|^2\), we have
% \begin{align}
%     \|\Delta {\bf B} {\bf S}\|^2 =& \|{\bf S}^{\top} \Delta {\bf B}^{\top}\|^2 \nonumber \\
%     =& \|{\bf S}^{\top} \delta {\bf b}_{1}, \ldots, {\bf S}^{\top} \delta {\bf b}_{n}\|^2  \nonumber \\
%     =& \sum_{j=1}^n \|{\bf S}^{\top} \delta {\bf b}_{j}\| \notag \\
%     =& \sum_{j=1}^n \delta {\bf b}_{j}^{\top} {\bf S} {\bf S}^{\top} \delta {\bf b}_{j} \label{SS_b} \\ 
%     \leq & \frac{r^2}{k_i^2} \|\Delta {\bf B}\|^2, \nonumber
% \end{align}
% where last inequality follows \eqref{diag:worstcase_bound}.
From \eqref{diag:worstcase_bound}, we know \(\|\Delta {\bf B} {\bf S}\|^2 \leq \frac{r^2}{k_i^2} \|\Delta {\bf B}\|^2\).
Therefore, we have
 \(\left \| \begin{bmatrix}
       {\Delta {\bf B}} {\bf S}\\
        \Delta {\bf A}
    \end{bmatrix} \right \|^2 = \frac{r^2}{k_i^2}\left \| \begin{bmatrix}
       \Delta {\bf B} \\
        \Delta {\bf A}
    \end{bmatrix} \right \|^2.
    \)
As a result, \(\tilde{f}_{i} ({\bf B},{\bf A}; {\bf S})\) (i.e., \(\tilde{f}_{i} ({\bf X}, {\bf S})\)) is \(L \frac{r^2}{k_i^2}\)-smooth. 

ii) Note that \(f_i^{\mathcal{S}}(\mathbf{X}) = \mathbb{E}_{{\bf S} \sim \mathcal{S}_i} [\tilde{f}_{i} ({\bf X}, {\bf S})] \).
Therefore, we further take expectation for \eqref{later_derivation_S} over \({\bf S} \sim \mathcal{S}_i\), leading to
\begin{align}
    f_i^{\mathcal{S}} ({\bf B} + \Delta {\bf B} , {\bf A}+ \Delta{\bf A}) \leq f_i^{\mathcal{S}} ({\bf B},{\bf A}) + 
    \left \langle \begin{bmatrix}
      \nabla_{\bf B} f_i^{\mathcal{S}}(\bf B,{\bf A}) \\
        \nabla_{\bf A} f_i^{\mathcal{S}} (\bf B,{\bf A})
    \end{bmatrix}, 
    \begin{bmatrix}
        \Delta {\bf B} \\
        \Delta {\bf A}  
    \end{bmatrix}
    \right \rangle + \frac{L}{2} \mathbb{E}_{{\bf S}\sim \mathcal{S}_i}\left \| \begin{bmatrix}
       \Delta {\bf B} {\bf S}\\
        \Delta {\bf A}
    \end{bmatrix} \right \|^2.
\end{align}
In particular, \(\mathbb{E}_{{\bf S}\sim \mathcal{S}_i}\left \| \begin{bmatrix}
       \Delta {\bf B} {\bf S}\nonumber\\
        \Delta {\bf A}
    \end{bmatrix} \right \|^2 = \mathbb{E}_{{\bf S}\sim \mathcal{S}_i}\|{\Delta {\bf B}} {\bf S}\|^2 + \|\Delta {\bf A}\|^2. \) 
From \eqref{diag:expected_bound}, we know \(\mathbb{E}_{{\bf S}\sim \mathcal{S}_i}\|\Delta {\bf B} {\bf S}\|^2 \leq \frac{r}{k_i} \|\Delta {\bf B}\|^2\). 
% \begin{align}
%     \mathbb{E}_{{\bf S}\sim \mathcal{S}_i}\|\Delta {\bf B} {\bf S}\|^2  =& \sum_{j=1}^n \delta {\bf b}_{j}^{\top} \mathbb{E}_{{\bf S}\sim \mathcal{S}_i} [{\bf S} {\bf S}^{\top}] \delta  {\bf b}_{j}  \nonumber \\
%     \leq & \frac{r}{k_i} \|\Delta {\bf B}\|^2,  \nonumber 
% \end{align}
% where last inequality follows \eqref{diag:expected_bound}.
In other words,
 \( \mathbb{E}_{{\bf S}\sim \mathcal{S}_i} \left \| \begin{bmatrix}
       {\Delta {\bf B}} {\bf S}\\
        \Delta {\bf A}
    \end{bmatrix} \right \|^2 \leq \frac{r}{k_i}\left \| \begin{bmatrix}
       \Delta {\bf B} \\
        \Delta {\bf A}
    \end{bmatrix} \right \|^2.
    \)
We thus claim that \(f_i^{\mathcal{S}} ({\bf B},{\bf A}) \) (i.e., \( f_i^{\mathcal{S}}(\mathbf{X}) \)) is \( L \frac{r}{k_i} \)-smooth.

iii) Finally, for \( f^{\mathcal{S}}(\mathbf{X}) = \frac{1}{N} \sum_{i=1}^N f_i^{\mathcal{S}}(\mathbf{X}) \), we have 
\begin{equation}
\nabla f^{\mathcal{S}}(\mathbf{X}) = \frac{1}{N} \sum_{i=1}^N \nabla f_i^{\mathcal{S}}(\mathbf{X}).
\end{equation}
Since  \( f_i^{\mathcal{S}}(\mathbf{X}) \) is \( L \frac{r}{k_i} \)-smooth, we thus have
\begin{equation}
   \|\nabla f_i^{\mathcal{S}}(\mathbf{X}) - \nabla f_i^{\mathcal{S}}(\mathbf{Y})\| \leq L \frac{r}{k_i} \|\mathbf{X} - \mathbf{Y}\|, \quad \forall \mathbf{X}, \mathbf{Y}.
\end{equation}
To find the Lipschitz constant of \( f^{\mathcal{S}}(\mathbf{X}) \), we analyze the difference between the gradients at two points \( \bf X \) and \( \bf Y \):
\begin{equation}
    \begin{aligned}
        \|\nabla f^{\mathcal{S}}(\mathbf{X}) - \nabla f^{\mathcal{S}}(\mathbf{Y})\| =& \left\|\frac{1}{N} \sum_{i=1}^N \left(\nabla f_i^{\mathcal{S}}(\mathbf{X}) - \nabla f_i^{\mathcal{S}}(\mathbf{Y})\right)\right\| \\
        \leq & \frac{1}{N} \sum_{i=1}^N \left\| \nabla f_i^{\mathcal{S}}(\mathbf{X}) - \nabla f_i^{\mathcal{S}}(\mathbf{Y})\right\| \\
        \leq & \left( \frac{1}{N} \sum_{i=1}^N \frac{r}{k_i} L \right) \left\| \mathbf{X} - \mathbf{Y}\right\|.
    \end{aligned}
\end{equation}
Therefore, \( f^{\mathcal{S}}(\mathbf{X}) \) is \( \left( \frac{1}{N} \sum_{i=1}^N \frac{r}{k_i} L \right) \)-smooth.

\subsection{Extension to Importance-Based Sketching}\label{appen:extention_importance}
In this work, we provide a theoretically grounded framework for heterogeneous LoRA fine-tuning. Although our algorithm and analysis are based on Random-$k$ sketching, the framework can be extended to importance-based sketching when reliable importance scores for the rows and columns of the global LoRA modules are available.

% Specifically, assume that we have access to the importance scores of the rank-one components
% $
% \{ \mathbf{b}_1 \mathbf{a}_1^\top, \mathbf{b}_2 \mathbf{a}_2^\top, \ldots, \mathbf{b}_r \mathbf{a}_r^\top \}.
% $
% We sort the columns of \(\mathbf{B}\) and the rows of \(\mathbf{A}\) according to these importance scores in descending order.
% Note that such a permutation does not affect the product \(\mathbf{BA}\) as long as the same permutation is applied to both \(\mathbf{B}\) and \(\mathbf{A}\). In other words, we can always permute the global LoRA modules so that the rank-one components are ordered in descending importance before sketching in each training round.

Without loss of generality, we assume that each column-row pair $(\mathbf{b}_i, \mathbf{a}_i)$ of $\mathbf{B}$ and $\mathbf{A}$ is associated with a static importance weight $w_i$, for $i = 1, \ldots, r$, where
\(w_1 \geq w_2 \geq \cdots \geq w_r\). This assumption is reasonable because permuting the column-row pairs does not change the product \(\mathbf{BA}\), as long as the same permutation is applied to both \(\mathbf{B}\) and \(\mathbf{A}\). Consequently, the global LoRA modules can always be reordered so that the components are indexed in descending order of importance prior to sketching in each training round. %Moreover, at an optimal (or stationary) LoRA solution, the relative contribution of each rank-one component is fixed. We therefore model the importance scores as static and invariant to stochastic training noise, which reflects properties of the target solution rather than intermediate iterates.

With the aforementioned importance scores, we sample \(k\) indices from \(\{1, \ldots, r\}\) without replacement according to the weights $\{\frac{w_i}{\sum_{j=1}^r w_j}\}_{i=1}^r$, and denote the selected index set by \(\mathcal{I}\).
The diagonal entries of the sketching matrix \(\mathbf{S}\) are defined as
\begin{equation}\label{importance_sampling_prob}
    s_i =
\begin{cases}
\frac{1}{\pi_i}, & i \in \mathcal{I}, \\[4pt]
0, & i \notin \mathcal{I},
\end{cases}
\qquad
\text{where} \quad
\pi_i = \frac{k w_i}{\sum_{j=1}^r w_j}.
\end{equation}
Here, \(\pi_i\) denotes the marginal probability that index \(i\) is selected.
With this construction, the sketch remains unbiased:
\[
\mathbb{E}[\mathbf{BSA}]
= \mathbf{B} \, \mathbb{E}[\mathbf{S}] \, \mathbf{A}
= \mathbf{BA}.
\]
The \emph{sparsity} of the stochastic sketching matrix can be controlled by the choice of \(k\).

Consequently, the convergence analysis for Random-$k$ sketching extends directly to the importance-based setting, with the sketch-dependent quantities modified to reflect the non-uniform sampling probabilities.
In particular, Lemma~\ref{lamma:sketching_bound} is updated as follows.
\begin{lemma}\label{lamma:sketching_bound_importance}
Let $\mathbf{S}$ be a stochastic diagonal sketching matrix defined in \eqref{importance_sampling_prob}. 
Then for any matrix $\mathbf{X}$, we have
\begin{equation}
\|\mathbf{X}\,\mathbf{S}\|^{2}
~\le~
\left( \frac{\sum_{j=1}^r w_j}{k w_r} \right)^2
\,\|\mathbf{X}\|^{2}
\quad\text{and}\quad
\mathbb{E}_{\mathbf{S}}\!\Bigl[\|\mathbf{X}\,\mathbf{S}\|^{2}\Bigr]
~\le~
\frac{\sum_{j=1}^r w_j}{k w_r}
\,\|\mathbf{X}\|^{2}.
\end{equation}
\end{lemma}

Building on Lemma~\ref{lamma:sketching_bound_importance}, the smoothness characterization in Proposition~\ref{smooth_derivation} is correspondingly updated as follows.
\begin{proposition}\label{smooth_derivation_importance}
Let $\mathcal{S}_i$ denote the set of sketching matrices, where each matrix is defined as in \eqref{importance_sampling_prob} and has $k_i$ non-zero diagonal entries. Under Assumption \ref{smoothness_assumption}, \(\tilde{f}_{i} ({\bf X}; {\bf S}) = f_i({\bf B} {\bf S}, {\bf A})\), \(f_i^{\mathcal{S}}({\bf X}) = \mathbb{E}_{{\bf S} \sim \mathcal{S}_i} [\tilde{f}_{i} ({\bf X}; {\bf S}) ]\), and \(f^{\mathcal{S}} ({\bf X}) = \frac{1}{N} \sum_{i=1}^N f_i^{\mathcal{S}}(\mathbf{X})\) are smooth with parameters \(L \left(\frac{\sum_{j=1}^r w_j}{k w_r}\right)^2 \), \(L \frac{\sum_{j=1}^r w_j}{k w_r}\), and \(\left( \frac{1}{N} \sum_{i=1}^N \frac{\sum_{j=1}^r w_j}{k w_r} \right) L \), respectively. 
\end{proposition}

With Proposition~\ref{smooth_derivation_importance}, we can obtain a convergence bound analogous to Theorem~\ref{convergence_fed_slora_condense}, in which the sketching-dependent coefficients
$\frac{r}{k_i}$ and $\frac{r^2}{k_i^2}$
are replaced by
$\frac{\sum_{j=1}^r w_j}{k_i w_r}$
and
$\left(\frac{\sum_{j=1}^r w_j}{k_i w_r}\right)^2$, respectively.
\begin{theorem}\label{convergence_fed_slora_condense_importance}
    Suppose that Assumptions \ref{smoothness_assumption}-\ref{assumption:gradient_dissimilarity} hold, then there exists a learning rate such that the iterates \(\{{\bf X}^t\}\) generated by FSLoRA satisfy 
\begin{align}
\frac{1}{T}\sum_{t=0}^{T-1} \mathbb{E} \left\| \nabla_{\bf X} f^{\mathcal{S}} ({\bf X}^t) \right\|^2 \! \leq \! 8 \frac{\sqrt{\bar{L}^{\prime} \mathcal{F}_0 \sigma_{\rho}^2} }{\sqrt{N TH}}
+ 10 (\widetilde{L}^{\prime} L)^{\frac{1}{3}} \left(\frac{\mathcal{F}_0 \sigma_{\rho}}{ T}\right)^{\frac{2}{3}} + \frac{4 \mathcal{F}_0 }{T},
\end{align}
where $ \sigma_{\rho}^2 = \sigma^2 + 3(\rho+1) \sigma_h^2$, \(\bar{L}^{\prime} = \left(\frac{1}{N} \sum_{i=1}^N \frac{\sum_{j=1}^r w_j}{k_i w_r}\right) L \), \(\widetilde{L}^{\prime} = \left(\frac{1}{N} \sum_{i=1}^N \left( \frac{\sum_{j=1}^r w_j}{k_i w_r} \right)^2\right) L\) and \(\mathcal{F}_0 = f^{\mathcal{S}}({\bf X}^0)  -  f^*\) with $f^*$ denoting the lower bound of $f^{\mathcal{S}}({\bf X})$.
\end{theorem}
The proofs of Lemma~\ref{lamma:sketching_bound_importance} and Proposition~\ref{smooth_derivation_importance}, and Theorem~\ref{convergence_fed_slora_condense_importance} follow those of Lemma~\ref{lamma:sketching_bound}, Proposition~\ref{smooth_derivation}, and Theorem~\ref{convergence_fed_slora_condense}, respectively, with the modified sketching constants incorporated accordingly.
This completes the extension to the importance-aware sketching.

\end{document}